\documentclass{article}

\usepackage{microtype}
\usepackage{graphicx}
\usepackage{booktabs} 
\usepackage{xspace}    
\usepackage{bm}
\usepackage{mathtools}
\usepackage{mathrsfs}  
\usepackage{amssymb} 
\usepackage{mathabx}
\usepackage{dsfont}
\usepackage{twoopt}
\usepackage{xifthen}
\usepackage{amsthm}
\usepackage{etoolbox}
\usepackage{dsfont}
\usepackage{multirow}
\usepackage{comment}
\usepackage{subcaption}
\usepackage{caption}
\usepackage{ifthen}
\usepackage{enumerate}
\usepackage{enumitem}

\usepackage[breaklinks,pdflang={en-US},pdftex]{hyperref}
\usepackage{tikz}
\usepackage{pgfplots}
\usetikzlibrary{arrows,automata,calc,shapes, positioning, shapes.geometric, fit}
                                         
\DeclareRobustCommand{\eg}{e.g.,\@\xspace}                                      
\DeclareRobustCommand{\ie}{i.e.,\@\xspace}                                      
\DeclareRobustCommand{\wrt}{w.r.t.\@\xspace}                                    

\DeclareRobustCommand{\quotes}[1]{``#1''}

\DeclareMathOperator*{\E}{\mathbb{E}}
\DeclareMathOperator*{\Var}{\mathbb{V}\mathrm{ar}}

\DeclareMathOperator*{\argmax}{arg\,max}
\DeclareMathOperator*{\arginf}{arg\,inf}

\newcommand{\mathbr}[1]{\bm{\mathbf{#1}}}

\newcommand{\de}{\mathrm{d}}

\newcommand{\Ss}{\mathcal{S}}
\newcommand{\As}{\mathcal{A}}
\newcommand{\Bs}{\mathcal{B}}
\newcommand{\SAs}{\mathcal{S} \times \mathcal{A}}
\newcommand{\EOp}[1][\pi]{T^{#1}}
\newcommand{\OOp}{T^*}
\newcommand{\POp}{T^{\delta}}
\newcommand{\aOOp}{\widehat{T}^*}
\newcommand{\aPOp}{\widehat{T}^{\delta}}
\newcommand{\Ppi}[1][\pi]{P^{#1}}
\newcommand{\Nat}[1][]{\mathbb{N}_{\ifthenelse{\isempty{#1}}{}{\ge #1}}}
\newcommand{\Reals}[1][]{\mathbb{R}_{\ifthenelse{\isempty{#1}}{}{\ge #1}}}
\newcommand{\MSPol}{\Pi}
\newcommand{\Id}{\mathrm{Id}}

\newcommand{\Kant}{\mathcal{W}_1}

\newcommand{\algname}[1][]{PFQI({\ifthenelse{\isempty{#1}}{$k$}{#1}})\@\xspace}

\newcommand{\suplip}{\sup_{f: \norm[L][]{f} \le 1}}

\newcommandtwoopt{\norm}[3][p][\mu]{\left\| #3 \right\|_{#1 {\ifthenelse{\isempty{#2}}{}{, #2}}}}

\makeatletter
\newcommand*{\bdiv}{%
  \nonscript\mskip-\medmuskip\mkern5mu%
  \mathbin{\operator@font div}\penalty900\mkern5mu%
  \nonscript\mskip-\medmuskip
}
\makeatother

\newcommand{\cvA}{c_{\mathrm{VI}_{1}, k, q, \rho,\nu}}
\newcommand{\cvB}{c_{\mathrm{VI}_{2}, k, q, \rho,\nu}}

\usepackage{array}
\newcolumntype{L}[1]{>{\raggedright\let\newline\\\arraybackslash\hspace{0pt}}m{#1}}
\newcolumntype{C}[1]{>{\centering\let\newline\\\arraybackslash\hspace{0pt}}m{#1}}
\newcolumntype{R}[1]{>{\raggedleft\let\newline\\\arraybackslash\hspace{0pt}}m{#1}}

\usepackage{thmtools}
\usepackage{thm-restate}

\declaretheorem[name=Lemma,numberwithin=section]{lemma}

\declaretheorem[name=Proposition,numberwithin=section]{prop}
\declaretheorem[name=Assumption,numberwithin=section]{ass}

\declaretheorem[name=Remark,numberwithin=section]{remark}

\allowdisplaybreaks[4]

\newcounter{proofcounter}
\newcounter{saveequation}
\AtBeginEnvironment{proof}{
	\setcounter{saveequation}{\value{equation}}%
    \setcounter{equation}{\value{proofcounter}}%
    \small
}
\AtBeginEnvironment{proofsketch}{
	\setcounter{saveequation}{\value{equation}}%
    \setcounter{equation}{\value{proofcounter}}%
    \small
}
\AtEndEnvironment{proof}{
	\setcounter{proofcounter}{\value{equation}}%
    \setcounter{equation}{\value{saveequation}}%
}
\AtEndEnvironment{proofsketch}{
	\setcounter{proofcounter}{\value{equation}}%
    \setcounter{equation}{\value{saveequation}}%
}

\newcommand\tikzmark[1]{%
  \tikz[remember picture,overlay]\node[inner xsep=0pt] (#1) {};}
\newcommandtwoopt\Textbox[9][0cm][5cm]{%
\begin{tikzpicture}[remember picture,overlay]
  \coordinate (aux) at ([xshift=#2, yshift=#8]#4);
  \node[inner ysep=3pt, xshift=#1, yshift=#9,draw=#6 ,thick, fit=(#3) (aux),baseline, #7] (box) {};
  \node[anchor=north east,font=\rmfamily\small,align=right, color=#6] at (box.north east) {#5};
\end{tikzpicture}%
}

\definecolor{brightBlue}{RGB}{68, 119, 170}
\definecolor{brightCyan}{RGB}{102, 204, 238}
\definecolor{brightGreen}{RGB}{34, 136, 51}
\definecolor{brightYellow}{RGB}{204, 187, 68}
\definecolor{brightRed}{RGB}{238, 102, 119}
\definecolor{brightPurple}{RGB}{170, 51, 119}
\definecolor{brightGrey}{RGB}{187, 187, 187}

\definecolor{vibrantBlue}{RGB}{0, 119, 187}
\definecolor{vibrantCyan}{RGB}{51, 187, 238}
\definecolor{vibrantTeal}{RGB}{0, 153, 136}
\definecolor{vibrantOrange}{RGB}{238, 119, 51}
\definecolor{vibrantRed}{RGB}{204, 51, 17}
\definecolor{vibrantMagenta}{RGB}{238, 51, 119}
\definecolor{vibrantGrey}{RGB}{100, 100, 100}

\tikzstyle{dotted}=                  [dash pattern=on \pgflinewidth off 2pt]
\tikzstyle{densely dotted}=          [dash pattern=on \pgflinewidth off 1pt]
\tikzstyle{loosely dotted}=          [dash pattern=on \pgflinewidth off 4pt]

\tikzstyle{dashed}=                  [dash pattern=on 3pt off 3pt]
\tikzstyle{densely dashed}=          [dash pattern=on 3pt off 2pt]
\tikzstyle{loosely dashed}=          [dash pattern=on 3pt off 6pt]

\tikzstyle{dashdotted}=              [dash pattern=on 3pt off 2pt on \pgflinewidth off 2pt]
\tikzstyle{densely dashdotted}=      [dash pattern=on 3pt off 1pt on \pgflinewidth off 1pt]
\tikzstyle{loosely dashdotted}=      [dash pattern=on 3pt off 4pt on \pgflinewidth off 4pt]

\tikzstyle{dash dot dot}=[dash pattern=on 3pt off 2pt on \pgflinewidth off 2pt on \pgflinewidth off 2pt]
\tikzstyle{densely dash dot dot}=[dash pattern=on 3pt off 1pt on \pgflinewidth off 1pt on \pgflinewidth off 1pt]
\tikzstyle{loosely dash dot dot}= [dash pattern=on 3pt off 4pt on \pgflinewidth off 4pt on \pgflinewidth off 4pt]

\tikzstyle{dash dash dot}=[dash pattern=on 3pt off 2pt on 3pt off 2pt on \pgflinewidth off 2pt]
\tikzstyle{densely dash dash dot}=[dash pattern=on 3pt off 1pt on 3pt off 1pt on \pgflinewidth off 1pt]
\tikzstyle{loosely dash dash dot}= [dash pattern=on 3pt off 4pt on 3pt off 4pt on \pgflinewidth off 4pt]

\tikzstyle{dash dash dot dot}=[dash pattern=on 3pt off 2pt on 3pt off 2pt on \pgflinewidth off 2pt on \pgflinewidth off 2pt]
\tikzstyle{densely dash dash dot dot}=[dash pattern=on 3pt off 1pt on 3pt off 1pt on \pgflinewidth off 1pt on \pgflinewidth off 1pt]
\tikzstyle{loosely dash dash dot dot}= [dash pattern=on 3pt off 4pt on 3pt off 4pt on \pgflinewidth off 4pt on \pgflinewidth off 4pt]

\pgfplotsset{
	cycle list/.define={vibrant}{
    	vibrantBlue, solid, every mark/.append style={solid, fill=vibrantBlue},mark=*\\
    	vibrantCyan, densely dashed, every mark/.append style={solid, fill=vibrantCyan}, mark=triangle*\\
    	vibrantTeal, densely dotted, every mark/.append style={solid, fill=vibrantTeal}, mark=square*\\
    	vibrantOrange, densely dashdotted, every mark/.append style={solid, fill=vibrantOrange}, mark=diamond*\\
        vibrantRed, densely dash dot dot, every mark/.append style={solid, fill=vibrantRed, scale=1.5}, mark=x\\
        vibrantMagenta, densely dash dash dot, every mark/.append style={solid, fill=vibrantMagenta, scale=1.5}, mark=star\\
        vibrantGrey, densely dash dash dot dot, every mark/.append style={solid, fill=vibrantGrey, scale=1.5}, mark=|\\
        },
    }

\usepackage[colorinlistoftodos, textsize=tiny]{todonotes}

\usepackage{hyperref}

\usepackage[accepted]{icml2020}

\icmltitlerunning{Control Frequency Adaptation via Action Persistence in Batch Reinforcement Learning}

\newcommand{\mywidth}{\the\columnwidth}

\begin{document}

\setlength{\abovedisplayskip}{4pt}
\setlength{\belowdisplayskip}{4pt}
\setlength{\textfloatsep}{12pt}

\twocolumn[
\icmltitle{Control Frequency Adaptation via Action Persistence \\ in Batch Reinforcement Learning}



\icmlsetsymbol{equal}{*}

\begin{icmlauthorlist}
\icmlauthor{Alberto Maria Metelli}{polimi}
\icmlauthor{Flavio Mazzolini}{polimi}
\icmlauthor{Lorenzo Bisi}{polimi,isi}
\icmlauthor{Luca Sabbioni}{polimi,isi}
\icmlauthor{Marcello Restelli}{polimi,isi}
\end{icmlauthorlist}

\icmlaffiliation{polimi}{Politecnico di Milano, Milan, Italy.}
\icmlaffiliation{isi}{Institute for Scientific Interchange Foundation, Turin, Italy}

\icmlcorrespondingauthor{Alberto Maria Metelli}{\href{mailto:albertomaria.metelli@polimi.it}{\texttt{albertomaria.metelli@polimi.it}}}

\icmlkeywords{Reinforcement Learning, Batch Reinforcement Learning, Action Persistence}

\vskip 0.3in

]



\printAffiliationsAndNotice{} 

\begin{abstract}
	The choice of the control frequency of a system has a relevant impact on the ability of \emph{reinforcement learning algorithms} to learn a highly performing policy. In this paper, we introduce the notion of \emph{action persistence} that consists in the repetition of an action for a fixed number of decision steps, having the effect of modifying the control frequency. We start analyzing how action persistence affects the performance of the optimal policy, and then we present a novel algorithm, \emph{Persistent Fitted Q-Iteration} (PFQI), that extends FQI, with the goal of learning the optimal value function at a given persistence. After having provided a theoretical study of PFQI and a heuristic approach to identify the optimal persistence, we present an experimental campaign on benchmark domains to show the advantages of action persistence and proving the effectiveness of our persistence selection method.
\end{abstract}

\section{Introduction}\label{sec:introduction}
In recent years, Reinforcement Learning~\citep[RL,][]{sutton2018reinforcement} has proven to be a successful approach to address complex control tasks: from robotic locomotion~\citep[\eg][]{peters2008reinforcement, kober2013learning, haarnoja2019learning, kilinc2019reinforcement} to continuous system control~\citep[\eg][]{ schulman2015trust, lillicrap2015continuous, schulman2017proximal}. These classes of problems are usually formalized in the framework of the \emph{discrete--time} Markov Decision Processes~\citep[MDP,][]{puterman2014markov}, assuming that the control signal is issued at discrete time instants. However, many relevant real--world problems are more naturally defined in the continuous--time domain~\citep{luenberger1979introduction}. Even though a branch of literature has studied RL in \emph{continuous--time} MDPs~\cite{bradtke1994reinforcement, munos1997reinforcement, Doya2000continuous}, the majority of the research has focused on the discrete--time formulation, which appears to be a necessary, but effective, approximation. 

Intuitively, increasing the \emph{control frequency} of the system offers the agent more control opportunities, possibly leading to improved performance as the agent has access to a larger \emph{policy space}. This might wrongly suggest that we should control the system with the highest frequency possible, within its physical limits. However, in the RL framework, the environment dynamics is unknown, thus, a too fine discretization could result in the opposite effect, making the problem harder to solve. Indeed, any RL algorithm needs samples to figure out (implicitly or explicitly) how the environment evolves as an effect of the agent's actions. When increasing the control frequency, the \emph{advantage} of individual actions becomes infinitesimal, making them almost indistinguishable for standard \textit{value-based} RL approaches~\citep{tallec2019time}. As a consequence, the \emph{sample complexity} increases. Instead, low frequencies allow the environment to evolve longer, making the effect of individual actions more easily detectable. Furthermore, in the presence of a system characterized by a \quotes{slowly evolving} dynamics, the gain obtained by increasing the control frequency might become negligible.
Finally, in robotics, lower frequencies help to overcome some partial observability issues, like action execution delays~\citep{kober2013learning}.

Therefore, we experience a fundamental \emph{trade--off} in the control frequency choice that involves the policy space (larger at high frequency) and the sample complexity (smaller at low frequency). Thus, it seems natural to wonder: \quotes{\emph{what is the optimal control frequency?}} An answer to this question can disregard neither the task we are facing nor the learning algorithm we intend to employ. Indeed, the performance loss we experience by reducing the control frequency depends strictly on the properties of the system and, thus, of the task. Similarly, the dependence of the sample complexity on the control frequency is related to how the learning algorithm will employ the collected samples.

In this paper, we analyze and exploit this trade--off in the context of batch RL~\cite{lange2012batch}, with the goal of enhancing the learning process and achieving higher performance. We assume to have access to a discrete--time MDP $\mathcal{M}_{\Delta t_0}$, called base MDP, which is obtained from the time discretization of a continuous--time MDP with fixed base control time step $\Delta t_0$, or equivalently, a control frequency equal to $f_0 = \frac{1}{\Delta t_0}$.
In this setting, we want to select a suitable \emph{control time step} $\Delta t$ that is an integer multiple of the base time step $\Delta t_0$, \ie  $\Delta t = k \Delta t_0$ with $k \in \Nat[1]$.\footnote{We are considering the \emph{near--continuous} setting. This is almost w.l.o.g. compared to the continuous time since the discretization time step $\Delta t_0$ can be chosen to be arbitrarily small. Typically, a lower bound on $\Delta t_0$ is imposed by the physical limitations of the system. Thus, we restrict the search of $\Delta t$ from the continuous set ${\Reals}_{>0}$ to the discrete set $ \{k \Delta t_0\, , \, k \in \Nat[1] \}$. Moreover, considering an already discretized MDP simplifies the mathematical treatment.} Any choice of $k$ generates an MDP $\mathcal{M}_{k \Delta t_0}$ obtained from the base one $\mathcal{M}_{\Delta t_0}$ by altering the transition model so that each action is repeated for $k$ times. For this reason, we refer to $k$ as the action \emph{persistence}, \ie the number of decision epochs in which an action is kept fixed. It is possible to appreciate the same effect in the base MDP $\mathcal{M}_{\Delta t_0}$ by executing a (non-Markovian and non-stationary) policy that persists every action for $k$ time steps. The idea of repeating actions has been previously employed, although heuristically, with deep RL architectures~\citep{lakshminarayanan2017dynamic}. 

The contributions of this paper are theoretical, algorithmic, and experimental.
We first prove that action persistence (with a fixed $k$) can be represented by a suitable modification of the Bellman operators, which preserves the contraction property and, consequently, allows deriving the corresponding value functions (Section~\ref{sec:persistency}). Since increasing the duration of the control time step $k \Delta t_0$ has the effect of degrading the performance of the optimal policy, we derive an algorithm--independent bound for the difference between the optimal value functions of MDPs $\mathcal{M}_{\Delta t_0}$ and $\mathcal{M}_{k \Delta t_0}$, which holds under Lipschitz conditions. The result confirms the intuition that the performance loss is strictly related to how fast the environment evolves as an effect of the actions (Section~\ref{sec:loose}). 
Then, we apply the notion of action persistence in the batch RL scenario, proposing and analyzing an extension of Fitted Q-Iteration~\citep[FQI,][]{ernst2005tree}. The resulting algorithm, \emph{Persistent Fitted Q-Iteration} (PFQI) takes as input a target persistence $k$ and estimates the corresponding optimal value function, assuming to have access to a dataset of samples collected in the base MDP $\mathcal{M}_{\Delta t_0}$ (Section~\ref{sec:pfqi}). 
Once we estimate the value function for a set of candidate persistences $\mathcal{K} \subset \Nat[1]$, we aim at selecting the one that yields the best performing greedy policy. Thus, we introduce a persistence selection heuristic able to approximate the optimal persistence, without requiring further interactions with the environment (Section~\ref{sec:PersistenceSelection}). 
After having revised the literature (Section~\ref{sec:relatedWorks}), we present an experimental evaluation on benchmark domains, to confirm our theoretical findings and evaluate our persistence selection method (Section~\ref{sec:experimental}). We conclude by discussing some open questions related to action persistence (Section~\ref{sec:discussion}). The proofs of all the results are available in Appendix~\ref{apx:proofs}.

\section{Preliminaries}\label{sec:preliminaries}
In this section, we introduce the notation and the basic notions that we will employ in the remainder of the paper. 

\textbf{Mathematical Background}~~Let $\mathcal{X}$ be a set with a $\sigma$-algebra $\sigma_{\mathcal{X}}$, we denote with $\mathscr{P}(\mathcal{X})$ the set of all probability measures and with $\mathscr{B}(\mathcal{X})$ the set of all bounded measurable functions over $(\mathcal{X},\sigma_{\mathcal{X}})$. If $x \in \mathcal{X}$, we denote with $\delta_{x}$ the Dirac measure defined on $x$.
Given a probability measure $\rho \in \mathscr{P}(\mathcal{X})$ and a measurable function $f \in \mathscr{B}(\mathcal{X})$,  we abbreviate $\rho f = \int_{\mathcal{X}} f(x) \rho(\de x)$ (\ie we use $\rho$ as an operator). Moreover, we define the $L_p(\rho)$-norm of $f$ as $\norm[p][\rho]{f}^p = \int_{\mathcal{X}} |f(x)|^p \rho(\de x)$ for $p \ge 1$, whereas the $L_{\infty}$-norm is defined as $\norm[\infty][]{f} = \sup_{x \in \mathcal{X}} f(x)$. Let $\mathcal{D} = \{x_i\}_{i=1}^n \subseteq \mathcal{X}$ we define the $L_p(\rho)$ empirical norm as $\norm[p][\mathcal{D}]{f}^p = \frac{1}{n} \sum_{i=1}^n |f(x_i)|^p$. 

\textbf{Markov Decision Processes}~~A discrete-time Markov Decision Process~\citep[MDP,][]{puterman2014markov} is a 5-tuple $\mathcal{M} = (\Ss, \As, P, R, \gamma)$, where $\Ss$ is a measurable set of states, $\As$ is a 
measurable set of actions, $P: \Ss \times \As \rightarrow \mathscr{P}(\Ss)$ is the transition kernel that for each state-action pair $(s,a) \in \SAs$ provides the probability distribution $P(\cdot|s,a)$ of the next state, $R: \Ss \times \As \rightarrow \mathscr{P}(\mathbb{R})$ is the reward distribution $R(\cdot|s,a)$ for performing action $a \in \As$ in state $s \in \Ss$, whose expected value is denoted by $r(s,a) = \int_{\mathbb{R}} x R(\de x | s,a)$ and uniformly bounded by $R_{\max}<+\infty$, and $\gamma \in [0,1)$ is the discount factor. 

A policy $\pi = \left(\pi_{t} \right)_{t \in \mathbb{N}}$ is a sequence of functions $\pi_t: \mathcal{H}_t \rightarrow \mathscr{P}(\mathcal{A})$ mapping a history $H_t = \left(S_0, A_0, ..., S_{t-1}, A_{t-1}, S_t \right)$ of length $t \in \Nat$ to a probability distribution over $\As$, where $\mathcal{H}_t = ( \SAs )^{t} \times \Ss$. If $\pi_t$ depends only on the last visited state $S_t$ then it is called Markovian, \ie $\pi_t: \mathcal{S} \rightarrow \mathscr{P}(\mathcal{A})$. Moreover, if $\pi_{t}$ does not depend on explicitly $t$ it is stationary, in this case we remove the subscript $t$. We denote with $\MSPol$ the set of Markovian stationary policies. 
A policy $\pi \in \Pi$ induces a (state-action) transition kernel $\Ppi: \SAs \rightarrow \mathscr{P}(\SAs)$, defined for any measurable set $\Bs \subseteq  \SAs$ as~\citep{farahmand2011regularization}:
\begin{align}\label{eq:jointkernel}
	(\Ppi) (\Bs|s,a) = \int_{\Ss} P(\de s'| s,a) \int_{\As} \pi(\de a' | s')  \delta_{(s',a')}(\Bs).
\end{align}

The \emph{action-value function}, or Q-function, of a policy $\pi \in \Pi$ is the expected discounted sum of the rewards obtained by performing action $a$ in state $s$ and following policy $\pi$ thereafter $Q^{\pi}(s,a) = \E \big[ \sum_{t=0}^{+\infty} \gamma^t R_t \rvert S_0=s,\, A_0=a\big]$,
where $R_t \sim R(\cdot|S_t,A_t)$, $S_{t+1} \sim P(\cdot|S_{t},A_t)$, and $A_{t+1} \sim \pi(\cdot|S_{t+1})$ for all $t \in \Nat$. The \emph{value function} is the expectation of the Q-function over the actions: $V^{\pi}(s) = \int_{\As} \pi(\de a|s) Q^\pi(s,a)$.
Given a distribution $\rho \in \mathscr{P}(\Ss)$, we define the \emph{expected return} as $J^{\rho,\pi}(s) =  \int_{\Ss} \rho(\de s) V^\pi(s)$.
The optimal Q-function is given by: $Q^*(s,a) = \sup_{\pi \in \Pi} Q^\pi(s,a)$ for all $(s,a) \in \SAs$. A policy $\pi$ is \emph{greedy} \wrt a function $f \in \mathscr{B}(\SAs)$ if it plays only greedy actions, \ie $\pi(\cdot|s) \in \mathscr{P}\left( \argmax_{a \in \As} f(s,a) \right)$. An \emph{optimal policy} $\pi^* \in \Pi$ is any policy greedy \wrt $Q^*$.

Given a policy $\pi \in \Pi$, the \emph{Bellman Expectation Operator} $\EOp: \mathscr{B}(\SAs) \rightarrow \mathscr{B}(\SAs)$ and the \emph{Bellman Optimal Operator} $\OOp: \mathscr{B}(\SAs) \rightarrow \mathscr{B}(\SAs)$ are defined for a bounded measurable function $f \in \mathscr{B}(\SAs)$ and $(s,a) \in \SAs$ as~\citep{bertsekas2004stochastic}:
\begin{align*}
	&(\EOp f)(s,a) = r(s,a) + \left(\Ppi f \right)(s,a), \\
	&(\OOp f)(s,a) = r(s,a) + \gamma \int_{\Ss} P(\de s'|s,a) \max_{a' \in \As} f(s',a').
\end{align*}
Both $\EOp$ and $\OOp$ are $\gamma$-contractions in $L_{\infty}$-norm and, consequently, they have a unique fixed point, that are the Q-function of policy $\pi$ ($\EOp Q^\pi = Q^\pi$) and the optimal Q-function ($\OOp Q^* = Q^*$) respectively.

\textbf{Lipschitz MDPs}~~Let $(\mathcal{X}, d_{\mathcal{X}})$ and $(\mathcal{Y}, d_{\mathcal{Y}})$ be two metric spaces, a function $f : \mathcal{X} \rightarrow \mathcal{Y}$ is called $L_f$-Lipschitz continuous ($L_f$-LC), where $L_f \ge 0$, if for all $x,x' \in \mathcal{X}$ we have:
\begin{equation}
	d_{\mathcal{Y}}(f(x),f(x')) \le L_f d_{\mathcal{X}} (x,x').
\end{equation}
Moreover, we define the Lipschitz semi-norm as $\norm[L][]{f} = \sup_{x,x' \in \mathcal{X} : x\neq x'} \frac{d_{\mathcal{Y}}(f(x),f(x'))}{d_{\mathcal{X}} (x,x')}$. For real functions
we employ Euclidean distance $d_{\mathcal{Y}}(y,y')=\norm[2][]{y-y'}$, while for probability distributions we use the Kantorovich ($L_1$-Wasserstein) metric defined for $\mu,\nu \in \mathscr{P}(\mathcal{Z})$ as~\citep{villani2008optimal}:
\begin{equation}
	d_{\mathcal{Y}}(\mu,\nu)= \Kant(\mu,\nu) = \suplip \left| \int_{\mathcal{Z}} f(z)  (\mu - \nu)(\de z) \right|. 
\end{equation}
We now introduce the notions of Lipschitz MDP and Lipschitz policy that we will employ in the following~\citep{rachelson2010locality, pirotta2015policy}.
\begin{restatable}[Lipschitz MDP]{ass}{}\label{ass:LipMDP}
	Let $\mathcal{M}$ be an MDP. $\mathcal{M}$ is called $(L_{P},L_r)$-LC if for all $(s,a),(\overline{s},\overline{a}) \in \SAs$:
	\begin{align*}
		& \Kant \left( P (\cdot|s,a), P(\cdot|\overline{s},\overline{a}) \right) \le L_{P} \,  d_{\SAs} \left((s,a), (\overline{s},\overline{a}) \right), \\
		& \left| r(s,a) - r(\overline{s},\overline{a}) \right| \le L_r \,  d_{\SAs} \left((s,a), (\overline{s},\overline{a}) \right).
	\end{align*}
\end{restatable}
\begin{restatable}[Lipschitz Policy]{ass}{}\label{ass:LipPolicy}
	Let $\pi \in \Pi$ be a Markovian stationary policy. $\pi$ is called $L_{\pi}$-LC if for all $s,\overline{s} \in \Ss$:
	\begin{align*}
		& \Kant \left( \pi (\cdot|s), \pi(\cdot|\overline{s}) \right) \le L_{\pi} \,  d_{\Ss} \left(s,\overline{s} \right).
	\end{align*}
\end{restatable}

\section{Persisting Actions in MDPs}\label{sec:persistency}
\begin{figure*}
\raggedleft
\includegraphics[width=.97\textwidth]{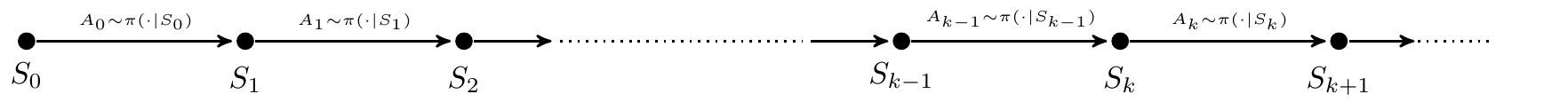}
\includegraphics[width=.97\textwidth]{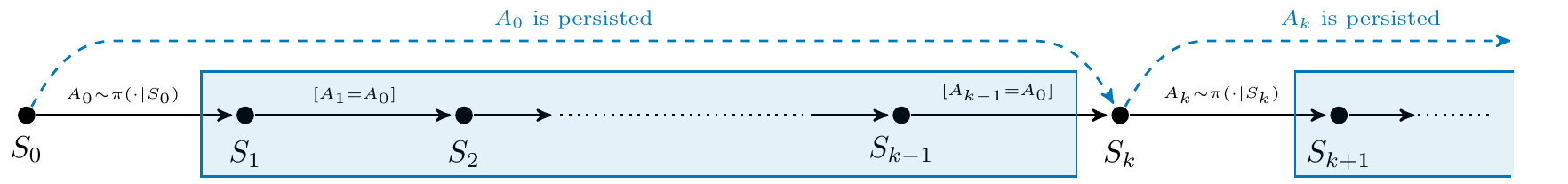}\vspace{-.2cm}
\caption{Agent-environment interaction without (top) and with (bottom) action persistence, highlighting duality. The transition generated by the $k$-persistent MDP $\mathcal{M}_k$ is the cyan dashed arrow, while the actions played by the $k$-persistent policy are inside the cyan rectangle.
\vspace{-.15cm}
}\label{fig:persistence}
\end{figure*}
By the phrase \quotes{executing a policy $\pi$ at persistence $k$}, with $k \in \Nat[1]$, we mean the following 
type of agent-environment interaction. At decision step $t=0$, the agent selects an action according to its policy $A_0 \sim \pi(\cdot|S_0)$. Action $A_0$ is kept fixed, or \emph{persisted}, for the subsequent $k-1$ decision steps, \ie actions $A_1, ..., A_{k-1}$  are all equal to $A_0$.
Then, at decision step $t=k$, the agent queries again the policy $A_k \sim \pi(\cdot|S_k)$ and persists action $A_k$ for the subsequent $k-1$ decision steps and so on. In other words, the agent employs its policy only at decision steps $t$ that are integer multiples of the persistence $k$ ($t \bmod k =0$). Clearly, the usual execution of $\pi$ corresponds to persistence 1.

\subsection{Duality of Action Persistence}\label{sec:duality}
Unsurprisingly, the execution of a Markovian stationary policy $\pi$ at persistence $k > 1$ produces a behavior that, in general, cannot be represented by executing any Markovian stationary policy at persistence 1. Indeed, at any decision step $t$, such a policy needs to remember which action was taken at the previous decision step  $t-1$ (thus it is non-Markovian with memory $1$) and has to understand whether to select a new action based on $t$ (so it is non-stationary).

\begin{restatable}[$k$-persistent policy]{defi}{}\label{def:policyPersisted}
Let $\pi \in \Pi$ be a Markovian stationary policy. For any $k \in \Nat[1]$, the \emph{$k$-persistent policy} induced by $\pi$ is a non--Markovian non--stationary policy, defined for any measurable set $\Bs \subseteq \As$ and $t \in \mathbb{N}$ as:
\begin{equation}\label{eq:policyPersisted}
    \pi_{t,k}(\Bs|H_t) = \begin{cases}
        \pi(\Bs|S_t) & \text{if } t \bmod k = 0\\
         \delta_{A_{t - 1}}(\Bs) & \text{otherwise}
    \end{cases}.
\end{equation}
Moreover, we denote with $\Pi_k = \{(\pi_{t,k})_{t \in \Nat} : \pi \in \MSPol\}$ the set of the $k$-persistent policies.
\end{restatable}
Clearly, for $k=1$ we recover policy $\pi$ as we always satisfy the condition $ t \bmod k = 0$\, \ie $\pi = \pi_{t,1}$ for all $t \in \mathbb{N}$. We refer to this interpretation of action persistence as \emph{policy view}.

A different perspective towards action persistence consists in looking at the effect of the original policy $\pi$ in a suitably modified MDP. To this purpose, we introduce the {(state-action) persistent transition probability kernel} $\Ppi[\delta]: \SAs \rightarrow \mathscr{P}(\SAs)$ defined for any measurable set $\Bs \subseteq \SAs$ as:
\begin{align}
	(\Ppi[\delta]) (\Bs|s,a) = \int_{\Ss} P(\de s'| s,a) \delta_{(s',a)}(\Bs).
\end{align}
The crucial difference between $\Ppi$ and $\Ppi[\delta]$ is that the former samples the action $a'$ to be executed in the next state $s'$ according to $\pi$, whereas the latter replicates in state $s'$ action $a$.
We are now ready to define the $k$-persistent MDP.
 
\begin{restatable}[$k$-persistent MDP]{defi}{}\label{def:MDPPersisted}
Let $\mathcal{M}$ be an MDP. For any $k \in \Nat[1]$, the \emph{$k$-persistent MDP} is the following MDP $\mathcal{M}_k = \left( \Ss, \As, P_k, R_k, \gamma^k \right)$, where $P_k$ and $R_k$ are the $k$-persistent transition model and reward distribution respectively, defined for any measurable sets $\Bs \subseteq \Ss$ , $\mathcal{C} \subseteq \mathbb{R}$ and state-action pair $(s,a) \in \SAs$ as:
\begin{equation}\label{eq:transitionPersisted}
   P_k(\Bs|s,a) = \left((P^{\delta})^{k-1}P\right)(\Bs|s,a),
\end{equation}
\begin{equation}\label{eq:rewardPersisted}
   R_k(\mathcal{C}|s,a) = \sum_{i=0}^{k-1} \gamma^i \left((P^{\delta})^{i} R\right)(\mathcal{C}|s,a),
\end{equation}
and $r_k(s,a) = \int_{\mathbb{R}} x R_k(\de x|s,a) = \sum_{i=0}^{k-1} \gamma^i \left((P^{\delta})^{i} r\right)(s,a)$ is the expected reward, uniformly bounded by $R_{\max} \frac{1-\gamma^k}{1-\gamma}$.
\end{restatable}
The $k$-persistent transition model $P_k$ keeps action $a$ fixed for $k-1$ steps while making the state evolve according to $P$. Similarly, the $k$-persistent reward $R_k$ provides the cumulative discounted reward over $k$ steps in which $a$ is persisted. We define the transition kernel $P^\pi_k$, analogously to $P^\pi$, as in Equation~\eqref{eq:jointkernel}. Clearly, for $k=1$ we recover the base MDP, \ie $\mathcal{M} = \mathcal{M}_1$.\footnote{If $\mathcal{M}$ is the base MDP $\mathcal{M}_{\Delta t_0}$, the $k$--persistent MDP $\mathcal{M}_k$ corresponds to $\mathcal{M}_{k \Delta t_0}$. We remove the subscript $ \Delta t_0$ for brevity.} 
Therefore, executing policy $\pi$ in $\mathcal{M}_k $ at persistence $1$ is equivalent to executing policy $\pi$ at persistence $k$ in the original MDP $\mathcal{M}$. We refer to this interpretation of persistence as \emph{environment view} (Figure~\ref{fig:persistence}). Thus, solving the base MDP $\mathcal{M}$ in the space of $k$-persistent policies $\Pi_{k}$ (Definition~\ref{def:policyPersisted}), thanks to this \emph{duality}, is equivalent to solving the $k$-persistent MDP $\mathcal{M}_k$ (Definition~\ref{def:MDPPersisted}) in the space of Markovian stationary policies $\Pi$. 

It is worth noting that the persistence $k \in \Nat[1]$ can be seen as an \emph{environmental parameter} (affecting $P$, $R$, and $\gamma$), which can be externally configured with the goal to improve the learning process for the agent. In this sense, the MDP $\mathcal{M}_k$ can be seen as a Configurable Markov Decision Process with parameter $k \in \Nat[1]$~\citep{metelli2018configurable, metelli2019reinforcement}.

Furthermore, a persistence of $k$ induces a $k$-persistent MDP $\mathcal{M}_k$ with smaller discount factor $\gamma^k$. Therefore, the effective horizon in $\mathcal{M}_k$ is  $\frac{1}{1-\gamma^k} < \frac{1}{1-\gamma}$. Interestingly, the end effect of persisting actions is similar to reducing the planning horizon, by explicitly reducing the discount factor of the task~\citep{petrik2008biasing, jiang2016the} or setting a maximum trajectory length~\cite{farahmand2016truncated}.

\subsection{Persistent Bellman Operators}
When executing policy $\pi$ at persistence $k$ in the base MDP $\mathcal{M}$, we can evaluate its performance starting from any state-action pair $(s,a) \in \SAs$, inducing a Q-function that we denote with $Q^{\pi}_k$ and call \emph{$k$-persistent action-value function} of $\pi$. Thanks to duality, $Q^{\pi}_k$ is also the action-value function of policy $\pi$ when executed in the $k$-persistent MDP $\mathcal{M}_k$. Therefore, $Q^{\pi}_k$ is the fixed point of the Bellman Expectation Operator of $\mathcal{M}_k$, \ie the operator defined for any $f \in \mathscr{B}(\SAs)$ as $(\EOp_k f)(s,a) = r_k(s,a) + \gamma^k (P_k^{\pi} f)(s,a)$, that we call \emph{$k$-persistent Bellman Expectation Operator}. Similarly, again thanks to duality, the optimal Q-function in the space of $k$-persistent policies $\Pi_k$, denoted by $Q^*_k$ and called \emph{$k$-persistent optimal action-value function}, corresponds to the optimal Q-function of the $k$-persistent MDP, \ie $Q^*_k(s,a) = \sup_{\pi \in \Pi} Q^\pi_k(s,a)$ for all $(s,a) \in \SAs$.
As a consequence, $Q^*_k$ is the fixed point of the Bellman Optimal Operator of $\mathcal{M}_k$, defined for $f \in \mathscr{B}(\SAs)$ as $(\OOp_k f)(s,a) = r_k(s,a) + \gamma^k \int_{\Ss} P_k(\de s' |s,a) \max_{a' \in \As} f(s',a')$, that we call \emph{$k$-persistent Bellman Optimal Operator}. Clearly, both $\EOp_k$ and $\OOp_k$ are $\gamma^k$-contractions in $L_{\infty}$-norm.

We now prove that the $k$-persistent Bellman operators are obtained as composition of the base operators $\EOp$ and $\OOp$.
\begin{restatable}[]{thr}{persistentOperators}\label{thr:persistentOperators}
Let $\mathcal{M}$ be an MDP, $k \in \mathbb{N}_{\ge 1}$ and $\mathcal{M}_k$ be the $k$-persistent MDP. Let $\pi \in \Pi$ be a Markovian stationary policy. Then, $\EOp_k$ and $\OOp_k$ can be expressed as:
\begin{equation}
	\EOp_k = \left(\POp \right)^{k-1} \EOp \quad \text{and} \quad \OOp_{k} = \left(\POp \right)^{k-1} \OOp,
\end{equation}
where $\POp : \mathscr{B}(\SAs) \rightarrow \mathscr{B}(\SAs)$ is the \emph{Bellman Persistent Operator}, defined for $f \in \mathscr{B}(\SAs)$ and $(s,a) \in \SAs$:
\begin{equation}
	\left(\POp f\right)(s,a) = r(s,a) + \gamma \left(\Ppi[\delta] f \right)(s,a).
\end{equation}
\end{restatable}
The fixed point equations for the $k$-persistent Q-functions become: $Q^\pi_k = \left(\POp \right)^{k-1} \EOp Q^\pi_k$ and $Q^*_k = \left(\POp \right)^{k-1} \OOp Q^*_k$.

\section{Bounding the Performance Loss}\label{sec:loose}
Learning in the space of $k$-persistent policies $\Pi_k$ can only lower the performance of the optimal policy, \ie $Q^{*}(s,a) \ge Q^*_k(s,a)$ for $k \in \Nat[1]$. The goal of this section is to bound $\norm[p][\rho]{Q^{*} - Q^*_k}$ as a function of the persistence $k \in \Nat[1]$. To this purpose, we focus on $\norm[p][\rho]{Q^\pi - Q^\pi_k}$ for a fixed policy $\pi \in \Pi$, since denoting with $\pi^*$ an optimal policy of $\mathcal{M}$ and with $\pi^*_k$ an optimal policy of $\mathcal{M}_k$, we have that:
\begin{equation*}
	Q^{*} - Q^*_k = Q^{\pi^*} - Q^{\pi^*_k}_k \le Q^{\pi^*} - Q^{\pi^*}_k,
	\end{equation*}
since $Q^{\pi^*_k}_k(s,a) \ge Q^{\pi^*}_k(s,a)$. We start with the following result which makes no assumption about the structure of the MDP and then we particularize it for the Lipschitz MDPs.
\begin{restatable}[]{thr}{PersistenceBound}\label{thr:PersistenceBound}
	Let $\mathcal{M} $ be an MDP and $\pi \in \Pi$ be a Markovian stationary policy. Let $\mathcal{Q}_k = \{  \left(\POp \right)^{k-2-l} \EOp Q^\pi_k \,:\, l \in \{0,\dots, k-2\} \}$ and for all $(s,a) \in \SAs$ let us define:
	\begin{equation*}
   	 \resizebox{8.4cm}{!}{%
  		$\displaystyle
		d_{\mathcal{Q}_k}^\pi(s,a) = \sup_{f \in \mathcal{Q}_k} \left| \int_{\Ss} \int_{\As} \left( \Ppi(\de s', \de a' | s,a) - P^\delta (\de s', \de a'| s,a) \right) f(s',a') \right|.$
  		}
	\end{equation*}
		 Then, for any $\rho \in \mathscr{P}(\SAs)$, $p \ge 1$, and $k \in \Nat[1]$, it holds that:
	\begin{align*}
		\norm[p][\rho]{Q^\pi - Q^\pi_k} & \le \frac{\gamma (1-\gamma^{k-1})}{(1-\gamma)(1-\gamma^k)} \left\| d^\pi_{\mathcal{Q}_k} \right\|_{p, \eta^{\rho,\pi}_{k}},
	\end{align*}
	where $\eta^{\rho,\pi}_{k} \in \mathscr{P}(\SAs)$ is a probability measure defined for any measurable set $\Bs \subseteq \SAs$ as:
	\begin{equation*}
		\eta^{\rho,\pi}_{k}(\Bs) = \frac{(1-\gamma)(1-\gamma^k)}{\gamma(1-\gamma^{k-1})} \!\!\! \sum_{\substack{i \in \mathbb{N} \\ i \bmod k \neq 0}} \!\!\! \gamma^{i} \left( \rho \left({\Ppi} \right)^{i-1} \right)(\Bs).
	\end{equation*}
\end{restatable}
The bound shows that the Q-function difference depends on the discrepancy $d_{\mathcal{Q}_k}^\pi$ between the transition-kernel $P^\pi$ and the corresponding persistent version $P^\delta$, which is a form of \emph{integral probability metric}~\citep{muller1997integral}, defined in terms of the set $\mathcal{Q}_k$. This term is averaged with the distribution $\eta^{\rho,\pi}_{k}$, which encodes the (discounted) probability of visiting a state-action pair, ignoring the visitations made at decision steps $i$ that are multiple of the persistence $k$. Indeed, in those steps, we play policy $\pi$ regardless of whether persistence is used.\footnote{$\eta^{\rho,\pi}_{k}$ resambles the $\gamma$-discounted state-action distribution~\cite{sutton2000policy}, but ignoring the decision steps multiple of $k$.} The dependence on $k$ is represented in the term $\frac{1-\gamma^{k-1}}{1-\gamma^k}$. When $k\rightarrow 1$ this term displays a linear growth in $k$, being asymptotic to $(k-1) \log \frac{1}{\gamma}$, and, clearly, vanishes for $k = 1$.
Instead, when $k\rightarrow \infty$ this term tends to 1.

If no structure on the MDP/policy is enforced, the dissimilarity term $d_{\mathcal{Q}_k}^\pi$ may become large enough to make the bound vacuous, \ie larger than $\frac{\gamma R_{\max}}{1-\gamma}$, even for $k=2$ (see Appendix~\ref{sec:discussionOnBoundd}). Intuitively, since the persistence will execute old actions in new states, we need to guarantee that the environment state changes slowly \wrt to time and the policy must play similar actions in similar states. This means that if an action is good in a state, it will also be almost good for states encountered in the near future. Although the condition on $\pi$ is directly enforced by Assumption~\ref{ass:LipPolicy}, we need a new notion of regularity over time for the MDP.

\begin{ass}\label{ass:TimeLipAss}
	Let $\mathcal{M}$ be an MDP. $\mathcal{M}$ is $L_T$--\emph{Time-Lipschitz Continuous} ($L_T$--TLC) if for all $(s,a) \in \SAs$:
	\begin{equation}
		\Kant \left( P(\cdot|s,a) , \delta_{s} \right) \le L_T.
	\end{equation}
\end{ass}
This assumption requires that the Kantorovich distance between the distribution of the next state $s'$ and the deterministic distribution centered in the current state $s$ is bounded by $L_T$, \ie the system does not evolve \quotes{too fast} (see Appendix~\ref{ex:dynamicalSystem}). We can now state the following result.

\begin{restatable}[]{thr}{corollTLC}\label{thr:corollTLC}
	Let $\mathcal{M}$ be an MDP and $\pi \in \Pi$ be a Markovian stationary policy. Under Assumptions~\ref{ass:LipMDP},~\ref{ass:LipPolicy}, and~\ref{ass:TimeLipAss}, if $\gamma \max\left\{ L_P + 1, L_P (1+L_\pi) \right\} < 1$ and if $\rho(s,a) = \rho_{\mathcal{S}}(s) \pi(a|s)$ with $\rho_{\Ss} \in \mathscr{P}(\Ss)$, then for any $k \in \Nat[1]$:
	\begin{equation*}
		\left\| d^\pi_{\mathcal{Q}_k} \right\|_{p, \eta^{\rho,\pi}_{k}} \le \textcolor{green!50!black}{L_{\mathcal{Q}_k}} \left[ \textcolor{vibrantCyan!80!black}{(L_\pi + 1)} \textcolor{vibrantOrange!80!black}{L_T} + \textcolor{vibrantMagenta!80!black}{\sigma_p} \right].
	\end{equation*}
	where $\sigma_p^p = \sup_{s \in \Ss} \int_{\As} \int_{\As} d_{\mathcal{A}} \left( a, a' \right)^p \pi(\de a|s) \pi(\de a'|s) $, and $L_{\mathcal{Q}_k} = \frac{L_r}{1-\gamma \max\left\{ L_P + 1, L_P (1+L_\pi) \right\}}$.
\end{restatable}

Thus, the dissimilarity $ d^\pi_{\mathcal{Q}_k}$ between $P^{\pi}$ and $P^\delta$ can be bounded with four terms. i) $\textcolor{green!50!black}{L_{\mathcal{Q}_k}}$ is (an upper-bound of) the Lipschitz constant of the functions in the set $\mathcal{Q}_k$. Indeed, under Assumptions~\ref{ass:LipMDP} and~\ref{ass:LipPolicy} we can reduce the dissimilatity term to the Kantorivich distance (Lemma~\ref{lemma:DissimilatityToKant}):
$$d^\pi_{\mathcal{Q}_k}(s,a) \le  L_{\mathcal{Q}_k} \Kant \left(P^\pi(\cdot|s,a), P^\delta(\cdot|s,a) \right).$$ 
ii) $\textcolor{vibrantCyan!80!black}{(L_\pi + 1)}$ accounts for the Lipschitz continuity of the policy, \ie policies that prescribe similar actions in similar states have a small value of this quantity. iii) $\textcolor{vibrantOrange!80!black}{L_T}$ represents the speed at which the environment state evolves over time. iv) $\textcolor{vibrantMagenta!80!black}{\sigma_p}$ denotes the average distance (in $L_p$-norm) between two actions prescribed by the policy in the same state. This term is zero for deterministic policies and can be related to the maximum policy variance (Lemma~\ref{lemma:toVariance}). A more detailed discussion on the conditions requested in Theorem~\ref{thr:corollTLC} is reported in Appendix~\ref{apx:conditionsThr42}.

\section{Persistent Fitted Q-Iteration}\label{sec:pfqi}
In this section, we introduce an extension of Fitted Q-Iteration~\citep[FQI,][]{ernst2005tree} that employs the notion of persistence.\footnote{From now on, we assume that $|\As| < +\infty$.} \emph{Persisted Fitted Q-Iteration} (\algname) takes as input a \emph{target persistence} $k \in \Nat[1]$ and its goal is to approximate the $k$-persistent optimal action-value function $Q^*_k$. Starting from an initial estimate $Q^{(0)}$, at each iteration we compute the next estimate $Q^{(j+1)}$ by performing an approximate application of $k$-persistent Bellman optimal operator to the previous estimate $Q^{(j)}$, \ie $Q^{(j+1)} \approx T^*_k Q^{(j)}$.
In practice, we have two sources of approximation in this process: i) the representation of the Q-function; ii) the estimation of the $k$-persistent Bellman optimal operator. (i) comes from the necessity of using functional space $\mathcal{F} \subset \mathscr{B}(\SAs)$ to represent $Q^{(j)}$ when dealing with continuous state spaces. (ii) derives from the approximate computation of $T^*_k$ which needs to be estimated from samples. 

Clearly, with samples collected in the $k$-persistent MDP $\mathcal{M}_k$, the process described above reduces to the standard FQI. However, our algorithm needs to be able  to estimate $Q^*_k$ for different values of $k$, using the same dataset of samples collected in the base MDP $\mathcal{M}$ (at persistence 1).\footnote{In real--world cases, we might be unable to interact with the physical system to collect samples for any persistence $k$ of interest.}
For this purpose, we can exploit the decomposition $T^*_k = (T^\delta)^{k-1}T^*$ of Theorem~\ref{thr:persistentOperators} to reduce a single application of $T^*_k$ to a sequence of $k$ applications of the 1-persistent operators. Specifically, at each iteration $j $ with $j \bmod k=0$, given the current estimate $Q^{(j)}$, we need to perform (in this order) a single application of $T^*$ followed by $k-1$ applications of $T^{\delta}$, leading to the sequence of approximations:
\begin{equation}
	Q^{(j+1)} \approx \begin{cases}
					\OOp Q^{(j)} & \text{if } j \bmod k = 0 \\
					\POp Q^{(j)} & \text{otherwise}
				\end{cases}.
\end{equation}

In order to estimate the Bellman operators, we have access to a dataset
$\mathcal{D} = \{(S_i,A_i,S'_i,R_i)\}_{i=1}^n$ collected in the base MDP $\mathcal{M}$, where $(S_i,A_i) \sim \nu $, $S'_i \sim P(\cdot|S_i,A_i)$, $R_i \sim {R}(\cdot|S_i,A_i)$, and $\nu\in \mathscr{P}(\SAs)$ is a sampling distribution. We employ $\mathcal{D}$ to compute the \emph{empirical Bellman operators}~\citep{farahmand2011regularization} defined for $f \in \mathscr{B}(\SAs)$ as:
\begin{equation*}
\begin{aligned}
	& (\aOOp f)(S_i,A_i) = R_i + \gamma \textstyle \max_{a \in \As} \displaystyle f(S_i',a) \quad i=1,\dots,n \\
	& (\aPOp f)(S_i,A_i) = R_i + \gamma f(S_i',A_i) \quad i=1,\dots,n. \\
\end{aligned}
\end{equation*}
These operators are unbiased conditioned to $\mathcal{D}$~\citep{farahmand2011regularization}: $\E[(\aOOp f)(S_i,A_i) | S_i,A_i] = (\OOp f)(S_i,A_i)$ and $\E[(\aPOp f)(S_i,A_i) | S_i,A_i] = (\POp f)(S_i,A_i)$.

The pseudocode of \algname is summarized in Algorithm~\ref{alg:pfqi}. At each iteration $j =0, \dots J-1$, we first compute the target values $Y^{(j)}$ by applying the empirical Bellman operators, $\aOOp$ or $\aPOp$, on the current estimate $Q^{(j)}$ (\hyperref[l:phase1]{\textcolor{vibrantBlue}{Phase 1}}). Then, we project the target $Y^{(j)}$ onto the functional space $\mathcal{F}$ by solving the least squares problem (\hyperref[l:phase2]{\textcolor{vibrantTeal}{Phase 2}}):
\begin{equation*}
	Q^{(j+1)} \in \arginf_{f \in \mathcal{F}} \left\| f - Y^{(j)} \right\|^2_{2,\mathcal{D}} = \frac{1}{n} \sum_{i=1}^n \left|f(S_i,A_i) - Y_i^{(j)}\right|^2.
\end{equation*}
Finally, we compute the approximation of the optimal policy $\pi^{(J)}$, \ie the greedy policy \wrt $Q^{(J)}$ (\hyperref[l:phase3]{\textcolor{vibrantRed}{Phase 3}}). 
\begin{algorithm}[t]
\caption{Persistent Fitted Q-Iteration \algname.} \label{alg:pfqi}
\small
\textbf{Input:} $k$ persistence, $J$ number of iterations ($J \bmod k =0$), $Q^{(0)}$ initial action-value function, $\mathcal{F}$ functional space, $\mathcal{D}=\{(S_i,A_i,S'_i,R_i)\}_{i=1}^n$ batch samples\\
\textbf{Output:} greedy policy $\pi^{(J)}$
\begin{algorithmic}
\FOR{$j = 0,\dots,J-1$}\vspace{.1cm}
	\IF{\tikzmark{start1} $j \bmod k = 0$}
		\STATE $ Y_i^{(j)} = \widehat{T}^* Q^{(j)}(S_i,A_i), \quad i=1,...,n$
	\ELSE
		\STATE $   Y_i^{(j)} = \widehat{T}^\delta Q^{(j)}(S_i,A_i),\quad i=1,...,n$\tikzmark{end1}
	\ENDIF\vspace{.25cm}
	\STATE \tikzmark{start2}  $  Q^{(j+1)} \in \arginf_{f \in \mathcal{F}} \big\| f - Y^{(j)} \big\|^2_{2,\mathcal{D}}$\tikzmark{end2}\vspace{.1cm}
\ENDFOR\vspace{.15cm}
\STATE \tikzmark{start3} $ \pi^{(J)}(s) \in \argmax_{a \in \mathcal{A}} Q^{(J)} (s,a), \quad \forall s \in \mathcal{S}$ \tikzmark{end3}
\vspace{-.2cm}
\end{algorithmic}
\Textbox[-.3cm][2.7cm]{start1}{end1}{Phase 1\label{l:phase1}}{vibrantBlue}{solid}{-0.48cm}{0.1cm}
\Textbox[-.03cm][2.93cm]{start2}{end2}{Phase 2\label{l:phase2}}{vibrantTeal}{solid}{-0.25cm}{0.15cm}
\Textbox[0cm][2.14cm]{start3}{end3}{Phase 3\label{l:phase3}}{vibrantRed}{solid}{-0.2cm}{0.15cm}
\end{algorithm}

\subsection{Theoretical Analysis}
In this section, we present the computational complexity analysis and the study of the error propagation in \algname.

\textbf{Computational Complexity}~~The computational complexity of \algname decreases monotonically with the persistence $k$. Whenever applying $\aPOp$, we need a single evaluation of $Q^{(j)}$, while $|\As|$ evaluations are needed for $\aOOp$ due to the $\max$ over $\As$. Thus, the overall complexity of $J$ iterations of \algname with $n$ samples, disregarding the cost of regression and assuming that a single evaluation of $Q^{(j)}$ takes constant time, is given by $\mathcal{O} \left( J n \left( 1 + (|\As| - 1)/k \right)\right)$ (Proposition~\ref{prop:complexity}). 

\textbf{Error Propagation}~~We now consider the error propagation in \algname. Given the sequence of Q-functions estimates $(Q^{(j)})_{j=0}^{J} \subset \mathcal{F}$ produced by \algname, we define the approximation error at each iteration $j = 0,\dots, J-1$ as:
\begin{equation}\label{eq:epsDef}
	\epsilon^{(j)} = \begin{cases}
				\OOp Q^{(j)} - Q^{(j+1)} & \text{if } j \bmod k = 0 \\
				\POp Q^{(j)} - Q^{(j+1)} & \text{otherwise}
			\end{cases}.
\end{equation}
The goal of this analysis is to bound the distance between the $k$--persistent optimal Q-function $Q^*_k$ and the Q-function $Q_k^{\pi^{(J)}}$ of the greedy policy $\pi^{(J)}$ \wrt $Q^{(J)}$, after $J$ iterations of \algname. The following result extends Theorem 3.4 of~\citet{farahmand2011regularization} to account for action persistence.

\begin{restatable}[Error Propagation for \algname]{thr}{errorProp}\label{thr:errorProp}
Let $p \ge 1$, $k \in \Nat[1]$, $J \in \Nat[1]$ with $J \bmod k = 0$ and $\rho \in \mathscr{P}(\SAs)$. Then for any sequence $(Q^{(j)})_{j=0}^{J} \subset \mathcal{F}$ uniformly bounded by $Q_{\max} \le \frac{R_{\max}}{1-\gamma}$, the corresponding $(\epsilon^{(j)})_{j=0}^{J-1}$ defined in Equation~\eqref{eq:epsDef} and for any $r \in [0,1]$ and $q \in [1, +\infty]$ it holds that:
\begin{align*}
	\norm[p][\rho]{Q^*_k - Q^{\pi^{(J)}}_k} & \le \frac{2\gamma^k}{(1-\gamma)(1-\gamma^k)} \bigg[ \frac{2}{1-\gamma} \gamma^{\frac{J}{p}} R_{\max} \\
	&  + C_{\mathrm{VI},\rho,\nu}^{\frac{1}{2p}}(J,r,q) \mathcal{E}^{\frac{1}{2p}} (\epsilon^{(0)},\dots, \epsilon^{(J-1)};r,q)   \bigg].
\end{align*}
The  expression of $ C_{\mathrm{VI},\rho,\nu}(J;r,q)$ and  $\mathcal{E}(\cdot;r,q)$ can be found in Appendix~\ref{apx:proofPFQI}.
\end{restatable}
We immediately observe that for $k=1$ we recover Theorem 3.4 of~\citet{farahmand2011regularization}. The term $C_{\mathrm{VI},\rho,\nu}(J;r,q)$ is defined in terms of suitable \emph{concentrability coefficients} (Definition~\ref{defi:concentrab}) and encodes the distribution shift between the sampling distribution $\nu$ and the one induced by the greedy policy sequence $(\pi^{(j)})_{j=0}^{J}$ encountered along the execution of \algname. $\mathcal{E}(\cdot;r,q)$ incorporates the approximation errors $(\epsilon^{(j)})_{j=0}^{J-1}$. In principle, it is hard to compare the values of these terms for different persistences $k$ since both the greedy policies and the regression problems are different. Nevertheless, it is worth noting that the multiplicative term $\frac{\gamma^k}{1-\gamma^k} $ decreases in $k \in \Nat[1]$. Thus, other things being equal, the bound value decreases when increasing the persistence. 

Thus, the trade-off in the choice of control frequency, which motivates action persistence, can now be stated more formally. We aim at finding the persistence $k \in \Nat[1]$ that, for a fixed $J$, allows learning a policy $\pi^{(J)}$ whose Q-function $Q^{\pi^{(J)}}_k$ is the closest to $Q^*$. Consider the decomposition:
\begin{equation*}
	\norm[p][\rho]{Q^* - Q^{\pi^{(J)}}_k} \le {\norm[p][\rho]{Q^* - Q^*_k}} + {\norm[p][\rho]{Q^*_k - Q^{\pi^{(J)}}_k}}.
\end{equation*}
The term $\norm[p][\rho]{Q^* - Q^*_k}$ accounts for the performance degradation due to action persistence: it is algorithm--independent, and it increases in $k$ (Theorem~\ref{thr:PersistenceBound}). Instead, the second term $\|{Q^*_k - Q^{\pi^{(J)}}_k}\|_{p,\rho}$ decreases with $k$ and depends on the algorithm (Theorem~\ref{thr:errorProp}). Unfortunately, optimizing their sum is hard since the individual bounds contain terms that are not known in general (\eg Lipschitz constants, $\epsilon^{(j)}$). The next section proposes heuristics to overcome this problem.

\section{Persistence Selection}\label{sec:PersistenceSelection}
\begin{algorithm}[t]
\caption{Heuristic Persistence Selection.}\label{alg:kselection}
\small
\textbf{Input:} batch samples $\mathcal{D} = \{(S_{0}^i,A_{0}^i,\dots,S_{H_i-1}^i, A_{H_i-1}^i,S_{H_i}^i)\}_{i=1}^m$, set of persistences $\mathcal{K}$, set of Q-function $\{ Q_k : k \in \mathcal{K} \}$, regressor $\mathtt{Reg}$ \\
\textbf{Output:} approximately optimal persistence $\widetilde{k}$
\begin{algorithmic}
\FOR{$k \in \mathcal{K}$}
	\STATE $\widehat{J}_k^\rho = \frac{1}{m} \sum_{i=1}^m V_k(S_0^i)$
	\STATE Use the $\mathtt{Reg}$ to get an estimate $\widetilde{Q}_k$ of $T^*_k Q_k$
	\STATE $\big\| \widetilde{Q}_k - Q_k \big\|_{1,\mathcal{D}}  = \scriptstyle \frac{1}{\sum_{i=1}^m H_i} \sum_{i=1}^m \sum_{t=0}^{H_i-1} |\widetilde{Q}_k(S_{t}^i,A_{t}^i) - Q_k(S_{t}^i,A_{t}^i) |$
\ENDFOR
\STATE $ \widetilde{k} \in \argmax_{k \in \mathcal{K}}  B_k = \widehat{J}_k^\rho - \frac{1}{1-\gamma^k} \big\| \widetilde{Q}_k - Q_k \big\|_{1,\mathcal{D}}$.
\end{algorithmic}
\end{algorithm}
In this section, we discuss how to select a persistence 
$k$ in a set $\mathcal{K} \subset \Nat[1]$ of candidate persistences, when we are given a set of estimated Q-functions: $\{Q_k \,:\,k \in \mathcal{K}\}$.\footnote{For instance, the $Q_k$ can be obtained  by executing \algname with different persistences $k \in \mathcal{K}$.} Each $Q_k$ induces a greedy policy $\pi_k$. Our goal is to find the persistence $k \in \mathcal{K}$ such that $\pi_k$ has the maximum expected return in the corresponding $k$--persistent MDP $\mathcal{M}_k$:
\begin{equation}\label{eq:persistenceSelectionProblem}
	k^* \in \argmax_{k \in \mathcal{K}} J^{\rho,\pi_k}_{k}, \quad \rho \in \mathscr{P}(\Ss).
\end{equation} 
In principle, we could execute $\pi_k$ in $\mathcal{M}_k$ to get an estimate of $J^{\rho,\pi_k}_{k}$ and employ it to select the persistence $k$.
However, in the batch setting, further interactions with the environment might be not allowed. On the other hand, directly using the estimated Q-function $Q_k$ is inappropriate, since we need to take into account how well $Q_k$ approximates $Q^{\pi_k}_k$. This trade--off is encoded in the following result, which makes use of the \emph{expected Bellman residual}.
\begin{restatable}[]{lemma}{lowerBoundQ}\label{thr:lowerBoundQ}
Let $Q \in \mathscr{B}(\SAs)$ and $\pi$ be a greedy policy \wrt  $Q$. Let  $J^{\rho} = \int \rho(\de s) V(s)$, with $V(s) = \max_{a \in \As} Q(s,a)$ for all $s \in \Ss$. Then, for any $k \in \Nat[1]$, it holds that:
\begin{equation}
	J^{\rho,\pi}_{k} \ge J^{\rho} - \frac{1}{1-\gamma^k} \norm[1][\eta^{\rho,\pi}]{T^*_kQ - Q},
\end{equation}
where $\eta^{\rho, \pi} = (1-\gamma^k)\rho \pi \left(\Id - \gamma^k P^\pi_k \right)^{-1}$, is the $\gamma$-discounted stationary distribution induced by policy $\pi$ and distribution $\rho$ in MDP $\mathcal{M}_k$.
\end{restatable}  
To get a usable bound, we need to make some simplifications. First, we assume that $\mathcal{D} \sim \nu$ is composed of $m$ \emph{trajectories}, \ie $\mathcal{D} = \{(S_{0}^i,A_{0}^i,\dots,S_{H_i-1}^i, A_{H_i-1}^i,S_{H_i}^i)\}_{i=1}^m$, where $H_i$ is the trajectory length and the initial states are sampled as $S_0^i \sim \rho$. 
In this way, $J^{\rho}$ can be estimated from samples as $\widehat{J}^{\rho} = \frac{1}{m} \sum_{i=1}^m V(S_{0}^i)$. Second, since we are unable to compute expectations over $\eta^{\rho,\pi}$, we replace it with the sampling distribution $\nu$.\footnote{This introduces a bias that is negligible if $\norm[\infty][]{\eta^{\rho,\pi}/{\nu}} \approx 1$ (details in Appendix~\ref{apx:changeDistribution}).} 
Lastly, estimating the expected Bellman residual is problematic since its empirical version is biased~\citep{antos2008learning}. Thus, we resort to an approach similar to~\citep{farahmand2011model}, assuming to have a regressor $\mathtt{Reg}$ able to output an approximation $\widetilde{Q}_k$ of $T^*_k Q$. In this way, we replace $\norm[1][\nu]{T^*_k Q - Q }$ with $\|\widetilde{Q}_k - Q \|_{1,\mathcal{D}}$ (details in Appendix~\ref{apx:discussionSimplification}). In practice, we set $Q=Q^{(J)}$ and we obtain $\widetilde{Q}_k$ running \algname for $k$ additional iterations, setting $\widetilde{Q}_k = Q^{(J+k)}$. Thus, the procedure (Algorithm~\ref{alg:kselection}) reduces to optimizing the index:
\begin{equation}\label{eq:index}
	\widetilde{k} \in \argmax_{k \in \mathcal{K}} B_k = \widehat{J}_k^\rho - \frac{1}{1-\gamma^k} \norm[1][\mathcal{D}]{\widetilde{Q}_k - Q_k}.
\end{equation}

\section{Related Works}\label{sec:relatedWorks}
In this section, we revise the works connected to persistence, focusing on continuous--time RL and temporal abstractions.

\textbf{Continuous--time RL}~~Among the first attempts to extend value--based RL to the continuous--time domain there is \emph{advantage updating}~\citep{bradtke1994reinforcement}, in which Q-learning~\citep{watkins1989learning} is modified to account for infinitesimal control timesteps. Instead of storing the Q-function, the \emph{advantage function} $A(s,a) = Q(s,a) - V(s)$ is recorder. The continuous time is addressed in \citet{baird1994reinforcement} by means of the semi-Markov decision processes~\citep{howard1963semi} for finite--state problems. The optimal control literature has extensively studied the solution of the Hamilton-Jacobi-Bellman equation, \ie the continuous--time counterpart of the Bellman equation, when assuming the knowledge of the environment~\citep{bertsekas2005dynamic, fleming2006controlled}. The model--free case has been tackled by resorting to time (and space) discretizations~\citep{peterson1993line}, with also convergence guarantees~\citep{munos1997convergent, munos1997reinforcement}, and coped with function approximation~\citep{dayan1995improving, Doya2000continuous}. More recently, the sensitivity of deep RL algorithm to the time discretization has been analyzed in~\citet{tallec2019time}, proposing an adaptation of advantage updating to deal with small time scales, that can be employed with deep architectures.

\setlength{\tabcolsep}{2pt}
\begin{table*}[t]
\caption{Results of PFQI in different environments and persistences. For each persistence $k$, we report the sample mean and the standard deviation of the estimated return of the last policy $\widehat{J}_k^{\rho, \pi_k}$. For each environment, the persistence with highest average performance and the ones not statistically significantly different from that one (Welch's t-test with $p < 0.05$) are in bold. The last column reports the mean and the standard deviation of the performance loss $\delta$ between the optimal persistence and the one selected by the index $B_k$ (Equation~\eqref{eq:index}).}
\label{tab:results}
\begin{center}
\scriptsize
\begin{tabular}{lcccccccc}
\toprule
\multirow{2}{*}{\footnotesize Environment} & \multicolumn{7}{c}{\footnotesize Expected return at persistence $k$ \scriptsize ($\widehat{J}_k^{\rho, \pi_k}$, mean $\pm$ std)} &  \footnotesize Performance loss\\
\cline{2-8}
 &  $k=1$ &  $k=2$ &  $k=4$ &  $k=8$ &  $k=16$ &  $k=32$ &  $k=64$ &   ($\delta$ mean $\pm$ std)\\
\midrule
\footnotesize Cartpole & $ 169.9 \pm 5.8 $ & $ 176.5 \pm 5.0 $ & $ \mathbf{ 239.5 \pm 4.4 } $ & $ 10.0 \pm 0.0 $ & $ 9.8 \pm 0.0 $& $ 9.8 \pm 0.0 $ & $ 9.8 \pm 0.0 $  & $0.0 \pm 0.0$ \\
\footnotesize MountainCar & $ -111.1 \pm 1.5 $ & $ -103.6 \pm 1.6 $ & $ -97.2 \pm 2.0 $ & $ \mathbf{ -93.6 \pm 2.1 } $ & $ \mathbf{ -94.4 \pm 1.8 } $ & $ \mathbf{ -92.4 \pm 1.5 } $ & $ -136.7 \pm 0.9 $&  $1.88 \pm 0.85$ \\
\footnotesize LunarLander & $ -165.8 \pm 50.4 $ & $ -12.8 \pm 4.7 $ & $ \mathbf{ 1.2 \pm 3.6 } $ & $ \mathbf{ 2.0 \pm 3.4 } $ & $ -44.1 \pm 6.9 $ & $ -122.8 \pm 10.5 $ & $ -121.2 \pm 8.6 $ & $2.12 \pm 4.21$ \\
\footnotesize Pendulum & $ \mathbf{ -116.7 \pm 16.7 } $ & $ \mathbf{ -113.1 \pm 16.3 } $ & $ \mathbf{ -153.8 \pm 23.0 } $ & $ -283.1 \pm 18.0 $ & $ -338.9 \pm 16.3 $ & $ -364.3 \pm 22.1 $ & $ -377.2 \pm 21.7 $&  $3.52 \pm 0.0$\\
\footnotesize Acrobot & $ -89.2 \pm 1.1 $ & $ \mathbf{ -82.5 \pm 1.7 } $ & $ \mathbf{ -83.4 \pm 1.3 } $ & $ -122.8 \pm 1.3 $ & $ -266.2 \pm 1.9$ & $-287.3 \pm 0.3$ & $ -286.7 \pm 0.6$ & $0.80 \pm 0.27$ \\
\footnotesize Swimmer & $ 21.3 \pm 1.1 $ & $ \mathbf{ 25.2 \pm 0.8 } $ & $ \mathbf{ 25.0 \pm 0.5 } $ & $ \mathbf{ 24.0 \pm 0.3 } $ & $ 22.4 \pm 0.3 $ & $ 12.8 \pm 1.2 $ & $ 14.0 \pm 0.2 $&  $2.69 \pm 1.71$\\
\footnotesize Hopper & $ 58.6 \pm 4.8 $ & $ 61.9 \pm 4.2 $ & $ 62.2 \pm 1.7 $ & $ 59.7 \pm 3.1 $ & $ 60.8 \pm 1.0 $ & $ 66.7 \pm 2.7 $ & $ \mathbf{ 73.4 \pm 1.2 } $& $5.33 \pm 2.32$ \\
\footnotesize  Walker 2D & $ 61.6 \pm 5.5 $ & $ 37.6 \pm 4.0 $ & $ 62.7 \pm 18.2 $ & $ \mathbf{ 80.8 \pm 6.6 } $ & $ \mathbf{ 102.1 \pm 19.3 } $ & $ \mathbf{ 91.5 \pm 13.0 } $ & $ \mathbf{ 97.2 \pm 17.6 } $ & $5.10 \pm 3.74$\\
\bottomrule
\end{tabular}
\end{center}
\vspace{-.4cm}
\end{table*}

\textbf{Temporal Abstractions}~~The notion of action persistence can be seen as a form of \emph{temporal abstraction}~\citep{sutton1999between, precup2001temporal}. Temporally extended actions have been extensively used in the hierarchical RL literature to model different time resolutions~\cite{singh1992reinforcement, singh1992scaling}, subgoals~\citep{dietterich1998the}, and combined with the actor--critic architectures~\citep{bacon2017the}. Persisting an action is a particular instance of a semi-Markov \emph{option}, always lasting $k$ steps. According to the {flat option representation}~\citep{precup2001temporal}, we have as initiation set $\mathcal{I} = \Ss$ the set of all states, as internal policy the policy that plays deterministically the action taken when the option was initiated, \ie the $k$--persistent policy, and as termination condition whether $k$ timesteps have passed after the option started, \ie $\beta(H_t) = \mathds{1}_{\{t \bmod k = 0\}}$. Interestingly, in~\citet{mann2015approximate} an approximate value iteration procedure for options lasting at least a given number of steps is proposed and analyzed. This approach shares some similarities with action persistence. Nevertheless, we believe that the option framework is more general and usually the time abstractions are related to the semantic of the tasks, rather than based on the modification of the control frequency, like action persistence.

\section{Experimental Evaluation}\label{sec:experimental}
In this section, we provide the empirical evaluation of PFQI, with the threefold goal: i) proving that a persistence $k > 1$ can boost learning, leading to more profitable policies, ii) assessing the quality of our persistence selection method, and iii) studying how the batch size influences the performance of PFQI policies for different persistences. Refer to Appendix~\ref{apx:Experiments} for detailed experimental settings. 

We train PFQI, using extra-trees~\citep{geurts2006extra} as a regression model, for $J$ iterations and different values of $k$, starting with the same dataset $\mathcal{D}$ collected at persistence 1. To compare the performance of the learned policies $\pi_k$ at the different persistences, we estimate their expected return ${J}^{\rho,\pi_k}_k$ in the corresponding MDP $\mathcal{M}_k$. Table~\ref{tab:results} shows the results for different continuous environments and different persistences averaged over 20 runs and highlighting in bold the persistence with the highest average performance and the ones that are not statistically significantly different from that one. Across the different environments we observe some common trends in line with our theory: i) persistence 1 rarely leads to the best performance; ii) excessively increasing persistence prevents the control at all. 
In Cartpole~\citep{barto1983neuronlike}, we easily identify a persistence ($k=4$) that outperforms all the others. In the Lunar Lander~\citep{brockman2016open} persistences $k \in\{4,8\}$ are the only ones that lead to positive return (\ie the lander does not crash) and in the Acrobot domain~\citep{geramifard2015rlpy} we identify $k \in\{2,4\}$ as optimal persistences. A qualitatively different behavior is displayed in Mountain Car~\citep{moore1990efficient}, Pendulum~\citep{brockman2016open}, and Swimmer~\citep{coulom2002reinforcement}, where we observe a plateau of three persistences with similar performance. An explanation for this phenomenon is that, in those domains, the optimal policy tends to persist actions on its own, making the difference less evident. Intriguingly, the more complex Mujoco domains, like Hopper and Walker 2D~\cite{erickson2019assistive}, seem to benefit from the higher persistences.

\begin{figure*}
\centering
	\includegraphics[width=.95\textwidth]{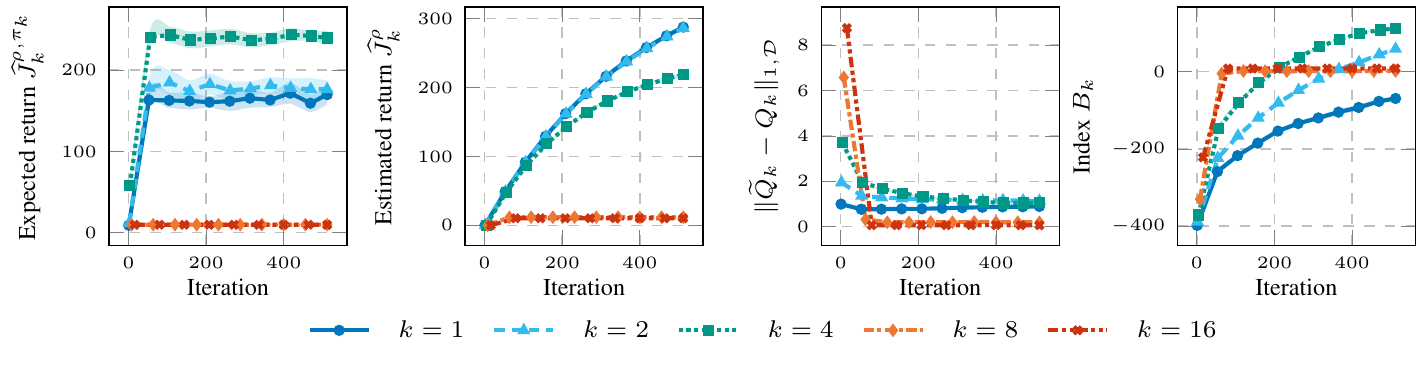}
	\vspace{-.4cm}
	\caption{Expected return $\widehat{J}_k^{\rho,\pi_k}$, estimated return $\widehat{J}_k^\rho$, estimated expected Bellman residual $\| \widetilde{Q}_k - Q_k \|_{1,\mathcal{D}}$, and persistence selection index $B_k$ in the Cartpole experiment as a function of the number of iterations for different persistences. 20 runs, 95 \% c.i.}\vspace{-.2cm}
	\label{fig:cartpole}
\end{figure*}

\begin{figure}
\centering
	\includegraphics[width=.9\linewidth]{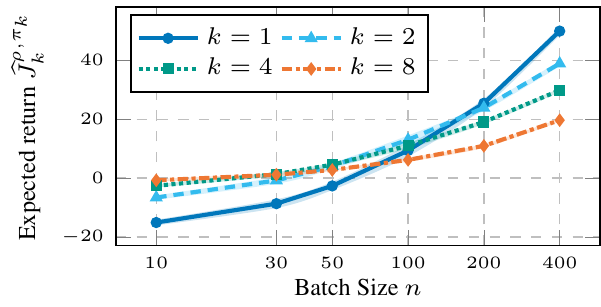}
	\vspace{-.2cm}
	\caption{Expected return $\widehat{J}_k^{\rho,\pi_k}$ in the Trading experiment as a function of the batch size. 10 runs, 95 \% c.i.}
	\label{fig:tradingMain}
\end{figure}

To test the quality of our persistence selection method, we compare the performance of the estimated optimal persistence, \ie the one with the highest estimated expected return $\widehat{k} \in \argmax \widehat{J}_k^{\rho,\pi_k}$, and the performance of the persistence $\widetilde{k}$ selected by maximizing the index $B_k$ (Equation~\eqref{eq:index}). For each run $i=1,\dots,20$, we compute the \emph{performance loss} $\delta_i = \widehat{J}_{\widehat{k}}^{\rho,\pi_{\widehat{k}}} - \widehat{J}_{\widetilde{k}_i}^{\rho,\pi_{\widetilde{k}_i}}$ and we report it in the last column of Table~\ref{tab:results}. In the Cartpole experiment, we observe a zero loss, which means that our heuristic always selects the optimal persistence ($k=4$). Differently, non--zero loss occurs in the other domains, which means that sometimes the index $B_k$ mispredicts the optimal persistence. 
Nevertheless, in almost all cases the average performance loss is significantly smaller than the magnitude of the return, proving the effectiveness of our heuristics.

In Figure~\ref{fig:cartpole}, we show the learning curves for the Cartpole experiment, highlighting the components that contribute to the index $B_k$. The first plot reports the estimated \emph{expected return} $\widehat{J}_k^{\rho,\pi_k}$, obtained by averaging 10 trajectories executing $\pi_k$ in the environment $\mathcal{M}_k$, which
confirms that $k=4$ is the optimal persistence. The second plot shows the \emph{estimated return} $\widehat{J}_k^\rho$ obtained by averaging the Q-function $Q_k$ learned with \algname, over the initial states sampled from $\rho$. We can see that for $k \in \{1,2\}$, \algname tends to overestimate the return, while for $k=4$ we notice a slight underestimation. The overestimation phenomenon can be explained by the fact that with small persistences we perform a large number of applications of the operator $\widehat{T}^*$, which involves a maximization over the action space, injecting an overestimation bias.
By combining this curve with the expected Bellman residual (third plot), we get the value of our persistence selection index $B_k$ (fourth plot). Finally, we observe that $B_k$ correctly ranks persistences 4 and 8, but overestimates persistences $8$ and $16$, compared to persistence $1$.

To analyze the effect of the batch size, we run PFQI on the Trading environment (see Appendix~\ref{apx:batch_dep}) varying the number of sampled trajectories. In Figure~\ref{fig:tradingMain}, we notice that the performance improves as the batch size increases, for all persistences. Moreover, we observe that if the batch size is small ($n \in \{10, 30, 50\}$), higher persistences ($k \in\{2,4,8\}$) result in better performances, while for larger batch sizes, $k=1$ becomes the best choice. Since data is taken from real market prices, this environment is very noisy, thus, when the amount of samples is limited, PFQI can exploit higher persistences to mitigate the poor estimation.

\section{Open Questions}\label{sec:discussion}

\textbf{Improving Exploration with Persistence}~~We analyzed the effect of action persistence on FQI with a fixed dataset, collected in the base MDP $\mathcal{M}$. In principle, samples can be collected at arbitrary persistence. We may wonder how well the same sampling policy (\eg the uniform policy over $\mathcal{A}$), executed at different persistences, explores the environment. For instance, in Mountain Car, high persistences increase the probability of reaching the goal, generating more informative datasets (preliminary results in Appendix~\ref{apx:persistenceExplor}).

\textbf{Learn in $\mathcal{M}_k$ and execute in $\mathcal{M}_{k'}$}~~Deploying each policy $\pi_k$ in the corresponding MDP $\mathcal{M}_k$ allows for some guarantees (Lemma~\ref{thr:lowerBoundQ}). However, we empirically discovered that using $\pi_k$ in an MDP $\mathcal{M}_{k'}$ with smaller persistence $k'$ sometimes improves its performance. (preliminary results in Appendix~\ref{apx:khpersistence}). We wonder what regularity conditions on the environment are needed to explain this phenomenon.

\textbf{Persistence in On--line RL}~Our approach focuses on batch off--line RL. However, the on--line framework could open up new opportunities for action persistence. Specifically, we could \emph{dynamically} adapt the persistence (and so the control frequency) to speed up learning. Intuition suggests that we should start with a low frequency, reaching a fairly good policy with few samples, and then increase it to refine the learned policy.

\section{Discussion and Conclusions}
In this paper, we formalized the notion of action persistence, \ie the repetition of a single action for a fixed number $k$ of decision epochs, having the effect of altering the control frequency of the system. We have shown that persistence leads to the definition of new Bellman operators and that we are able to bound the induced performance loss, under some regularity conditions on the MDP. Based on these considerations, we presented and analyzed a novel batch RL algorithm, PFQI, able to approximate the value function at a given persistence. The experimental evaluation justifies the introduction of persistence, since reducing the control frequency can lead to an improvement when dealing with a limited number of samples.
Furthermore, we introduced a persistence selection heuristic, which is able to identify good persistence in most cases.  
We believe that our work makes a step towards understanding why repeating actions may be useful for solving complex control tasks. Numerous questions remain unanswered, leading to several appealing future research directions.

\section*{Acknowledgements}
The research was conducted under a cooperative agreement between ISI Foundation, Banca IMI and Intesa Sanpaolo Innovation Center. 
\bibliography{paper}
\bibliographystyle{icml2020}

\clearpage
\onecolumn
\appendix
\setlength{\abovedisplayskip}{8pt}
\setlength{\belowdisplayskip}{8pt}

\section*{Index of the Appendix}
In the following, we briefly recap the contents of the Appendix.
\begin{itemize}[leftmargin=*, label={--}]
	\item Appendix~\ref{apx:proofs} reports all proofs and derivations.
	\item Appendix~\ref{apx:details} provides additional considerations and discussion concerning the regularity conditions for bounding the performance loss due to action persistence.
	\item Appendix~\ref{apx:discussionSimplification} illustrates the motivations behind the choice we made for defining our persistence selection index.
	\item Appendix~\ref{apx:Experiments} presents the experimental setting, together with additional experimental results (including some experiments with neural networks as regressor).
	\item Appendix~\ref{apx:openQuestions} reports some preliminary experiments to motivate the open questions stated in the main paper.
\end{itemize}

\section{Proofs and Derivations}\label{apx:proofs}
In this appendix, we report the proofs of all the results presented in the main paper.

\subsection{Proofs of Section~\ref{sec:persistency}}
\persistentOperators*
\begin{proof}
	We derive the result by explicitly writing the definitions of the $k$-persistent transition model $P_k$ and $k$-persistent reward distribution $R_k$ in terms of $P$, $R$ and $\gamma$ in the definition of the $k$-persistent Bellman expectation operator $\EOp_k$.
	Let $f \in \mathscr{B}(\SAs)$ and $(s,a) \in \SAs$:
	\begin{align}
		(\EOp_k f)(s,a) &= r_k(s,a) + \gamma^k (P_k^{\pi} f)(s,a) \notag\\
				& = \sum_{i=0}^{k-1} \gamma^i \left((P^\delta)^i r \right) (s,a) + \gamma^k ((P^\delta)^{k-1} P^{\pi} f)(s,a) \label{p:101}\\
				& = \left( \sum_{i=0}^{k-1} \gamma^i (P^\delta)^i r   + \gamma^k (P^\delta)^{k-1} P^{\pi} f \right)(s,a) \notag \\
				& = \left( \sum_{i=0}^{k-2} \gamma^i (P^\delta)^i r  + \gamma^{k-1} (P^\delta)^{k-1}  \left(r +   \gamma P^{\pi} f \right) \right)(s,a) \label{p:102} \\
				& = \left( \sum_{i=0}^{k-2} \gamma^i (P^\delta)^i r  + \gamma^{k-1} (P^\delta)^{k-1}  \EOp f \right)(s,a) \label{p:103},
	\end{align}
	where line~\eqref{p:101} follows from Definition~\ref{def:MDPPersisted}, line~\eqref{p:102} is obtained by isolating the last term in the summation $\gamma^{k-1} (P^\delta)^{k-1} r$ and collecting $\gamma^{k-1} (P^\delta)^{k-1}$ thanks to the linearity of $(P^\delta)^{k-1}$, and line~\eqref{p:103} derives from the definition of the Bellman expectation operator $\EOp $. It remains to prove that for $g \in \mathscr{B}(\SAs)$ and $(s,a) \in \SAs$, we have the following identity:
	\begin{equation}\label{eq:pIdPersistence}
		(\POp)^{k-1} g = \sum_{i=0}^{k-2} \gamma^i (P^\delta)^i r  + \gamma^{k-1} (P^\delta)^{k-1}  g.
	\end{equation}
	We prove it by induction on $k \in \Nat[1]$. For $k=1$ we have only $g = (\POp)^{0} g$. Let us assume that the identity hold for all integers $h < k$, we prove the statement for $k$:
	\begin{align}
		\left((\POp)^{k-1} g \right)(s,a) & = \left(\sum_{i=0}^{k-2} \gamma^i (P^\delta)^i r  + \gamma^{k-1} (P^\delta)^{k-1}  g\right)(s,a) \notag \\
			& = \left(\sum_{i=0}^{k-3} \gamma^i (P^\delta)^i r  + \gamma^{k-2} (P^\delta)^{k-2} (r + \gamma  P^\delta g) \right)(s,a) \label{p:104} \\
			& = \left(\sum_{i=0}^{k-3} \gamma^i (P^\delta)^i r  + \gamma^{k-2} (P^\delta)^{k-2} \POp g \right) (s,a) \label{p:105} \\ 
			& = (\POp)^{k-2} \POp g = (\POp)^{k-1} g. \label{p:106}
	\end{align}
	where line~\eqref{p:104} derives from isolating the last term in the summation and collecting $\gamma^{k-2} (P^\delta)^{k-2}$ thanks to the linearity of $(P^\delta)^{k-2}$, line~\eqref{p:105} comes from the definition of the Bellman persisted operator $\POp$, and finally line~\eqref{p:106} follows from the inductive hypothesis. We get the result by taking $g = T^\pi f$.
	
	Concerning the $k$-persistent Bellman optimal operator the derivation is analogous. For simplicity, we define the $\max$-operator $M : \mathscr{B}(\SAs) \rightarrow \mathscr{B}(\Ss)$ defined for a bounded measurable function $f \in \mathscr{B}(\SAs)$ and a state $s \in \Ss$ as $(M f)(s) = \max_{a \in \As} f(s,a)$. As a consequence the Bellman optimal operator becomes: $\OOp f = r + \gamma P M f$. Therefore, we have:
	\begin{align}
	(\OOp_k f)(s,a) &= r_k(s,a) + \gamma^k \int_{\Ss} P_k(\de s' |s,a) \max_{a' \in \As} f(s',a') \notag\\
					& = r_k(s,a) + \gamma^k \int_{\Ss} P_k(\de s' |s,a) M f(s') \label{p:107}\\
					& = \left( r_k + \gamma^k  P_k M  f \right) (s,a) \label{p:108}\\
				& = \left( \sum_{i=0}^{k-1} \gamma^i (P^\delta)^i r   + \gamma^k (P^\delta)^{k-1} P M f \right)(s,a) \notag \\
				& = \left( \sum_{i=0}^{k-2} \gamma^i (P^\delta)^i r  + \gamma^{k-1} (P^\delta)^{k-1}  \left(r +   \gamma P M f \right) \right)(s,a) \\
				& = \left( \sum_{i=0}^{k-2} \gamma^i (P^\delta)^i r  + \gamma^{k-1} (P^\delta)^{k-1}  \OOp f \right)(s,a),
	\end{align}
	where line~\eqref{p:107} derives from the definition of the $\max$-operator $M$ and line~\eqref{p:107} from the definition of the operator $P_k$. By applying Equation~\eqref{eq:pIdPersistence} we get the result.
\end{proof}

\subsection{Proofs of Section~\ref{sec:loose}}

\begin{restatable}[]{lemma}{}\label{lemma:QFunctionDecomposition1}
	Let $\mathcal{M}$ be an MDP and $\pi \in \Pi$ be a Markovian stationary policy, then for any $k \in \Nat[1]$ the following two identities hold:
	\begin{align*}
		Q^{\pi} - Q^\pi_{k} & = \left(\Id - \gamma^k \left( P^{\pi} \right)^k \right)^{-1} \left( \left(\EOp\right)^k Q^{\pi}_k -  \left(\POp \right)^{k-1} \EOp Q^{\pi}_k \right) \\
		& = \left(\Id - \gamma^k \left( P^{\delta} \right)^{k-1} \Ppi \right)^{-1} \left( \left(\EOp \right)^k Q^\pi -  \left(\POp \right)^{k-1} \EOp Q^\pi\right),
	\end{align*}
	where $\Id : \mathscr{B}(\SAs) \rightarrow \mathscr{B}(\SAs)$ is the identity operator over $\SAs$.
\end{restatable}

\begin{proof}
	We prove the equalities by exploiting the facts that $Q^{\pi}$ and $Q^\pi_{k}$ are the fixed points of $\EOp$ and $\EOp_k$:
	\begin{align}
		Q^{\pi} - Q^\pi_{k} &= \EOp Q^{\pi} - \EOp_k Q^\pi_{k} \notag \\
				& = \left(\EOp\right)^{k} Q^{\pi} - \left( \POp \right)^{k-1} \EOp Q^\pi_{k} \label{p:201} \\
				& =  \left(\EOp\right)^{k} Q^{\pi} - \left( \POp \right)^{k-1} \EOp Q^\pi_{k} \pm  \left(\EOp\right)^{k} Q^{\pi}_k \label{p:202} \\
				& = \gamma^k \left( P^{\pi} \right)^k \left( Q^{\pi} - Q^{\pi}_k \right) + \left( \left(\EOp\right)^{k} Q^{\pi}_k - \left( \POp \right)^{k-1} \EOp Q^\pi_{k}  \right), \label{p:203}
	\end{align}
	where line~\eqref{p:201} derives from recalling that $Q^{\pi} = \EOp Q^{\pi} $ and exploiting Theorem~\ref{thr:persistentOperators}, line~\eqref{p:203} is obtained by exploiting the identity that holds for two generic bounded measurable functions $f,g \in \mathscr{B}(\SAs)$:
	\begin{equation}\label{eq:IdentityForOperatorsPpi}
	\left(\EOp\right)^{k} f - \left(\EOp\right)^{k} g = \gamma^k \left( P^{\pi} \right)^k (f - g).
	\end{equation}
	We prove this identity by induction. For $k=1$ the identity clearly holds. Suppose Equation~\eqref{eq:IdentityForOperatorsPpi} holds for all integers $h < k$, we prove that it holds for $k$ too:
	\begin{align}
	\left(\EOp\right)^{k} f - \left(\EOp\right)^{k} g & = \EOp \left(\EOp\right)^{k-1} f - \EOp \left(\EOp\right)^{k-1} g \notag\\
	& = r + \gamma P^\pi \left(\EOp\right)^{k-1} f - r - P^\pi \gamma \left(\EOp\right)^{k-1} g \notag\\
	& =  \gamma P^\pi \left( \left(\EOp\right)^{k-1} f - \left(\EOp\right)^{k-1} g \right) \label{p:211} \\
	& = \gamma P^\pi \gamma^{k-1} \left( P^{\pi} \right)^{k-1} (f - g) \label{p:212}\\
	& = \gamma^k \left( P^{\pi} \right)^k (f - g), \notag
	\end{align}
	where line~\eqref{p:211} derives from the linearity of operator $P^\pi$ and line~\eqref{p:212} follows from the inductive hypothesis. From line~\eqref{p:203} the result follows immediately, recalling that since $\gamma < 1$ the inversion of the operator is well-defined:
	\begin{align*}
		& Q^{\pi} - Q^\pi_{k}  = \gamma^k \left( P^{\pi} \right)^k \left( Q^{\pi} - Q^{\pi}_k \right) + \left( \left(\EOp\right)^{k} Q^{\pi}_k - \left( \POp \right)^{k-1} \EOp Q^\pi_{k}  \right) \implies \\
		& \quad \left( \Id - \gamma^k \left( P^{\pi} \right)^k  \right) \left(Q^{\pi} - Q^\pi_{k} \right) = \left( \left(\EOp\right)^{k} Q^{\pi}_k - \left( \POp \right)^{k-1} \EOp Q^\pi_{k}  \right) \implies \\
		& \quad  Q^{\pi} - Q^\pi_{k}  = \left( \Id - \gamma^k \left( P^{\pi} \right)^k  \right)^{-1} \left( \left(\EOp\right)^{k} Q^{\pi}_k - \left( \POp \right)^{k-1} \EOp Q^\pi_{k}  \right).
	\end{align*}
	The second identity of the statement is obtained with an analogous derivation, in which at line~\eqref{p:202} we sum and subtract $\left(\POp\right)^{k-1} \EOp Q^{\pi}$ and we exploit the identity for two bounded measurable functions $f,g \in \mathscr{B}(\SAs)$:
		\begin{equation}\label{eq:IdentityForOperatorsPdelta}
	\left( \POp \right)^{k-1} \EOp Q f - \left( \POp \right)^{k-1} \EOp Q g = \gamma^k \left( P^{\delta} \right)^{k-1} \Ppi (f - g).
	\end{equation}
\end{proof}

\begin{restatable}[]{lemma}{}\label{lemma:QFunctionDecomposition2}
	Let $\mathcal{M}$ be an MDP and $\pi \in \Pi$ be a Markovian stationary policy, then for any $k \in \Nat[1]$ and any bounded measurable function $f \in \mathscr{B}(\SAs)$ the following two identities hold:
	\begin{align*}
		\left(\EOp\right)^{k-1} f -  \left(\POp \right)^{k-1}  f  & = \sum_{i=0}^{k-2} \gamma^{i+1} \left( \Ppi \right)^{i} \left( \Ppi - \Ppi[\delta] \right) \left(\POp \right)^{k-2-i} f \\
		& = \sum_{i=0}^{k-2} \gamma^{i+1} \left( \Ppi[\delta] \right)^{i} \left( \Ppi - \Ppi[\delta] \right) \left(\EOp \right)^{k-2-i}  f.
	\end{align*}
\end{restatable}

\begin{proof}
	We start with the first identity and we prove it by induction on $k$. For $k = 1$, we have that the left hand side is zero and the summation on the right hand side has no terms. Suppose that the statement holds for every $h < k$, we prove the statement for $k$:
	\begin{align}
		\left(\EOp\right)^{k-1} f -  \left(\POp \right)^{k-1}  f & = \left(\EOp\right)^{k-1} f -  \left(\POp \right)^{k-1} f \pm  \left(\EOp \right)^{k-2}\POp f \label{p:301}\\
		& = \left( \left(\EOp\right)^{k-2 } \EOp f - \left(\EOp \right)^{k-2}\POp f  \right) + \left( \left(\EOp \right)^{k-2}\POp f  -  \left(\POp \right)^{k-2} \POp f \right) \notag \\
		& = \gamma^{k-2} \left(P^{\pi}\right)^{k-2} \left( \EOp f -\POp f \right) + \left( \left(\EOp \right)^{k-2}\POp f  -  \left(\POp \right)^{k-2} \POp f \right) \label{p:302} \\
		& = \gamma^{k-1} \left(P^{\pi} \right)^{k-2} \left( \Ppi  -P^{\delta}  \right) f + \sum_{i=0}^{k-3} \gamma^{i+1} \left( \Ppi \right)^{i} \left( \Ppi - \Ppi[\delta] \right) \left(\POp \right)^{k-3-i} \POp f \label{p:304} \\ 
		& = \sum_{i=0}^{k-2} \gamma^{i+1} \left( \Ppi \right)^{i} \left( \Ppi - \Ppi[\delta] \right) \left(\POp \right)^{k-2-i} f \label{p:305},
	\end{align}
	where in line~\eqref{p:302} we exploited the identity at Equation~\eqref{eq:IdentityForOperatorsPpi}, line~\eqref{p:304} derives from observing that $\EOp f -\POp f  = \gamma \left( \Ppi  -P^{\delta}  \right) f$ and by inductive hypothesis applied on $\POp f$ which is a bounded measurable function as well. Finally, line~\eqref{p:305} follows from observing that the first term completes the summation up to $k-2$. The second identity in the statement can be obtained by an analogous derivation in which at line~\eqref{p:301} we sum and subtract $\left(\POp \right)^{k-2} \EOp f$ and, later, exploit the identity at Equation~\eqref{eq:IdentityForOperatorsPdelta}.
\end{proof}

\begin{restatable}[Persistence Lemma]{lemma}{}\label{lemma:PersistenceLemma}
	Let $\mathcal{M}$ be an MDP and $\pi \in \Pi$ be a Markovian stationary policy, then for any $k \in \Nat[1]$ the following two identities hold:
	\begin{align*}
		Q^\pi - Q^\pi_k & = \sum_{\substack{i \in \mathbb{N} \\ i \bmod k \neq 0}} \gamma^{i} \left( \Ppi \right)^{i-1} \left( \Ppi - \Ppi[\delta] \right) \left(\POp \right)^{k-2-(i-1) \bmod k} \EOp Q^\pi_k \\
		& =\sum_{\substack{i \in \mathbb{N} \\ i \bmod k \neq 0}} \gamma^{i} \left( \left( P^{\delta} \right)^{k-1} \Ppi \right)^{i \bdiv k} \left( P^{\delta} \right)^{ i \bmod k - 1} \left( \Ppi - \Ppi[\delta] \right) \left(\EOp \right)^{k-i \bmod k} Q^\pi,
	\end{align*}
	where for two non-negative integers $a,b \in \Nat$, we denote with $a \bmod b$ and $a \bdiv b$ the remainder and the quotient of the integer division between $a$ and $b$ respectively.
\end{restatable}

\begin{proof}
	We start proving the first identity. Let us consider the first identity of Lemma~\ref{lemma:QFunctionDecomposition1}:
	\begin{align}
		Q^{\pi} - Q^\pi_{k} & = \left(\Id - \gamma^k \left( P^{\pi} \right)^k \right)^{-1} \left( \left(\EOp\right)^k Q^{\pi}_k -  \left(\POp \right)^{k-1} \EOp Q^{\pi}_k \right) \notag \\
		& =  \left( \sum_{j=0}^{+\infty} \gamma^{kj}  \left(P^{\pi} \right)^{kj} \right) \left( \left(\EOp\right)^k Q^{\pi}_k -  \left(\POp \right)^{k-1} \EOp Q^{\pi}_k \right) \label{p:401} \\
		& = \left( \sum_{j=0}^{+\infty} \gamma^{kj}  \left(P^{\pi} \right)^{kj} \right) \sum_{l=0}^{k-2} \gamma^{l+1} \left( \Ppi \right)^{l} \left( \Ppi - \Ppi[\delta] \right) \left(\POp \right)^{k-2-l} \EOp Q^\pi_k  \label{p:402}\\
		& = \sum_{j=0}^{+\infty} \gamma^{kj}  \left(P^{\pi} \right)^{kj}  \sum_{l=0}^{k-2} \gamma^{l+1} \left( \Ppi \right)^{l} \left( \Ppi - \Ppi[\delta] \right) \left(\POp \right)^{k-2-l} \EOp Q^\pi_k \notag \\
		& = \sum_{j=0}^{+\infty} \sum_{l=0}^{k-2} \gamma^{kj + l + 1} \left( \Ppi \right)^{kj + l} \left( \Ppi - \Ppi[\delta] \right) \left(\POp \right)^{k-2-l} \EOp Q^\pi_k, \notag
	\end{align}
	where line~\eqref{p:401} follows from applying the Neumann series at the first factor, line~\eqref{p:402} is obtained by applying the first identity of Lemma~\ref{lemma:QFunctionDecomposition2} to the bounded measurable function $\EOp Q^\pi_k$. The subsequent lines are obtained by straightforward algebraic manipulations. Now we rename the indexes by setting $i = kj + l + 1$. Since $l \in \{0,\dots, k-2\}$ we have that $j = (i-1) \bdiv k$ and $l = (i-1) \bmod k$. Moreover, we observe that $i$ ranges over all non-negative integers values except for the multiples of the persistence $k$, \ie $i \in \{n \in \Nat\,:\, n \bmod k \neq 0\}$. Now, recalling that $i \bmod k \neq 0$, we observe that for the distributive property of the modulo operator we have $(i-1) \bmod k = \left(i \bmod k - 1 \bmod k \right) \bmod k = \left(i \bmod k - 1 \right) \bmod k = i \bmod k - 1$.
	The second identity is obtained by an analogous derivation in which we exploit the second identities at Lemmas~\ref{lemma:QFunctionDecomposition1} and~\ref{lemma:QFunctionDecomposition2}.
\end{proof}

\PersistenceBound*
\begin{proof}
	We start from the first equality derived in Lemma~\ref{lemma:PersistenceLemma}, and we apply the $L_p(\rho)$-norm both sides, with $p \ge 1$:
	\begin{align}
	\norm[p][\rho]{Q^\pi - Q^\pi_k}^p & = \norm[p][\rho]{\sum_{\substack{i \in \mathbb{N} \\ i \bmod k \neq 0}} \gamma^{i} \left( \Ppi \right)^{i-1} \left( \Ppi - \Ppi[\delta] \right) \left(\POp \right)^{k-2-(i-1) \bmod k} \EOp Q^\pi_k  }^p \notag\\
	& =  \rho \left| \sum_{\substack{i \in \mathbb{N} \\ i \bmod k \neq 0}} \gamma^{i} \left( \Ppi \right)^{i-1} \left( \Ppi - \Ppi[\delta] \right) \left(\POp \right)^{k-2-(i-1) \bmod k} \EOp Q^\pi_k \right|^p  \label{p:501}\\
	& \le \rho \left| \sum_{\substack{i \in \mathbb{N} \\ i \bmod k \neq 0}} \gamma^{i} \left( \Ppi \right)^{i-1} \sup_{f \in \mathcal{Q}_k}  \left| \left( \Ppi - \Ppi[\delta] \right) f  \right| \right|^p \label{p:502} \\
	& = \left(\frac{\gamma(1-\gamma^{k-1})}{(1-\gamma)(1-\gamma^k)}\right)^p \rho  \left|\frac{(1-\gamma)(1-\gamma^k)}{\gamma(1-\gamma^{k-1})}  \sum_{\substack{i \in \mathbb{N} \\ i \bmod k \neq 0}} \gamma^{i} \left( \Ppi \right)^{i-1} d_{\mathcal{Q}_k}^\pi \right|^p \label{p:503} \\
	& \le \left(\frac{\gamma(1-\gamma^{k-1})}{(1-\gamma)(1-\gamma^k)}\right)^p  \frac{(1-\gamma)(1-\gamma^k)}{\gamma(1-\gamma^{k-1})} \rho  \sum_{\substack{i \in \mathbb{N} \\ i \bmod k \neq 0}} \gamma^{i} \left( \Ppi \right)^{i-1} \left| d_{\mathcal{Q}_k}^\pi \right|^p \label{p:504}\\
	& = \left(\frac{\gamma(1-\gamma^{k-1})}{(1-\gamma)(1-\gamma^k)}\right)^p \eta^{\rho,\pi}_{k} \left| d_{\mathcal{Q}_k}^\pi \right|^p \label{p:505}\\
	& = \left(\frac{\gamma(1-\gamma^{k-1})}{(1-\gamma)(1-\gamma^k)}\right)^p \left\| d^\pi_{\mathcal{Q}_k} \right\|_{p, \eta^{\rho,\pi}}^p. \label{p:506}
%
	\end{align}
	where line~\eqref{p:501} is obtained by the definition of norm, written in the operator form, line~\eqref{p:502} is obtained by bounding $ \left( \Ppi - \Ppi[\delta] \right) \left(\POp \right)^{k-2-(i-1) \bmod k} \le \sup_{f \in \mathcal{Q}_k}  \left| \left( \Ppi - \Ppi[\delta] \right) f  \right|$, recalling the definition of $\mathcal{Q}_k$ and that $(i-1) \bmod k \le k-2$ for all $i \in \mathbb{N}$ and $i \bmod k \neq 0$. Then, line~\eqref{p:503} follows from deriving the normalization constant in order to make the summation $\sum_{\substack{i \in \mathbb{N} \\ i \bmod k \neq 0}} \gamma^{i} \left( \Ppi \right)^{i-1}$ a proper probability distribution. Such a constant can be obtained as follows:
	\begin{align*}
		\sum_{\substack{i \in \mathbb{N} \\ i \bmod k \neq 0}} \gamma^{i} = \sum_{i \in \mathbb{N}} \gamma^{i} - \sum_{i \in \mathbb{N}} \gamma^{ki} = \frac{\gamma(1-\gamma^{k-1})}{(1-\gamma)(1-\gamma^k)}.
	\end{align*}
	Line~\eqref{p:504} is obtained by applying Jensen inequality recalling that $p \ge 1$. Finally, line~\eqref{p:505} derives from the definition of the distribution $\eta^{\rho,\pi}_{k} $ and line~\eqref{p:506} from the definition of $L_p(\eta^{\rho,\pi}_{k})$-norm.
\end{proof}



\begin{lemma}\label{lemma:Lipoperators}
	Let $\mathcal{M}$ be an MDP and $\pi \in \Pi$ be a Markovian stationary policy. Let $f \in \mathscr{B}(\SAs)$ that is $L_f$--LC. Then, under Assumptions~\ref{ass:LipMDP} and~\ref{ass:LipPolicy}, the following statements hold:
	\begin{enumerate}[label=\roman*)]
		\item $ T^{\pi} f$ is $\left( L_r + \gamma L_P (L_{\pi} + 1) L_f \right)$--LC;
		\item $ T^{\delta} f$ is $\left( L_r + \gamma (L_{P} + 1) L_f \right)$--LC;
		\item $ T^{*} f$ is $\left( L_r + \gamma L_P  L_f \right)$--LC.
	\end{enumerate}
\end{lemma}

\begin{proof}
	Let $f\in \mathscr{B}(\SAs)$ be $L_f$-LC. Consider an application of $T^\pi$ and $(s,a),(\overline{s},\overline{a}) \in \SAs$:
	\begin{align}
		\left| (T^\pi f)(s,a) - (T^\pi f)(\overline{s},\overline{a}) \right| & = \left| r(s,a) + \gamma \int_{\Ss} \int_{\As} P(\de s'|s,a) \pi(\de a'|s') f(s',a') - r(\overline{s},\overline{a}) - \gamma \int_{\Ss} \int_{\As} P(\de s'|\overline{s},\overline{a}) \pi(\de a'|s') f(s',a')  \right| \notag \\
		& \le \left| r(s,a) - r(\overline{s},\overline{a}) \right| + \gamma \left| \int_{\Ss}  \left(P(\de s'|s,a) -  P(\de s'|\overline{s},\overline{a}) \right) \int_{\As} \pi(\de a'|s') f(s',a') \right| \label{p:801}\\
		& \le \left| r(s,a) - r(\overline{s},\overline{a}) \right| +  \gamma  (L_\pi+1) L_{f} \suplip \left| \int_{\Ss} \left( P(\de s'|s,a) -  P(\de s'|\overline{s},\overline{a}) \right) f(s') \right| \label{p:802}\\
		&  \le \left( L_r + \gamma L_P (L_{\pi} + 1) L_f \right) d_{\SAs}\left((s,a),(\overline{s},\overline{a}) \right), \label{p:803}
	\end{align}
	where line~\eqref{p:801} follows from triangular inequality, line~\eqref{p:802} is obtained from observing that the function $g_f(s') = \int_{\As} \pi(\de a'|s') f(s',a')$ is $(L_\pi+1)L_f$--LC, since for any $s,\overline{s} \in \Ss$:
	\begin{align}
		\left|g_f(s) - g_f(\overline{s})\right| & = \left| \int_{\As}\pi(\de a|s) f(s,a) - \int_{\As} \pi(\de a|\overline{s}) f(\overline{s},a) \right| \notag\\
		& = \left| \int_{\As}\pi(\de a|s) f(s,a) - \int_{\As} \pi(\de a|\overline{s}) f(\overline{s},a) \pm \int_{\As} \pi(\de a|\overline{s}) f(s,a) \right| \notag\\
		& \le \left| \int_{\As} \left( \pi(\de a|s) - \pi(\de a|\overline{s}) \right) f(s,a) \right| + \left| \int_{\As} \pi(\de a|\overline{s})  \left( f(\overline{s},a) - f(s,a) \right) \right| \notag \\
		& \le L_f \suplip \left| \int_{\As} \left( \pi(\de a|s) - \pi(\de a|\overline{s}) \right) f(a) \right| + \left| \int_{\As} \pi(\de a|\overline{s})  \left( f(\overline{s},a) - f(s,a) \right) \right|  \notag\\
		& \le L_f L_\pi d_{\Ss} (s,\overline{s}) + L_f d_{\Ss} (s,\overline{s}),\notag
	\end{align}
	where we exploited the fact that $L_{\pi}$--LC. 
	Finally, line~\eqref{p:803} is obtained by recalling that the reward function is $L_r$--LC and the transition model is $L_P$--LC. The derivations are analogous for $T^\delta$ and $T^*$. Concerning $T^\delta$ we have:
	\begin{align*}
	\left| (T^\delta f)(s,a) - (T^\delta f)(\overline{s},\overline{a}) \right| & \le \left| r(s,a) - r(\overline{s},\overline{a}) \right| + \gamma \left| \int_{\Ss} \int_{\As} \left(\delta_{a}(\de a') P(\de s'|s,a) -  \delta_{\overline{a}}(\de a')  P(\de s'|\overline{s},\overline{a}) \right)f(s',a')  \right|\\ 
	& \le L_r d_{\SAs}\left((s,a),(\overline{s},\overline{a}) \right) + \gamma  \left| \int_{\Ss} \left(P(\de s'|s,a) -  P(\de s'|\overline{s},\overline{a}) \right) \int_{\As} \delta_{a}(\de a') f(s',a')  \right|\\
	& \quad + \gamma \int_{\Ss} P(\de s'|\overline{s},\overline{a}) \left| \int_{\As} \left( \delta_{a}(\de a') - \delta_{\overline{a}}(\de a') \right) f(s',a') \right|	 \\
	& \le \left(L_r  + \gamma L_f L_P + \gamma L_f \right) d_{\SAs}\left((s,a),(\overline{s},\overline{a}) \right),
	\end{align*}
	where we observed that $\int_{\As} \delta_{a}(\de a') f(s',a') = f(s',a)$ is $L_f$--LC and that $\int_{\As} \left| \delta_{a}(\de a') - \delta_{\overline{a}}(\de a') \right| f(s',a') \le L_f d_{\As}(a,\overline{a}) \le L_f d_{\SAs}((s,a),(\overline{a},\overline{a}))$. Finally, considering $T^*$, we have:
	\begin{align*}
		\left| (T^* f)(s,a) - (T^* f)(\overline{s},\overline{a}) \right| & \le \left| r(s,a) - r(\overline{s},\overline{a}) \right| + \gamma \left| \int_{\Ss}  \left(P(\de s'|s,a) -  P(\de s'|\overline{s},\overline{a}) \right)\max_{a' \in As} f(s',a')  \right| \\
		& \le  \left(L_r  + \gamma L_f L_P \right) d_{\SAs}\left((s,a),(\overline{s},\overline{a}) \right),
	\end{align*}
	where we observed that the function $h_f(s') = \max_{a' \in As} f(s',a')$ is $L_f$--LC, since:
	\begin{align*}
		\left| h_f(s) - h_f(\overline{s}) \right| & = \left| \max_{a' \in As} f(s,a') - \max_{a' \in As} f(\overline{s},a') \right| \\
		& \le \max_{a' \in \As} \left| f(s,a') - f(\overline{s},a') \right| \\
		& \le L_f d_{\Ss} (s,\overline{s}).
	\end{align*}
\end{proof}

\begin{lemma}\label{lemma:DissimilatityToKant}
	Let $\mathcal{M}$ be an MDP and $\pi \in \Pi$ be a Markovian stationary policy. Then, under Assumptions~\ref{ass:LipMDP} and~\ref{ass:LipPolicy}, if $\gamma \max \{ L_P+1,  L_P(L_\pi+1) \} < 1$, the functions $f \in \mathcal{Q}_k$ are $L_{\mathcal{Q}_k}$--LC, where:
	\begin{equation}
		L_{\mathcal{Q}_k} \le \frac{L_r}{1-\gamma \max \{ L_P+1,  L_P(L_\pi+1) \}}.
	\end{equation}
	Furthermore, for all $(s,a) \in \SAs$ it holds that:
	\begin{equation}
		d_{\mathcal{Q}_k}(s,a) \le L_{\mathcal{Q}_k} \Kant \left(P^{\pi}(\cdot|s,a), P^\delta (\cdot|s,a) \right).
	\end{equation}
\end{lemma}

\begin{proof}
	First of all consider the action-value function of the $k$--persistent MDP $Q^{\pi}_k$, which is the fixed point of the operator $T^{\pi}_k$ that decomposes into $(T^\delta)^{k-1} T^\pi$ according to Theorem~\ref{thr:persistentOperators}. It follows that for any $f \in \mathscr{B}(\SAs)$ we have:
	\begin{equation*}
		Q^{\pi}_k = \lim_{j \rightarrow + \infty} \left( T^{\pi}_k \right)^j f = \lim_{j \rightarrow + \infty} \left( (T^\delta)^{k-1} T^\pi \right)^j f.
	\end{equation*}
	We now want to bound the Lipschitz constant of $Q^{\pi}_k$. To this purpose, let us first compute the Lipschitz constant of $T^\pi_k f = ((T^\delta)^{k-1} T^\pi) f$ for $f \in \mathscr{B}(\SAs)$ being an $L_f$--LC function. From Lemma~\ref{lemma:Lipoperators} we can bound the Lipschitz constant $a_h$ of $(T^\delta)^h T^\pi f$ for $h \in \{0,...k-1\}$, leading to the sequence:
	\begin{align*}
		a_h = \begin{cases}
				L_r + \gamma L_P(L_\pi+1) L_f & \text{if } h = 0 \\
				L_r + \gamma (L_P+1) a_{h-1} & \text{if } h \in \{1,...k-1\}
			\end{cases}.
	\end{align*}
	Thus, the Lipschitz constant of $((T^\delta)^{k-1} T^\pi) f$ is $a_{k-1}$. By unrolling the recursion we have:
	\begin{align*}
		a_{k-1} = L_r \sum_{i=0}^{k-1} \gamma^i (L_P+1)^i + \gamma^k L_P(L_\pi+1) (L_P+1)^{k-1}   L_f = L_r \frac{1- \gamma^k (L_P+1)^k}{1- \gamma(L_P+1)} + \gamma^k L_P(L_\pi+1) (L_P+1)^{k-1}   L_f.
	\end{align*}
	Let us now consider the sequence $b_j$ of the Lipschitz constants of $(T^\pi_k)^j  f$ for $j \in \Nat$:
	\begin{align*}
		b_j = \begin{cases}
				L_f & \text{if } j = 0 \\
				L_r \frac{1- \gamma^k (L_P+1)^k}{1- \gamma(L_P+1)} + \gamma^k L_P(L_\pi+1) (L_P+1)^{k-1}   b_{j-1} & \text{if } j \in \Nat[1]
			\end{cases}.
	\end{align*}
	The sequence $b_j$ converges to a finite limit as long as $\gamma^k L_P(L_\pi+1) (L_P+1)^{k-1} < 1$. In such case, the limit $b_{\infty}$ can be computed solving the fixed point equation:
	\begin{align*}
	b_{\infty} = L_r \frac{1- \gamma^k (L_P+1)^k}{1- \gamma(L_P+1)} + \gamma^k L_P(L_\pi+1) (L_P+1)^{k-1}   b_{\infty} \quad \implies \quad b_{\infty} = \frac{L_r \left(1- \gamma^k (L_P+1)^k \right)}{\left( 1- \gamma(L_P+1) \right) \left(1 -  \gamma^k L_P(L_\pi+1) (L_P+1)^{k-1} \right)}.
\end{align*}	 
Thus, $b_{\infty}$ represents the Lipschitz constant of $Q^{\pi}_k$.
%
%
%
%
%
It is worth noting that when setting $k=1$ we recover the Lipschitz constant of the $Q^{\pi}$ as in~\citep{rachelson2010locality}. To get a bound that is independent on $k$ we define $L = \max\{L_P(L_\pi+1), L_P+1\}$, assuming that $\gamma L < 1$ so that:
	\begin{align*}
	b_{\infty} = \frac{L_r \left(1- \gamma^k (L_P+1)^k \right)}{\left( 1- \gamma(L_P+1) \right) \left(1 -  \gamma^k L_P(L_\pi+1) (L_P+1)^{k-1} \right)} \le \frac{L_r }{1-\gamma L},
\end{align*}	 
having observed that $\frac{1- \gamma^k (L_P+1)^k }{1- \gamma(L_P+1)} \le \frac{1- \gamma^k L^k }{1- \gamma L} $. Thus, we conclude that $Q^\pi_k$ is also $ \frac{L_r }{1-\gamma L}$--LC for any $k \in \Nat[1]$. Consider now the application of the operator $T^{\pi}$ to $Q_k^\pi$, we have that the corresponding Lipschitz constant can be bounded by:
	\begin{equation}
		L_{T^{\pi} Q_k^\pi} \le L_r + \gamma L_{P}(L_\pi+1)  \frac{L_r }{1-\gamma L} \le L_r + \gamma L  \frac{L_r }{1-\gamma L} = \frac{L_r }{1-\gamma L}.
\end{equation}	 
A similar derivation holds for the application of $T^\delta$. As a consequence, any arbitrary sequence of applications of $T^\pi$ and $T^\delta$ to $Q^\pi_k$ generates a sequence of $\frac{L_r }{1-\gamma L}$--LC functions. Even more so for the functions in the set $\mathcal{Q}_k = \{  \left(\POp \right)^{k-2-l} \EOp Q^\pi_k \,:\, l \in \{0,\dots, k-2\} \}$. As a consequence, we can rephrase the dissimilarity term $d_{\mathcal{Q}_k}^\pi(s,a)$ as a Kantorovich distance:
\begin{align*}
	 d_{\mathcal{Q}_k}^\pi(s,a) &= \sup_{f \in \mathcal{Q}_k} \left| \int_{\Ss} \int_{\As} \left( \Ppi(\de s', \de a' | s,a) - P^\delta (\de s', \de a'| s,a) \right) f(s',a') \right| \\
	 & \le L_{\mathcal{Q}_k} \suplip\left| \int_{\Ss} \int_{\As} \left( \Ppi(\de s', \de a' | s,a) - P^\delta (\de s', \de a'| s,a) \right) f(s',a') \right| \\
	 & =  L_{\mathcal{Q}_k}\Kant \left(P^{\pi}(\cdot|s,a), P^\delta (\cdot|s,a) \right).
\end{align*}
\end{proof}

\corollTLC*
\begin{proof}
	Let us now consider the dissimilarity term in norm:
	\begin{align*}
		\left\| d_{\mathcal{Q}_k}^\pi \right\|_{p, \eta^{\rho,\pi}_{k}}^p & =  \int_{\Ss} \int_{\As}  \eta^{\rho,\pi}_{k}(\de s, \de a) \left| \sup_{f \in \mathcal{Q}_k} \left| \int_{\Ss} \int_{\As} \left( P^{\pi}(\de s', \de a' | s,a) - P^\delta(\de s', \de a' | s,a) \right) f(s',a') \right| \right|^p \\
		& \le L_{\mathcal{Q}_k}^p \int_{\Ss} \int_{\As}  \eta^{\rho,\pi}_{k}(\de s, \de a) \left| \suplip \left| \int_{\Ss} \int_{\As} \left( \Ppi(\de s', \de a' | s,a) - P^\delta (\de s', \de a'| s,a) \right) f(s',a') \right| \right|^p,
	\end{align*}
	where the inequality follows from Lemma~\ref{lemma:DissimilatityToKant}.
	We now consider the inner term and perform the following algebraic manipulations:
	\begin{align}
	\suplip & \left| \int_{\Ss} \int_{\As} \left( P^{\pi}(\de s', \de a' | s,a) - P^\delta(\de s', \de a' | s,a) \right) f(s',a') \right| \notag\\
	& = \suplip \bigg| \int_{\Ss} \int_{\As} P(\de s'|s,a)\pi(\de a' | s') f(s',a') - \int_{\Ss} \int_{\As} P(\de s'|s,a) \delta_a(\de a') f(s',a') \notag\\
	& \quad \pm  \int_{\Ss} \int_{\As} \delta_{s}(\de s')\pi(\de a' | s') \pm \int_{\Ss} \int_{\As} \delta_{s}(\de s') \delta_a(\de a') f(s',a') \bigg|  \notag\\
	& \le \suplip \left| \int_{\Ss} \left( P(\de s'|s,a) - \delta_s(\de s') \right) \int_{\As} \pi(\de a' | s') f(s', a') \right| \notag \\
	& \quad + \suplip \left| \int_{\Ss} \left( P(\de s'|s,a) - \delta_s(\de s') \right) \int_{\As} \delta_a(\de a') f(s', a') \right| \notag \\
	& \quad + \suplip \left| \int_{\Ss} \delta_s(\de s')  \int_{\As} \left( \pi(\de a'|s') - \delta_a(\de a')\right) f(s', a') \right|. \notag
\end{align}
We now consider the first two terms:
\begin{align}
	\suplip & \left| \int_{\Ss} \left( P(\de s'|s,a) - \delta_s(\de s') \right) \int_{\As} \pi(\de a' | s') f(s', a') \right| + \suplip \left| \int_{\Ss} \left( P(\de s'|s,a) - \delta_s(\de s') \right) \int_{\As} \delta_a(\de a') f(s', a') \right| \notag \\
	& \le  (L_{\pi} + 1 ) \Kant \left( P(\cdot|s,a) , \delta_s \right) \label{p:601} \\
	& \le (L_{\pi} + 1 ) L_T,\notag
\end{align}
where line~\eqref{p:601} follows from observing that the function $g_f(s') =  \int_{\As} \pi(\de a' | s') f(s', a')$ is $L_{\pi}$-LC, and function $h_f(s') = \int_{\As} \delta_a(\de a') f(s', a') = f(s',a)$ is 1-LC. Moreover, under Assumption~\ref{ass:TimeLipAss}, we have that $\Kant \left( P(\cdot|s,a) , \delta_s \right)  \le L_T$. Let us now focus on the third term:
\begin{align}
	\suplip & \left| \int_{\Ss} \delta_s(\de s')  \int_{\As} \left( \pi(\de a'|s') - \delta_a(\de a')\right) f(s', a') \right|  = \suplip \left| \int_{\As} \left( \pi(\de a'|s) - \delta_a(\de a')\right) f(s, a') \right| \notag \\
	& = \suplip \left| \int_{\As} \left( \pi(\de a'|s) - \delta_a(\de a')\right) f(a') \right|\label{p:610} \\
	& = \suplip \left| \int_{\As} \left( \int_{\As} \pi(\de a''|s) \delta_{a'}(\de a'') - \delta_a(\de a')\right) f(a') \right|\label{p:611} \\ 
	& = \suplip \left| \int_{\As}   \pi(\de a''|s) \int_{\As} \left(  \delta_{a''}(\de a') - \delta_a(\de a') \right) f(a') \right| \label{p:612}\\ 
	& \le  \int_{\As}   \pi(\de a''|s)  \suplip \left| \int_{\As} \left(  \delta_{a''}(\de a') - \delta_a(\de a') \right) f(a') \right| \label{p:613}\\ 
	& = \int_{\As}   \pi(\de a''|s) d_{\As}(a,a''),\label{p:614}
\end{align}
where line~\eqref{p:610} follows from observing that the dependence on $s$ for function $f$ can be neglected because of the supremum, line~\eqref{p:611} is obtained from the equality $ \pi(\de a'|s) =  \int_{\As} \pi(\de a''|s) \delta_{a'}(\de a'')$, line~\eqref{p:612} derives from moving the integral over $a''$ outside and recalling that $\delta_{a''}(\de a') = \delta_{a'}(\de a'') $, line~\eqref{p:613} comes from Jensen inequality. Finally, line~\eqref{p:614} is obtained from the definition of Kantorovich distance between Dirac deltas. Now, we take the expectation \wrt $\eta_k^{\rho,\pi}$. Recalling that $\rho(s,a) = \rho_{\Ss}(s) \pi(a|s)$ it follows that the same decomposition holds for $\eta_{k}^{\rho,\pi} (s,a) = \eta_{k, \Ss}^{\rho,\pi}(s) \pi(a|s)$. Consequently, exploiting the above equation, we have:
\begin{align*}
	\int_{\Ss} \eta_{k,\Ss}^{\rho, \pi}(\de s) \int_{\As} \pi(\de a |s) \left| \int_{\As}   \pi(\de a''|s) d_{\As}(a,a'') \right|^p & \le \int_{\Ss} {(\eta_k^{\rho, \pi})}_{\Ss}(\de s)  \int_{\As} \pi(\de a |s) \int_{\As}   \pi(\de a''|s)  d_{\As}(a,a'')^p \\
	& \le \sup_{s \in \Ss} \int_{\As}\int_{\As}   \pi(\de a |s)   \pi(\de a''|s)  d_{\As}(a,a'')^p = \sigma_p^p,
\end{align*}
where the first inequality follows from an application of Jensen inequality. 
An application of Minkowski inequality on the norm $\left\| d_{\mathcal{Q}_k}^\pi \right\|_{p, \eta^{\rho,\pi}_{k}}$ concludes the proof.
\end{proof}

\begin{lemma}\label{lemma:toVariance}
	If $\mathcal{A}= \mathbb{R}^{d_{\mathcal{A}}}$, $d_{\As}(\mathrm{\mathbf{a}},\mathrm{\mathbf{a}}') = \norm[2][]{\mathrm{\mathbf{a}} - \mathrm{\mathbf{a}}'}$, then it holds that $\sigma_2^2 \le2 \sup_{s \in \Ss} \Var[A]$, with $A \sim \pi(\cdot|s)$.
\end{lemma}

\begin{proof}
	Let $s \in \mathcal{S}$ and define the mean--action in state $s$ as:
	\begin{equation*}
		\overline{\mathrm{\mathbf{a}}}(s) = \int_{\As} \mathrm{\mathbf{a}} \pi(\de \mathrm{\mathbf{a}}|s).
	\end{equation*}
	Thus, we have:
	\begin{align*}
		\sigma_2^2 & = \sup_{s \in \Ss} \int_{\As} \int_{\As}  \left\| \mathrm{\mathbf{a}}-\mathrm{\mathbf{a}}' \right\|_2^2 \pi(\de \mathrm{\mathbf{a}}|s) \pi(\de \mathrm{\mathbf{a}}'|s) \\
		& = \sup_{s \in \Ss} \int_{\As} \int_{\As}  \left\| \mathrm{\mathbf{a}}-\mathrm{\mathbf{a}}' \pm  \overline{\mathrm{\mathbf{a}}}(s) \right\|_2^2 \pi(\de \mathrm{\mathbf{a}}|s) \pi(\de \mathrm{\mathbf{a}}'|s)\\
		& \le  \sup_{s \in \Ss} \int_{\As} \int_{\As}  \left\| \mathrm{\mathbf{a}}- \overline{\mathrm{\mathbf{a}}}(s) \right\|_2^2 \pi(\de \mathrm{\mathbf{a}}|s) \pi(\de \mathrm{\mathbf{a}}'|s) + \sup_{s \in \Ss} \int_{\As} \int_{\As}  \left\| \mathrm{\mathbf{a}}' - \overline{\mathrm{\mathbf{a}}}(s) \right\|_2^2 \pi(\de \mathrm{\mathbf{a}}|s) \pi(\de \mathrm{\mathbf{a}}'|s) \\
		& = \sup_{s \in \Ss}\int_{\As}  \left\| \mathrm{\mathbf{a}}- \overline{\mathrm{\mathbf{a}}}(s) \right\|_2^2 \pi(\de \mathrm{\mathbf{a}}|s) + \sup_{s \in \Ss} \int_{\As}  \left\| \mathrm{\mathbf{a}}' - \overline{\mathrm{\mathbf{a}}}(s) \right\|_2^2 \pi(\de \mathrm{\mathbf{a}}'|s) \\
		& = 2 \sup_{s \in \Ss}\int_{\As}  \left\| \mathrm{\mathbf{a}}- \overline{\mathrm{\mathbf{a}}}(s) \right\|_2^2 \pi(\de \mathrm{\mathbf{a}}|s) = 2 \sup_{s \in \Ss} \Var[A].
	\end{align*}
\end{proof}

\begin{remark}[On the choice of $d_{\mathcal{A}}$ when $|\As|<+\infty$] When the action space $\As$ is finite and it is a subset of a metric space (\eg $\mathbb{R}^{d_{\As}}$) we can employ the same metric as $d_{\mathcal{A}}$. Otherwise, we use the \emph{discrete metric} $d_{\As}(a,a') = \mathds{1} \{a \neq a'\}$ . 
\end{remark}

\subsection{Proofs of Section~\ref{sec:pfqi}}\label{apx:proofPFQI}

\begin{prop}\label{prop:complexity}
	Assuming that the evaluation of the estimated Q-function in a state action pair has computational complexity $\mathcal{O}(1)$, the computational complexity of $J$ iterations of \algname run with a dataset $\mathcal{D}$ of $n$ samples, neglecting the cost of the regression, is given by:
	\begin{equation*}
	\mathcal{O} \left( J n \left( 1 + \frac{|\As| - 1}{k} \right)\right).
	\end{equation*}
\end{prop}

\begin{proof}
	Let us consider an iteration $j =0,\dots,J-1$. If $j\bmod k = 0$, we perform an application of $\aOOp$ which requires to perform $n |\As|$ evaluations of the next-state value function in order to compute the maximum over the actions. On the contrary, when $j\bmod k \neq 0$, we perform an application of $\aPOp$ which requires just $n$ evaluations, since the next-state value function is evaluated in the persistent action only. By the definition of \algname, $J$ must be an integer multiple of the persistence $k$. Recalling that a single evaluation of the approximate Q-function is $\mathcal{O}(1)$, we have that the overall complexity is:
	\begin{equation*}
		\mathcal{O} \left( \sum_{j \in \{0,...,J-1\} \, \wedge \, j\bmod k = 0} n |\As| + \sum_{j \in \{0,...,J-1\} \, \wedge \, j\bmod k \neq 0} n \right) = \mathcal{O} \left( \frac{J}{k} n |\As| + \frac{J(k-1)}{k} n \right) = \mathcal{O} \left( J n \left( 1 + \frac{|\As| - 1}{k} \right)\right).
	\end{equation*}
\end{proof}

\errorProp*

Before proving the main result, we need to introduce a variation of the \emph{concentrability} coefficients~\cite{antos2008learning, farahmand2011regularization} to account for action persistence.

\begin{restatable}[Persistent Expected Concentrability]{defi}{}\label{defi:concentrab}
	Let $\rho, \nu \in \mathscr{P}(\SAs)$, $L \in \Nat[1]$, and an arbitrary sequence of stationary policies $(\pi^{(l)})_{l=1}^L$. Let $k \in \Nat[1]$ be the persistence. For any $m_1,m_2, m_3 \in \Nat[1]$ and $q \in [1,+\infty]$, we define:
	\begin{align*}
	& \cvA  (m_1,m_2,m_3;\pi) =  \E \Bigg[ \bigg| \frac{\de \big( \rho (P^{\pi}_k)^{m_1} (P^{\pi^*_k}_k)^{m_2} (P^\delta)^{m_3} \big) }{\de \nu} (S,A) \bigg|^\frac{q}{q-1} \Bigg] ^{\frac{q-1}{q}},\\
	& \cvB (m_1,m_2;(\pi^{(l)})_{l=1}^L) =   \E \Bigg[ \bigg| \frac{\de \big( \rho (P^{\pi^{(L)}}_k)^{m_1} P_k^{\pi^{(L-1)}} \dots P_k^{\pi^{(1)}} (P^\delta)^{m_2} \big) }{\de \nu} (S,A)\bigg|^\frac{q}{q-1} \Bigg]^{\frac{q-1}{q}}, 
	\end{align*}
with $(S,A) \sim \nu$. If $\rho (P^{\pi}_k)^{m_1} (P^{\pi^*_k}_k)^{m_2} (P^\delta)^{m_3} $ (resp. $ \rho (P^{\pi^{(L)}}_k)^{m_1} P_k^{\pi^{(L-1)}} \dots P_k^{\pi^{(1)}} (P^\delta)^{m_2}$) is not absolutely continuous \wrt to $\nu$, then we take $c_{\mathrm{VI}_{1},\rho,\nu}(m_1,m_2,m_3;\pi,k) = +\infty$ (resp. $c_{\mathrm{VI}_{2},\rho,\nu}(m_1,m_2;(\pi^{(l)})_{l=1}^L,k)= +\infty$).
\end{restatable}

This definition is a generalization of that provided in~\citet{farahmand2011regularization}, that can be recovered by setting $k=1$, $q=2$, $m_3 = 0$ for the first coefficient and $m_2=0$ for the second coefficient..

\begin{proof}
	The proof follows most of the steps of Theorem 3.4 of~\citet{farahmand2011regularization}. We start by deriving a bound relating $Q^* - Q^{(J)}$ to $(\epsilon^{(j)})_{j=0}^{J-1}$. To this purpose, let us first define the cumulative error over $k$ iterations for every $j \bmod k = 0$:
	\begin{equation}
		\epsilon_k^{(j)} = T^*_k Q^{(j)} - Q^{(j+k)}.
	\end{equation}
	Let us denote with $\pi^*_k$ one of the optimal policies of the $k$-persistent MDP $\mathcal{M}_k$. We have:
	\begin{equation*}
	\begin{aligned}
	& Q^*_k - Q^{(j+k)} = T^{\pi_k^*}_kQ^*_k - T^{\pi_k^*}_kQ^{(j)} + T^{\pi_k^*}_kQ^{(j)} - T^*_k Q^{(j)} + \epsilon_k^{(j)} \le \gamma^k P^{\pi^*_k}_k (Q^*_k - Q^{(j)}) + \epsilon_k^{(j)}, \\
	& Q^*_k - Q^{(j+k)} = T^{*}_k Q^*_k - T^{\pi^{(j)}}_k Q^* +  T^{\pi^{(j)}}_k Q^* - T^*_k Q^{(j)} + \epsilon_k^{(j)} \ge  \gamma^k P^{\pi^{(j)}}_k (Q^*_k - Q^{(j)}) + \epsilon_k^{(j)},
	\end{aligned}
	\end{equation*}
	where we exploited the fact that $T^*_k Q^{(j)} \ge T^{\pi_k^*}_k Q^{(j)}$, the definition of greedy policy $\pi^{(j)}$ that implies that $T_k^{\pi^{(j)}} Q^{(j)} = T_k^{*} Q^{(j)}$ and the definition of $\epsilon_k^{(j)}$. By unrolling the expression derived above, we have that for every $J \bmod k = 0$:
	\begin{equation}\label{p:701}
	\begin{aligned}
	& Q^*_k - Q^{(J)} \le \sum_{h=0}^{\frac{J}{k}-1} \gamma^{J - k(h+1)} \left( P_k^{\pi^*_k} \right)^{\frac{J}{k} - h - 1} \epsilon_k^{(j)} +  \gamma ^J \left( P_k^{\pi^*_k} \right)^{\frac{J}{k}} (Q^*_k - Q^{(0)}) \\
	& Q^*_k - Q^{(J)} \ge \sum_{h=0}^{\frac{J}{k}-1} \gamma^{J - k(h+1)} \left( P_k^{\pi^{(J-k)}} P_k^{\pi^{(J-2k)}} \dots P_k^{\pi^{(k(h+1))}} \right) \epsilon_k^{(j)} +  \gamma ^J \left( P_k^{\pi^{(J)}}  P_k^{\pi^{(J-k)}} \dots P_k^{\pi^{(k)}} \right) (Q^*_k - Q^{(0)}).
	\end{aligned}
	\end{equation}
	We now provide the following bound relating the difference $Q^*_k - Q_k^{\pi^{(J)}}$ to the difference $Q^*_k - Q^{(J)}$:
	\begin{align*}
	Q^*_k - Q_k^{\pi^{(J)}} & = T_k^{\pi_k^*} Q^*_k - T_k^{\pi_k^*} Q^{(J)} + T_k^{\pi_k^*} Q^{(J)} - T^*_k Q^{(J)} + T^*_k Q^{(J)} - T^{\pi^{(J)}}_k Q_k^{\pi^{(J)}} \\
	& \le T_k^{\pi_k^*} Q^*_k - T_k^{\pi_k^*} Q^{(J)}  + T^*_k Q^{(J)} - T^{\pi^{(J)}}_k Q_k^{\pi^{(J)}} \\
	& = \gamma^k P^{\pi^*_k}_k (Q^* - Q^{(J)}) + \gamma^k P_k^{\pi^{(J)}} (Q^{(J)} - Q^{\pi^{(J)}}_k) \\
	& = \gamma^k P^{\pi^*_k}_k (Q^* - Q^{(J)}) + \gamma^k P_k^{\pi^{(J)}} (Q^{(J)} - Q^*_k + Q^*_k - Q^{\pi^{(J)}}_k),
	\end{align*}
	where we exploited the fact that $T^*_k Q^{(J)} \ge T_k^{\pi_k^*} Q^{(J)} $ and observed that $T^*_k Q^{(J)} = T^{\pi^{(J)}}_k Q^{(J)} $. By using Lemma 4.2 of ~\citet{Munos2007performance}
	 we can derive:
	\begin{equation}\label{p:702}
	Q^*_k - Q^{\pi^{(J)}}_k \le \gamma^k \left( \Id - \gamma^k P_k^{\pi^{(J)}} \right)^{-1} \left(P^{\pi^*_k}_k - P^{\pi^{(J)}}_k \right)(Q^* - Q^{(J)}).
\end{equation}	 
By plugging Equation~\eqref{p:701} into Equation~\eqref{p:702}:
\begin{equation}\label{p:703}
\begin{aligned}
	Q^*_k - Q^{\pi^{(J)}}_k & \le \gamma^k \left( \Id - \gamma^k P_k^{\pi^{(J)}} \right)^{-1} \bigg[ \sum_{h=0}^{\frac{J}{k}-1} \gamma^{J - k(h+1)} \left( \left( P_k^{\pi^*_k} \right)^{\frac{J}{k} - h} -  \left(  P_k^{\pi^{(J)}} P_k^{\pi^{(J-k)}} P_k^{\pi^{(J-2k)}} \dots P_k^{\pi^{(k(h+1))}} \right) \right) \epsilon_k^{(j)} \\
	& \quad +  \gamma ^J \left( \left( P_k^{\pi^*_k} \right)^{\frac{J}{k} + 1} - \left( P_k^{\pi^{(J)}} P_k^{\pi^{(J)}} P_k^{\pi^{(J-k)}} \dots P_k^{\pi^{(k)}}  \right) \right) (Q^*_k - Q^{(0)}) \bigg].
	\end{aligned}
\end{equation}
Before proceeding, we need to relate the cumulative errors $\epsilon_k^{(j)}$ to the single-step errors $\epsilon^{(j)}$:
\begin{align*}
	\epsilon_k^{(j)} &= T^*_k Q^{(j)} - Q^{(j+k)} \\
	& = (\POp)^{k-1} \OOp Q^{(j)} - (\POp)^{k-1} Q^{(j+1)} +  (\POp)^{k-1} Q^{(j+1)} - Q^{(j+k)} \\
	& = \gamma^{k-1} (P^{\delta})^{k-1} \left( \OOp Q^{(j)} - Q^{(j+1)} \right) + (\POp)^{k-1} Q^{(j+1)} - Q^{(j+k)} \\
	& =  \gamma^{k-1} (P^{\delta})^{k-1} \epsilon^{(j)} + (\POp)^{k-1} Q^{(j+1)} - Q^{(j+k)}.
\end{align*}
Let us now consider the remaining term $(\POp)^{k-1} Q^{(j+1)} - Q^{(j+k)}$:
\begin{align*}
	(\POp)^{k-1} Q^{(j+1)} - Q^{(j+k)} &= (\POp)^{k-1} Q^{(j+1)} - (\POp)^{k-2} Q^{(j+2)} + (\POp)^{k-2} Q^{(j+2)} - Q^{(j+k)} \\
	& = \gamma^{k-2} (P^\delta)^{k-2} \left( \POp Q^{(j+1)} - Q^{(j+2)} \right) + (\POp)^{k-2} Q^{(j+2)} - Q^{(j+k)} \\
	& = \gamma^{k-2} (P^\delta)^{k-2} \epsilon^{(j+1)} + (\POp)^{k-2} Q^{(j+2)} - Q^{(j+k)} \\
	& = \sum_{l=2}^k \gamma^{k-l} (P^\delta)^{k-l} \epsilon^{(j+l-1)} ,
\end{align*}
where the last step is obtained by unrolling the recursion. Putting all together, we get:
\begin{equation}
\epsilon_k^{(j)} = \sum_{l=1}^k \gamma^{k-l} (P^\delta)^{k-l} \epsilon^{(j+l-1)}.
\end{equation}
Consequently, we can rewrite Equation~\eqref{p:703} as follows:
\begin{align}
	Q^*_k - Q^{\pi^{(J)}}_k & \le \gamma^k \left( \Id - \gamma^k P_k^{\pi^{(J)}} \right)^{-1} \Bigg[ \sum_{h=0}^{\frac{J}{k}-1} \gamma^{J - k(h+1)} \left( \left( P_k^{\pi^*_k} \right)^{\frac{J}{k} - h} -  \left(  P_k^{\pi^{(J)}} P_k^{\pi^{(J-k)}} P_k^{\pi^{(J-2k)}} \dots P_k^{\pi^{(k(h+1))}} \right) \right)  \notag  \\
	& \quad \times \sum_{l=1}^k \gamma^{k-l} (P^\delta)^{k-l} \epsilon^{(j+l-1)}+  \gamma ^J \left( \left( P_k^{\pi^*_k} \right)^{\frac{J}{k} + 1} - \left( P_k^{\pi^{(J)}} P_k^{\pi^{(J)}} P_k^{\pi^{(J-k)}} \dots P_k^{\pi^{(k)}}  \right) \right) (Q^*_k - Q^{(0)}) \Bigg] \notag \\
	& = \gamma^k \left( \Id - \gamma^k P_k^{\pi^{(J)}} \right)^{-1} \bigg[ \sum_{h=0}^{\frac{J}{k}-1} \sum_{l=1}^k  \gamma^{J - kh-l} \left( \left( P_k^{\pi^*_k} \right)^{\frac{J}{k} - h} -  \left(  P_k^{\pi^{(J)}} P_k^{\pi^{(J-k)}} P_k^{\pi^{(J-2k)}} \dots P_k^{\pi^{(k(h+1))}} \right) \right) \label{p:710}\\
	& \quad \times  (P^\delta)^{k-l} \epsilon^{(j+l-1)} +  \gamma ^J \left( \left( P_k^{\pi^*_k} \right)^{\frac{J}{k} + 1} - \left( P_k^{\pi^{(J)}} P_k^{\pi^{(J)}} P_k^{\pi^{(J-k)}} \dots P_k^{\pi^{(k)}}  \right) \right) (Q^*_k - Q^{(0)}) \bigg] \notag \\
	& = \gamma^k \left( \Id - \gamma^k P_k^{\pi^{(J)}} \right)^{-1} \bigg[ \sum_{j=0}^{J - 1} \gamma^{J - j - 1} \left( \left( P_k^{\pi^*_k} \right)^{\frac{J}{k} - j \bdiv k } -  \left(  P_k^{\pi^{(J)}} P_k^{\pi^{(J-k)}} P_k^{\pi^{(J-2k)}} \dots P_k^{\pi^{(J - k (j \bdiv k + 1))}} \right) \right) \notag \\
	& \quad \times (P^\delta)^{ k - j \bmod k - 1} \epsilon^{(j)}  +  \gamma ^J \left( \left( P_k^{\pi^*_k} \right)^{\frac{J}{k}+1} - \left( P_k^{\pi^{(J)}} P_k^{\pi^{(J)}} P_k^{\pi^{(J-k)}}  \dots P_k^{\pi^{(k)}}  \right) \right) (Q^*_k - Q^{(0)}) \bigg] \label{p:711}\\
	& \le \gamma^k \left( \Id - \gamma^k P_k^{\pi^{(J)}} \right)^{-1} \bigg[ \sum_{j=0}^{J - 1} \gamma^{J - j - 1} \left( \left( P_k^{\pi^*_k} \right)^{\frac{J}{k} - j \bdiv k } +  \left(  P_k^{\pi^{(J)}} P_k^{\pi^{(J-k)}} P_k^{\pi^{(J-2k)}} \dots P_k^{\pi^{(J - k (j \bdiv k + 1))}} \right) \right) \notag \\
	& \quad \times (P^\delta)^{ k - j \bmod k - 1} \left| \epsilon^{(j)} \right| +  \gamma ^J \left( \left( P_k^{\pi^*_k} \right)^{\frac{J}{k}+1} + \left( P_k^{\pi^{(J)}} P_k^{\pi^{(J)}} P_k^{\pi^{(J-k)}}  \dots P_k^{\pi^{(k)}}  \right) \right) \left|Q^*_k - Q^{(0)}\right| \bigg],\label{p:712}
\end{align}
where line~\eqref{p:710} derives from rearranging the two summations, line~\eqref{p:711} is obtained from a redefinition of the indexes. Specifically, we observed that $h = j \bdiv k$, $j+1=kh+l$, and $l = j \bmod k +1$. Finally, line~\eqref{p:712} is obtained by applying the absolute value to the right hand side and using Jensen inequality. We now introduce the following terms:
\begin{equation}
	A_j = \begin{cases}
		\frac{1-\gamma^k}{2} \left( \Id - \gamma^k P_k^{\pi^{(J)}} \right)^{-1} \left( \left( P_k^{\pi^*_k} \right)^{\frac{J}{k}-j\bdiv k} +  \left(  P_k^{\pi^{(J)}} P_k^{\pi^{(J-k)}} P_k^{\pi^{(J-2k)}} \dots P_k^{\pi^{(J - k (j \bdiv k+1))}} \right) \right) (P^\delta)^{ k - j \bmod k - 1} & \text{if } 0 \le j < J \\
		\frac{1-\gamma^k}{2} \left( \Id - \gamma^k P_k^{\pi^{(J)}} \right)^{-1} \left( \left( P_k^{\pi^*_k} \right)^{\frac{J}{k} + 1} + \left( P_k^{\pi^{(J)}} P_k^{\pi^{(J)}} P_k^{\pi^{(J-k)}}  \dots P_k^{\pi^{(k)}}  \right) \right) & \text{if } j = J
	\end{cases}.
\end{equation}
Let us recall the definition of $\alpha_j$ as in~\citet{farahmand2011regularization}:
\begin{equation}\label{eq:alphaj}
	\alpha_j = \begin{cases}
		\frac{(1-\gamma)\gamma^{J-j-1}}{1-\gamma^{J+1}} & \text{if } 0 \le j < J \\
		\frac{(1-\gamma) \gamma^J}{1-\gamma^{J+1}} & \text{if } j = J
	\end{cases}.
\end{equation}
Recalling that $\left|Q^*_k - Q^{(0)}\right| \le  Q_{\max} + \frac{R_{\max}}{1-\gamma} \le \frac{2R_{\max}}{1-\gamma} $ and applying Jensen inequality we get to the inequality:
\begin{align*}
	Q^*_k - Q^{\pi^{(J)}}_k   \le \frac{2 \gamma^k(1-\gamma^{J+1})}{(1-\gamma^k)(1-\gamma)} \left[ \sum_{j=0}^{J-1} \alpha_j A_j \left|\epsilon^{(j)} \right| + \alpha_J \frac{2R_{\max}}{1-\gamma} \mathrm{\mathbf{1}} \right],
\end{align*}
where $\mathrm{\mathbf{1}}$ denotes the constant function on $\SAs$ with value 1. Taking the $L_p(\rho)$--norm both sides, recalling that $\sum_{j=1}^J \alpha_j = 1$ and that the terms $A_j$ are positive linear operators $A_j : \mathscr{B}(\SAs) \rightarrow \mathscr{B}(\SAs)$ such that $A_j \mathrm{\mathbf{1}} = \mathrm{\mathbf{1}}$. Thus, by Lemma 12 of~\citet{antos2008learning}, we can apply Jensen inequality twice (once \wrt $\alpha_j$ and once \wrt $A_j$), getting:
\begin{align*}
\norm[p][\rho]{Q^*_k - Q^{\pi^{(J)}}_k }^p \le \left( \frac{2 \gamma^k(1-\gamma^{J+1})}{(1-\gamma^k)(1-\gamma)} \right)^p \rho \left[\sum_{j=0}^{J-1} \alpha_j A_j \left|\epsilon^{(j)} \right|^p + \alpha_J \left( \frac{2R_{\max}}{1-\gamma} \right)^p \mathrm{\mathbf{1}} \right].
\end{align*}
Consider now the individual terms $\rho A_j \left|\epsilon^{(j)} \right|^p$ for $0 \le j < J$. By the properties of the Neumann series we have:
\begin{align*}
	\rho A_j \left|\epsilon^{(j)} \right|^p & = \frac{1-\gamma^k}{2} \rho \left( \Id - \gamma^k P_k^{\pi^{(J)}} \right)^{-1} \left( \left( P_k^{\pi^*_k} \right)^{\frac{J}{k}-j \bdiv k} +  \left(  P_k^{\pi^{(J)}} P_k^{\pi^{(J-k)}} P_k^{\pi^{(J-2k)}} \dots P_k^{\pi^{(J - k (j \bdiv k + 1))}} \right) \right) \\
	& \quad \times (P^\delta)^{ k - j \bmod k - 1} \left|\epsilon^{(j)} \right|^p \\
	& = \frac{1-\gamma^k}{2} \rho \left[ \sum_{m=0}^{+\infty} \gamma^{km}  \left( \left(P_k^{\pi^{(J)}} \right)^{m} \left( P_k^{\pi^*_k} \right)^{\frac{J}{k}-j \bdiv k} +  \left(  \left(P_k^{\pi^{(J)}} \right)^{m+1} P_k^{\pi^{(J-k)}} P_k^{\pi^{(J-2k)}} \dots P_k^{\pi^{(J - k (j \bdiv k))}} \right) \right) \right] \\
	& \quad \times (P^\delta)^{ k - j \bmod k - 1} \left|\epsilon^{(j)} \right|^p. 
\end{align*}
We now aim at introducing the concentrability coefficients and for this purpose, we employ the following inequality. For any measurable function $f \in \mathscr{b}(\mathcal{X}) \rightarrow \mathbb{R}$, and the probability measures $\mu_1,\mu_2 \in \mathscr{P}(\mathcal{X})$ such that $\mu_2$ is absolutely continuous \wrt $\mu_1$, we have the following H\"older inequality, for any $q \in [1, +\infty]$:
\begin{equation}
	\int_\mathcal{X} f \de \mu_1 \le \left(\int_\mathcal{X} \left| \frac{\de \mu_1 }{\de \mu_2} \right|^\frac{q}{q-1} \de \mu_2  \right)^{\frac{q-1}{q}} \left( \int_\mathcal{X} |f|^q  \de \mu_2\right)^{\frac{1}{q}} .
\end{equation}
We now focus on a single term $ \rho \left(P_k^{\pi^{(J)}} \right)^{m} \left( P_k^{\pi^*_k} \right)^{\frac{J}{k}-j \bdiv k}  \left|\epsilon^{(j)} \right|^p$ and we apply the above inequality:
\begin{align*}
	\rho \left(P_k^{\pi^{(J)}} \right)^{m} \left( P_k^{\pi^*_k} \right)^{\frac{J}{k}-j \bdiv k}  (P^\delta)^{ k - j \bmod k - 1} \left|\epsilon^{(j)} \right|^p & \le \left(\int_{\SAs} \left| \frac{\de \rho \left(P_k^{\pi^{(J)}} \right)^{m} \left( P_k^{\pi^*_k} \right)^{\frac{J}{k}-j \bdiv k}  (P^\delta)^{ k - j \bmod k - 1}  }{\de \nu} \right|^\frac{q}{q-1}  \de \nu \right)^{\frac{q-1}{q}} \\
	& \quad \times \left(  \int_{\SAs} \left|\epsilon^{(j)} \right|^{pq} \de \nu\right)^{\frac{1}{q}} \\
	& = \cvA \left( m, \frac{J}{k}-j \bdiv k, k- j\bmod k - 1; \pi^{(J)} \right) \norm[pq][\nu]{\epsilon^{(j)}}^p.
\end{align*}
Proceeding in an analogous way for the remaining terms, we get to the expression:
\begin{align*}
\norm[p][\rho]{Q^*_k - Q^{\pi^{(J)}}_k }^p & \le \left( \frac{2 \gamma^k(1-\gamma^{J+1})}{(1-\gamma^k)(1-\gamma)} \right)^p \Bigg[ \frac{1-\gamma^k}{2} \sum_{j=0}^{J-1} \sum_{m=0}^{+\infty} \gamma^{km} \bigg( \cvA \left( m, \frac{J}{k}-j \bdiv k, k- j\bmod k - 1; \pi^{(J)} \right) \\
& \quad + \cvB \left( m+1, k- j\bmod k - 1  ; \{\pi^{(J-lk)}\}_{l=1}^{j \bdiv k} \right) \bigg)  \norm[pq][\nu]{\epsilon^{(j)}}^p  + \alpha_J \left(\frac{2R_{\max}}{1-\gamma} \right)^p  \Bigg].
\end{align*}
To separate the concentrability coefficients and the approximation errors, we apply H\"older inequality with $s \in [1,+\infty]$:
\begin{equation}
	\sum_{j=0}^J a_j b_j \le \left( \sum_{j=0}^J |a_j|^s \right)^{\frac{1}{s}}  \left( |b_j|^\frac{s}{s-1} \right)^{\frac{s-1}{s}}.
\end{equation}
Let $r \in [0,1]$, we set $a_j = \alpha_j^{r} \norm[pq][\nu]{\epsilon^{(j)}}^p$ and $b_j = \alpha_j^{1-r} \frac{1-\gamma^k}{2} \sum_{j=0}^{J-1} \sum_{m=0}^{+\infty} \gamma^{km} \bigg( \cvA \left( m, \frac{J}{k}-j \bdiv k, k- j\bmod k - 1; \pi^{(J)} \right)
+ \cvB \left( m+1, k- j\bmod k - 1  ; \{\pi^{(J-lk)}\}_{l=1}^{j \bdiv k} \right) \bigg)$. The application of H\"older inequality leads to:
\begin{align*}
\norm[p][\rho]{Q^*_k - Q^{\pi^{(J)}}_k }^p & \le \left( \frac{2 \gamma^k(1-\gamma^{J+1})}{(1-\gamma^k)(1-\gamma)} \right)^p  \frac{1-\gamma^k}{2} \Bigg[ \sum_{j=0}^{J-1} \alpha_j^{\frac{s(1-r)}{s-1}} \bigg( \sum_{m=0}^{+\infty} \gamma^{km} \bigg( \cvA \left( m, \frac{J}{k}-j \bdiv k, k- j\bmod k - 1; \pi^{(J)} \right) \\
& \quad + \cvB \left( m+1, k- j\bmod k - 1  ; \{\pi^{(J-lk)}\}_{l=1}^{j \bdiv k} \right) \bigg)  \bigg)^\frac{s}{s-1} \Bigg]^{\frac{s-1}{s}} \Bigg[ \sum_{j=0}^{J-1} \alpha_j^{sr}  \norm[pq][\nu]{\epsilon^{(j)}}^{sp}   \Bigg]^{\frac{1}{s}}\\
& \quad + \left( \frac{2 \gamma^k(1-\gamma^{J+1})}{(1-\gamma^k)(1-\gamma)} \right)^p \alpha_J  \left(\frac{2R_{\max}}{1-\gamma} \right)^p.
%
%
%
\end{align*}
Since the policies $(\pi^{(J-lk)})_{l=1}^{j \bdiv k}$ are not known, we define the following quantity by taking the supremum over any sequence of policies:
\begin{equation}\label{eq:CConcetrability}
\begin{aligned}
	C_{\mathrm{VI},\rho,\nu}(J;r,s,q) & = \left( \frac{1-\gamma^k}{2} \right)^s \sup_{\pi_0,...,\pi_{J} \in \Pi} \sum_{j=0}^{J-1} \alpha_j^{\frac{s(1-r)}{s-1}} \bigg( \sum_{m=0}^{+\infty} \gamma^{km} \bigg( \cvA \left( m, \frac{J}{k}-j \bdiv k, k- j\bmod k - 1; \pi_J \right) \\
& \quad + \cvB \left( m+1, k- j\bmod k - 1  ; \{\pi_l\}_{l=1}^{j \bdiv k} \right) \bigg)  \bigg)^\frac{s}{s-1}.
\end{aligned}
\end{equation}
Moreover, we define the following term that embeds all the terms related to the approximation error:
\begin{align}\label{eq:CalE}
	\mathcal{E}(\epsilon^{(0)}, \dots, \epsilon^{(J-1)} ;r,s,q) = \sum_{j=0}^{J-1} \alpha_j^{sr}  \norm[pq][\nu]{\epsilon^{(j)}}^{sp}.
\end{align}
Observing that $\frac{1-\gamma}{1-\gamma^{J+1}} \le 1$ and $1-\gamma^{J-1} \le 1$, we can put all together and taking the $p$--th root and recalling that the inequality holds for all $q \in [1,+\infty]$, $r \in [0,1]$, and $s \in [1, +\infty]$:
\begin{align*}
\norm[p][\rho]{Q^*_k - Q^{\pi^{(J)}}_k } & \le  \frac{2 \gamma^k}{(1-\gamma^k)(1-\gamma)} \left[ \inf_{\substack{q \in [1,+\infty] \\ r \in [0,1] \\ s \in [1, +\infty]}}  C_{\mathrm{VI},\rho,\nu}(J;r,s,q)^{\frac{s-1}{ps}} \mathcal{E}(\epsilon^{(0)}, \dots, \epsilon^{(J-1)} ;r,s,q)^{\frac{1}{ps}} +  \gamma^{\frac{J}{p}} \frac{2R_{\max}}{1-\gamma}\right].
\end{align*}
The statement is simplified by taking $s = 2$.

\end{proof}


\subsection{Proofs of Section~\ref{sec:PersistenceSelection}}
\lowerBoundQ*
\begin{proof}
	We start by providing the following equality, recalling that $T^*_k Q = T^\pi_k Q$, being $\pi$ the greedy policy \wrt $Q$:
	\begin{align*}
		Q^{\pi}_k - Q &= T^\pi_k Q^\pi_k - T^\pi_k Q + T^*_k Q  - Q \\
			&  = \gamma^k P^\pi_k \left( Q^\pi_k - Q \right) +  T^*_k Q  - Q \\
			& = \left( \Id - \gamma^k P^\pi_k \right)^{-1} \left(  T^*_k Q  - Q \right),
	\end{align*}
	where the last equality follows from the properties of the Neumann series. We take the expectation \wrt to the distribution $\rho \pi$ both sides. For the left hand side we have:
		\begin{equation*}
			J^{\rho,\pi}_k - J^{\rho}= \rho \pi Q^{\pi}_k - \rho \pi Q.
		\end{equation*}
		Concerning the right hand side, instead, we have:
		\begin{equation*}
		 \rho \pi \left( \Id - \gamma^k P^\pi_k \right)^{-1}  \left(  T^*_k Q  - Q \right) = \frac{1}{1-\gamma^k} \eta^{\rho,\pi} \left(  T^*_k Q  - Q \right),
		\end{equation*}
		where we introduced the $\gamma$--discounted stationary distribution~\citep{sutton2000policy} after normalization. Putting all together, we can derive the following inequality:
		\begin{align*}
		J^{\rho,\pi}_k - J^{\rho} & = \frac{1}{1-\gamma^k} \eta^{\rho,\pi} \left(  T^*_k Q  - Q \right) \\
		& \ge -\frac{1}{1-\gamma^k} \eta^{\rho,\pi} \left|  T^*_k Q  - Q \right| \\
		& = -\frac{1}{1-\gamma^k} \norm[1][\eta^{\rho,\pi}]{T^*_k Q  - Q}.
		\end{align*}
\end{proof}

\section{Details on Bounding the Performance Loss (Section~\ref{sec:loose})}\label{apx:details}
In this appendix, we report some additional material that is referenced in Section~\ref{sec:loose}, concerning the performance loss due to the usage of action persistence.

\subsection{Discussion on the Persistence Bound (Theorem~\ref{thr:PersistenceBound})}
\label{sec:discussionOnBoundd}

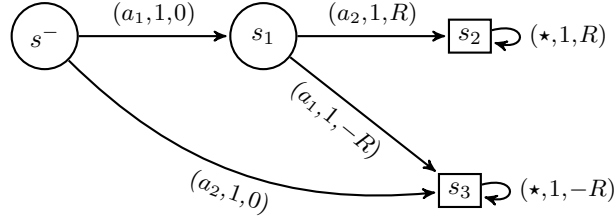
\begin{figure}[!ht]
    \centering
%
%
    \begin{tikzpicture}[->,>=stealth',shorten >=1pt,thick, node distance=2cm]
  \tikzstyle{every state}=[fill=white,draw=black,text=black, align=center]
  \tikzset{rect state/.style={draw,rectangle}}

  \node[state] (A) {$s^-$};
  \node[state] (B) [right=of A] {$s_1$};
  \node[rect state] (D) [right=2cm of B] {$s_2$};
  \node[rect state] (E) [below right=1.5cm and 2cm of B] {$s_3$};
  
  \path (A) edge node (AB) [inner sep=0mm,pos=0.2] {} node [above, pos=0.5] {\small $(a_1,1,0)$} (B) ;
  \path (B) edge node (BD) [inner sep=0mm,pos=0.2] {} node [above, pos=0.5] {\small $(a_2,1,R)$} (D) ;
  \path (B) edge  node (BE) [inner sep=0mm,pos=0.2] {} node [below, pos=0.4, rotate=-39] {\small $(a_1,1,-R)$} (E) ;
  \path (D) edge [loop right] node (DD) [inner sep=0mm,pos=0.2] {} node [right, pos=0.5] {\small $(\star,1,R)$} (D) ;
  \path (E) edge [loop right] node (EE) [inner sep=0mm,pos=0.2] {} node [right, pos=0.5] {\small $(\star,1,-R)$} (E) ;
   \path (A) edge[bend left=-25] node (AB) [inner sep=0mm,pos=0.2] {} node [below, pos=0.5, rotate=-20] {\small $(a_2,1,0)$} (E);
\end{tikzpicture}
\caption{The MDP counter-example of Proposition~\ref{prop:negative}, where $R>0$. Each arrow connecting two states $s$ and $s'$ is labeled with the 3-tuple $(a,P(s'|s,a),r(s,a))$; the symbol $\star$ denotes any action in $\As$. While the optimal policy in the original MDP starting in $s^-$ can avoid negative rewards by executing an action sequence of the kind $(a_1, a_2, \dots)$, every policy in the $k$-persistent MDP, with $k \in \Nat[2]$, inevitably ends in the negative terminal state, as the only possible action sequences are of the kind $(a_1, a_1, \dots)$ and $(a_2, a_2, \dots)$.}\label{fig:CounterExample}
\end{figure}

We start with a negative result, showing that with no structure it is possible to make the bound of Theorem~\ref{thr:PersistenceBound} vacuous, and thus, independent from $k$.

\begin{restatable}[]{prop}{negative}\label{prop:negative}
	For any MDP $\mathcal{M}$ and $k \in \mathbb{N}_{\ge 2}$ it holds that:
	\begin{equation}
		 V^*_k(s) \ge V^*(s) - \frac{2 \gamma R_{\max}}{1-\gamma}, \; \quad \forall s \in \Ss.
	\end{equation}	
	Furthermore, there exists an MDP $\mathcal{M}^-$ (Figure~\ref{fig:CounterExample}) and a state $s^- \in \Ss$ such that the bound holds with equality for all $k \in \mathbb{N}_{\ge 2}$.
\end{restatable}

\begin{proof}
First of all, we recall that $V^*(s) - V^*_k(s) \ge 0$ since we cannot increase performance when executing a policy with a persistence $k$. Let $\pi^*$ an optimal policy on the MDP $\mathcal{M}$, we  observe that for all $s \in \Ss$:
    \begin{equation}
       V^*(s) - V^*_k(s) \le V^{\pi^*}(s) - V^{\pi^*}_k(s),
    \end{equation}
    since $V^{\pi^*}(s) = V^*(s)$ and $V^*_k(s) \ge V^{\pi^*}_k(s)$. Let us now consider the corresponding Q-functions $Q^{\pi^*}(s,a)$ and $Q^{\pi^*}_k(s,a)$. Recalling that they are the fixed points of the Bellman operators $\EOp[\pi^*]$ and $\EOp[\pi^*]_k$ we have:
   \begin{align*}
   		Q^{\pi^*} - Q^{\pi^*}_k & =  \EOp[\pi^*] Q^{\pi^*}  -  \EOp[\pi^*]_k Q^{\pi^*}_k  \\
   		& = r + \gamma P^{\pi}  Q^{\pi^*} - r_k - \gamma^k P_k^\pi Q^{\pi^*}_k \\
   		& = r + \gamma P^{\pi}  Q^{\pi^*} - \sum_{i=0}^{k-1} \gamma^i \left(P^{\delta} \right)^i r - \gamma^k P_k^\pi Q^{\pi^*}_k \\
   		& = \gamma P^{\pi}  Q^{\pi^*} - \sum_{i=1}^{k-1} \gamma^i \left(P^{\delta} \right)^i r - \gamma^k P_k^\pi Q^{\pi^*}_k, \\
   	\end{align*}
   	where we exploited the definitions of the Bellman expectation operators in the $k$-persistent MDP. As a consequence, we have that for all $(s,a) \in \SAs$:
   	\begin{align*}
   		Q^{\pi^*}(s,a) - Q^{\pi^*}(s,a) & \le \gamma \frac{R_{\max}}{1-\gamma} + R_{\max} \sum_{i=1}^{k-1} \gamma^i +  \gamma^k \frac{R_{\max}}{1-\gamma} \\
   		& = \gamma \frac{R_{\max}}{1-\gamma} + R_{\max}\frac{\gamma(1-\gamma^{k-1})}{1-\gamma} +  \gamma^k \frac{R_{\max}}{1-\gamma} = \frac{2\gamma R_{\max}}{1-\gamma},
   		\end{align*}
   		where we considered the following facts that hold for all $(s,a) \in \SAs$: $\left(P^{\pi}  Q^{\pi^*} \right)(s,a) \le \frac{R_{\max}}{1-\gamma}$, $\left(\left(P^{\delta} \right)^i r\right)(s,a) \le R_{\max}$, and $\left( P_k^\pi Q^{\pi^*}_k \right) \le R_{\max}$. The result follows, by simply observing that $V^{\pi^*}(s) - V^{\pi^*}_k(s) = \E \left[ Q^{\pi^*}(s,A) - Q^{\pi^*}(s,A) \right]$, where $A \sim \pi^*(\cdot|s)$.
   		
We now prove that the bound is tight for the MDP of Figure~\ref{fig:CounterExample}. From inspection, we observe that the optimal policy must reach the terminal state $s_2$ yielding the positive reward $R > 0$. Thus the optimal policy plays action $a_1$ in state $s^-$ and action $a_2$ in state $s_1$, generating a value function $V^*(s^-) = \frac{\gamma R}{1-\gamma}$. Let us now consider the $2$-persistent MDP $\mathcal{M}^-_2$. Whichever action is played in state $s^-$ it is going to be persisted for the subsequent decision epoch and, consequently, we will end up in state $s_3$, yielding the negative reward $-R < 0$. Thus, the optimal value function will be $V^*_2(s^-) = -\frac{\gamma R}{1-\gamma}$. Clearly, the same rationale holds for any persistence $k \in \Nat[3]$.
\end{proof}

The quantity $\frac{2 \gamma R_{\max}}{1-\gamma}$ is the maximum performance that we can lose if we perform the same action at decision epoch $t=0$ and then we follow an arbitrary policy thereafter.

\subsection{On using divergences other than the Kantorovich}
The Persistence Bound presented in Theorem~\ref{thr:PersistenceBound} is defined in terms
of the dissimilarity index $d_{\mathcal{Q}_k}^\pi$ which depends on the set of functions $\mathcal{Q}_k $ defined in terms of the $k$-persistent Q-function $Q^\pi_k$ and in terms of the Bellman operators $\EOp$ and $\POp$. Clearly, this bound is meaningful when it yields a value that is smaller than $\frac{2\gamma R_{\max}}{1-\gamma}$ that we already know to be the maximum performance degradation we experience when executing policy $\pi$ with persistence (Proposition~\ref{prop:negative}). Therefore, for any meaningful choice of $\mathcal{Q}_k $, we require that, at least for $k=2$, the following condition to hold:
\begin{equation}\label{eq:ConditionValidBound}
	 \frac{\gamma (1-\gamma^{k-1})}{(1-\gamma)(1-\gamma^k)} \left\| d^\pi_{\mathcal{Q}_k} \right\|_{p, \eta^{\rho,\pi}_{k}} \bigg\rvert_{k=2} = \frac{\gamma}{(1-\gamma^2)} \left\| d^\pi_{\mathcal{Q}_2} \right\|_{p,\eta^{\rho,\pi}_{2}} < \frac{2\gamma R_{\max}}{1-\gamma}.
\end{equation}
If we require no additional regularity conditions on the MDP, we can only exploit the fact that all functions $f \in \mathcal{Q}_k $ are uniformly bounded by $\frac{R_{\max}}{1-\gamma}$, reducing $d^\pi_{\mathcal{Q}_k}$ to the total variation distance between $P^\pi$ and $P^\delta$:
\begin{equation}
	d^\pi_{\mathcal{Q}_k}(s,a) \le \frac{R_{\max}}{1-\gamma} \sup_{f : \norm[\infty][]{f} \le 1}  \left| \int_{\Ss} \int_{\As} \left( \Ppi(\de s', \de a' | s,a) - P^\delta (\de s', \de a'| s,a) \right) f(s',a') \right| = \frac{2 R_{\max}}{1-\gamma} d_{\text{TV}}^\pi(s,a).
\end{equation}

We restrict our discussion to deterministic policies and, for this purpose, we denote with $\pi(s) \in \As$ the action prescribed by policy $\pi$ in the state $s \in \Ss$. Thus, the total variation distance as follows:
\begin{align*}
d_{\text{TV}}^\pi(s,a) &= \frac{1}{2} \int_{\Ss} \int_{\As} \left| P^\pi(\de s', \de a' |s,a) -P^\delta(\de s', a' |s,a) \right| \\
& =  \frac{1}{2} \int_{\Ss} P(\de s' |s,a) \int_{\As} \left| \pi( \de a' | s') -\delta_{a}( \de a') \right|  \\
& = \frac{1}{2} \int_{\Ss} P(\de s' |s,a) \int_{\As} \left| \delta_{\pi(s')}(a')-\delta_{\pi(s)}( \de a') \right| \\
& = \int_{\Ss} P(\de s' |s,a) \mathds{1}_{\{\pi(s) \neq \pi(s')\}},
\end{align*}
where $\mathds{1}_{\mathcal{X}}$ denotes the indicator function for the measurable set $\mathcal{X}$. Consequently, we can derive for the norm:
\begin{align*}
\left\| d^\pi_{\mathcal{Q}_2} \right\|_{p,\eta^{\rho,\pi}_{2}}^p & \le \frac{2 R_{\max}}{1-\gamma} \int_{\Ss} \int_{\As} \eta_{k}^{\rho,\pi}(\de s, \de a) \left|\int_{\Ss} P(\de s' |s,a) \mathds{1}_{\{\pi(s) \neq \pi(s')\}} \right|^p \\
& \le \frac{2 R_{\max}}{1-\gamma}  \int_{\Ss} \int_{\As} \eta_{k}^{\rho,\pi}(\de s, \de a) \int_{\Ss} P(\de s' |s,a) \left|\mathds{1}_{\{\pi(s) \neq \pi(s')\}} \right|^p \\
& =  \frac{2 R_{\max}}{1-\gamma} \int_{\Ss} \int_{\As} \eta_{k}^{\rho,\pi}(\de s, \de a) \int_{\Ss} P(\de s' |s,a) \mathds{1}_{\{\pi(s) \neq \pi(s')\}}. 
\end{align*}
Thus, such term depend on the expected fraction of state-next-state pairs such that their policies prescribe different actions. Consequently, considering the condition at Equation~\eqref{eq:ConditionValidBound}, we have that it must be fulfilled:
\begin{equation*}
	\int_{\Ss} \int_{\As} \eta_{k}^{\rho,\pi}(\de s, \de a) \int_{\Ss} P(\de s' |s,a) \mathds{1}_{\{\pi(s) \neq \pi(s')\}} \le 1-\gamma^2.
\end{equation*}
However, if for every state-next-state pair the prescribed actions are different (even if very similar in some metric space), the left hand side would be 1 and the inequality never satisfied. To embed the notion of closeness of actions we need to resort to distance metrics different from the total variation (\eg the Kantorovich). These considerations can be extended to 
the case of stochastic policies.

\subsection{Time--Lipschitz Continuity for dynamical systems}\label{ex:dynamicalSystem}
We now draw a connection between the rate at which a dynamical system evolves and the $L_T$ constant of Assumption~\ref{ass:TimeLipAss}.	Consider a continuous-time dynamical system having $\mathcal{S}=\Reals^{d_{\mathcal{S}}}$ and $\mathcal{A}= \Reals^{{d}_{\mathcal{A}}}$ governed by the law $\dot{\mathbr{s}}(t) = \mathbr{f}(\mathbr{s}(t),\mathbr{a}(t))$ such that $\sup_{\mathbr{s} \in \Ss, \mathbr{a} \in \As}\left\| \mathbr{f}(\mathbr{s},\mathbr{a}) \right\| \le F < +\infty$. Suppose to control the system with a discrete time step $\Delta t_0 > 0$, inducing an MDP with transition model $P_{\Delta t_0}$. Using a norm $\left\| \cdot \right\|$, Assumption~\ref{ass:TimeLipAss} becomes: 
	\begin{align*}
		\Kant \left( P_{\Delta t_0}(\cdot|\mathbr{s},\mathbr{a}) , \delta_{\mathbr{s}} \right) & = \left\|\mathbr{s}(t+\Delta t_0) - \mathbr{s}(t)  \right\| \\
		& = \left\| \int_{t}^{t+\Delta t_0} \dot{\mathbr{s}}(\de t) \right\| \le F \Delta t_0.
	\end{align*}
	Thus, the Time Lipschitz constant $L_T$ depends on: i) how fast the dynamical system evolves ($F$); ii) the duration of the control time step ($\Delta t_0$).
	
\subsection{Discussion on Conditions of Theorem~\ref{thr:corollTLC}}\label{apx:conditionsThr42}
In order to bound the dissimilarity term $\left\| d^\pi_{\mathcal{Q}_k} \right\|_{p, \eta^{\rho,\pi}_{k}}$ we require in Theorem~\ref{thr:corollTLC} that $ \max\left\{ L_P + 1, L_P (1+L_\pi) \right\} < \frac{1}{\gamma}$. This condition can be decomposed in the two conditions: (i) $ L_P + 1 < \frac{1}{\gamma}$ and ii) $  L_P (1+L_\pi) < \frac{1}{\gamma}$. While (ii) inherits from the Lipschitz MDP literature with Wasserstein metric~\cite{rachelson2010locality}, condition i) is typical of action persistence. In principle, we could replace Wasserstein with Total Variation, getting less restrictive conditions~\citep[][Section 7]{munos2008finite} but this would rule out deterministic systems. Moreover, the Lipschitz constants are a bound, derived to separate the effects of $\pi$ and $P$, as commonly done in the literature. Tighter bounds can be obtanied if we consider the Lipschitz constants of the joint transition models $P^{\pi}$ and $P^\delta$. Indeed, lookning at the proof of Lemma~\ref{lemma:Lipoperators} we immediately figure out that:
\begin{equation}\label{eq:boundLipConst}
	L_{P^{\pi}} \le L_{P}(L_{\pi}+1), \qquad L_{P^\delta} \le L_{P}+1.
\end{equation}

To clarify the point, consider the following deterministic dynamical linear system with $\mathcal{S}=\mathbb{R}^{d_{\mathcal{S}}}$ controlled via a deterministic linear policy with $\mathcal{A}=\mathbb{R}^{d_{\mathcal{A}}}$:
\begin{align*}
	& \mathbf{s}_{t+1} = \mathbf{A} \mathbf{s}_{t} + \mathbf{B} \mathbf{a}_{t}, \\
	& \mathbf{a}_{t} = \mathbf{K} \mathbf{s}_{t},
\end{align*}
where $\mathbf{A}$, $\mathbf{B}$, and $\mathbf{K}$ are properly sized matrices. Let us now compute $L_{P^{\pi}}$ and $L_{P^\delta}$ and the corresponding bounds of Equation~\eqref{eq:boundLipConst}. To this purpose we use as metric $d_{\SAs}((\mathbf{s},\mathbf{a}),(\overline{\mathbf{s}},\overline{\mathbf{a}})) = \| \mathbf{s}-\overline{\mathbf{s}}\| +  \| \mathbf{a}-\overline{\mathbf{a}}\|$:
\begin{align*}
\mathcal{W}_1 \left( P^\pi(\cdot|\mathbf{s},\mathbf{a}), P^\pi(\cdot|\overline{\mathbf{s}},\overline{\mathbf{a}}) \right) & \le \left\| \mathbf{A} (\mathbf{s}-\overline{\mathbf{s}}) + \mathbf{B} (\mathbf{a}-\overline{\mathbf{a}}) \right\| + \left\| \mathbf{KA} (\mathbf{s}-\overline{\mathbf{s}}) + \mathbf{KB} (\mathbf{a}-\overline{\mathbf{a}}) \right\| \\
& \le (\|\mathbf{KA}\| + \|\mathbf{A}\|) \left\| \mathbf{s}-\overline{\mathbf{s}}\right\|  + (\|\mathbf{KB}\| + \|\mathbf{B}\|) \left\| \mathbf{a}-\overline{\mathbf{a}}\right\|,
\end{align*}
\begin{align*}
\mathcal{W}_1 \left( P^\delta(\cdot|\mathbf{s},\mathbf{a}), P^\delta(\cdot|\overline{\mathbf{s}},\overline{\mathbf{a}}) \right) & \le \left\| \mathbf{A} (\mathbf{s}-\overline{\mathbf{s}}) + \mathbf{B} (\mathbf{a}-\overline{\mathbf{a}}) \right\| + \left\|\mathbf{a}-\overline{\mathbf{a}} \right\| \\
& \le  \|\mathbf{A}\| \left\| \mathbf{s}-\overline{\mathbf{s}}\right\|  + (\|\mathbf{B}\| + 1)\left\| \mathbf{a}-\overline{\mathbf{a}}\right\|,
\end{align*}
leading to $L_{P^\pi} \le \max \left\{ \|\mathbf{KA}\| + \|\mathbf{A}\| ,  \|\mathbf{KB}\| + \|\mathbf{B}\| \right\}$ and $L_{P^\delta} \le \max \left\{  \|\mathbf{A}\| ,  \|\mathbf{B}\| + 1 \right\}$. If instead, we compute the corresponding bounds of Equation~\eqref{eq:boundLipConst}, we have:
\begin{align*}
\mathcal{W}_1 \left( P(\cdot|\mathbf{s},\mathbf{a}), P(\cdot|\overline{\mathbf{s}},\overline{\mathbf{a}}) \right) & \le \left\| \mathbf{A} (\mathbf{s}-\overline{\mathbf{s}}) + \mathbf{B} (\mathbf{a}-\overline{\mathbf{a}}) \right\| \le \|\mathbf{A}\| \left\| \mathbf{s}-\overline{\mathbf{s}}\right\|  +  \|\mathbf{B}\| \left\| \mathbf{a}-\overline{\mathbf{a}}\right\|,
\end{align*}
\begin{align*}
\mathcal{W}_1 \left( \pi(\cdot|\mathbf{s}), \pi(\cdot|\overline{\mathbf{s}}) \right) & \le \left\| \mathbf{K} (\mathbf{s}-\overline{\mathbf{s}}) \right\| \le \| \mathbf{K} \| \left\| \mathbf{s}-\overline{\mathbf{s}} \right\|,
\end{align*}
leading to $L_{P} \le \max\{\|\mathbf{A}\|,\|\mathbf{B}\|\}$ and $L_{\pi} \le  \| \mathbf{K} \|$ and, consequently, $L_{P}(L_{\pi+1}) \le \max\{\|\mathbf{A}\|,\|\mathbf{B}\|\}( \| \mathbf{K} \| +1)$ and $L_P+1 \le  \max\{\|\mathbf{A}\|,\|\mathbf{B}\|\}+1$. Clearly, these latter results induce more restrictive conditions for certain values of $\mathbf{A}$, $\mathbf{B}$, and $\mathbf{K}$. Nevertheless, we believe that the bounds of Equation~\eqref{eq:boundLipConst} are unavoidable in the general case.
 
\section{Details on Persistence Selection (Section~\ref{sec:PersistenceSelection})}\label{apx:discussionSimplification}
In this appendix, we illustrate some details behind the simplifications of Lemma~\ref{thr:lowerBoundQ} to get the persistence selection index $B_k$.

\subsection{Change of Distribution}\label{apx:changeDistribution}
We discuss intuitively the effects of replacing the distribution $\eta^{\rho,\pi}$ with the sampling distribution $\nu$. To this purpose, we consider the particular case in which $\nu$ is the $\gamma$-discounted stationary distribution obtained by running a sampling policy $u$ in the environment and using the same $\rho$ as initial state distribution. Therefore, we can state:
\begin{align*}
 & \eta^{\rho,\pi} = (1-\gamma^k) \rho \pi \left( \Id - \gamma^k P^\pi_k \right)^{-1} = (1-\gamma^k) \sum_{i=0}^{\infty} \gamma^{ki} \rho \pi\left(P^{\pi}_k \right)^i, \\ 
 & \nu = (1-\gamma) \rho \pi \left( \Id - \gamma P^u \right)^{-1} = (1-\gamma) \sum_{i=0}^{\infty} \gamma^{i} \rho u \left(P^{u} \right)^i.
\end{align*}

There are two main differences between $\eta^{\rho,\pi}$ and $ \nu$. First, $\eta^{\rho,\pi}$ a discounted stationary distribution in the $k$-persistent MDP, while $\nu$ is the sampling distribution and thus, it is defined in the original (1-persistent) MDP. Second, while $\eta^{\rho,\pi}$ comes from the execution of the policy  $\pi$ obtained after a certain number iterations of learning, $\nu$ is derived by the execution of the sampling policy $u$. To decouple the effects stated above, let us define the following auxiliary discounted stationary distributions:
\begin{align*}
 & \eta^{\rho,\pi}_1 = (1-\gamma) \rho \pi \left( \Id - \gamma P^\pi \right)^{-1} , \\
 & \nu_k = (1-\gamma^k) \rho \pi \left( \Id - \gamma^k P^u_k \right)^{-1}.
\end{align*}
Thus, $\eta^{\rho,\pi}_1$ is obtained by executing policy $\pi$ in the original (1-persistent) MDP, while $\nu_k$ comes from the execution of $u$ in the $k$-persistent MDP. Therefore, we can provide the following two decomposition of $\left\| \frac{\eta^{\rho,\pi}}{\nu} \right\|_{\infty}$:
\begin{align*}
& \left\| \frac{\eta^{\rho,\pi}}{\nu} \right\|_{\infty} = \left\| \frac{\eta^{\rho,\pi}}{\eta^{\rho,\pi}_1} \frac{\eta^{\rho,\pi}_1}{\nu} \right\|_{\infty} \le \left\| \frac{\eta^{\rho,\pi}}{\eta^{\rho,\pi}_1} \right\|_{\infty} \left\| \frac{\eta^{\rho,\pi}_1}{\nu} \right\|_{\infty},\\
& \left\| \frac{\eta^{\rho,\pi}}{\nu} \right\|_{\infty} = \left\| \frac{\eta^{\rho,\pi}}{\nu_k} \frac{\nu_k}{\nu} \right\|_{\infty} \le \left\| \frac{\eta^{\rho,\pi}}{\nu_k} \right\|_{\infty} \left\| \frac{\nu_k}{\nu} \right\|_{\infty}.
\end{align*}
Therefore, looking at the first decomposition, we observe that in order to keep $\left\| \frac{\eta^{\rho,\pi}}{\nu} \right\|_{\infty}$ small we can require the following two conditions. First, executing the same policy $\pi$ at persistence $k$ and 1 must induce similar discounted stationary distributions, \ie $\left\| \frac{\eta^{\rho,\pi}}{\eta^{\rho,\pi}_1} \right\|_{\infty} \simeq 1$. This is a condition related to persistence only and connected, in some sense, to the regularity conditions employed in Section~\ref{sec:loose} to bound the loss induced by action persistence. Second, executing policy $\pi$ or policy $u$ in the same 1-persistent MDP must induce similar $\gamma$-discounted stationary distributions, \ie $\left\| \frac{\eta^{\rho,\pi}_1}{\nu} \right\|_{\infty} \simeq 1$. This condition, instead, depends on the similarity between policies $\pi$ and $u$ and on the properties of the transition model. Clearly, an analogous rationale holds when focusing on the second decomposition. We leave as future work the derivation of more formal conditions to bound the magnitude of $\left\| \frac{\eta^{\rho,\pi}}{\nu} \right\|_{\infty}$.

\subsection{Estimating the Expected Bellman Residual}
Once we have an approximation $\widetilde{Q}_k$ of $T^*_k Q$ obtained with the regressor $\mathtt{Reg}$, we can proceed to the decomposition, thanks to the triangular inequality:
\begin{equation}
\norm[1][\nu]{T^*_k Q - Q } \le \norm[1][\nu]{ \widetilde{Q}_k - Q }  + \norm[1][\nu]{ T^*_kQ - \widetilde{Q}_k }.
\end{equation}
As discussed in~\citet{farahmand2011model}, simply using $\norm[1][\nu]{ \widetilde{Q}_k - Q } $ as a proxy for $\norm[1][\nu]{T^*_k Q - Q }$ might be overlay optimistic. To overcome this problem we must prevent the underestimation of the expected Bellman residual. The idea proposed in~\citet{farahmand2011model} consists in replacing the regression error $\norm[1][\nu]{ T^*_kQ - \widetilde{Q}_k }$ with a high--probability bound $b_{k,\mathcal{G}}$, depending on the functional space $\mathcal{G}$ of the chosen regressor $\mathtt{Reg}$. Clearly, we have the new problem of getting a meaningful bound $b_{k,\mathcal{G}}$. This issue is treated in Section 7.4 of~\citet{farahmand2011model}. If $\mathcal{G}$ is a \emph{small} functional space, \ie with finite pseudo--dimension, we can employ a standard learning theory bound~\citep{gyorfi2002a}. Since for the persistence selection we employ the same functional space $\mathcal{G}$ and the same number of samples $m$ for all persistences $k \in \mathcal{K}$, the value of such a bound will not depend on $k$ and, therefore, it can be neglected in the optimization process. We stress that our goal is to provide a practical method to have an idea on which is a reasonable persistence to employ.

\clearpage
\section{Details on Experimental Evaluation (Section~\ref{sec:experimental})}\label{apx:Experiments}
In this appendix, we report the details about our experimental setting (Appendix~\ref{apx:experimentaldetails}), together with additional plots (Appendix~\ref{apx:additionalPlots}) and an experiment investigating the effect of the batch size when using persistence (Appendix~\ref{apx:batch_dep}).

\subsection{Experimental Setting}\label{apx:experimentaldetails}
Table~\ref{tab:PFQI_experiments} reports the parameters of the experimental setting, which are described in the following.

\textbf{Infrastructure}~~The experiments have been run on a machine with two CPUs Intel(R) Xeon(R) CPU E7-8880 v4 @ 2.20GHz (22 cores, 44 thread, 55 MB cache) and 128 GB RAM.

\textbf{Environments}~~The implementation of the environments are the ones provided in Open AI Gym~\citep{brockman2016open} \url{https://gym.openai.com/envs/}.

\textbf{Action Spaces}~~For the environments with finite action space, we collect samples with a uniform policy over $\As$; whereas for the environments with a continuous action space, we perform a discretization, reported in the column \quotes{Action space}, and we employ the uniform policy over the resulting finite action space. 

\textbf{Sample Collection}~~Samples are collected in the base MDP at persistence 1, although for some of them the uniform policy is executed at a higher persistence, $k_{\text{sampling}}$, reported in the column \quotes{Sampling Persistence}. Using a persistence greater than 1 to generate samples has been fundamental in some cases (\eg Mountain Car) to get a better exploration of the environment and improving the learning performances.\footnote{When considering a sampling persistence $k_{\text{sampling}} > 1$, we record in the dataset all the intermediate repeated actions, so that the tuples $(S_t,A_t,S'_t,R_t)$ are transitions of the base MDP $\mathcal{M}$.} 

\textbf{Number of Iterations}~~In order to perform a complete application of a $k$-Persisted Bellman Operator in the PFQI algorithm, we need $k$ iterations, so the total number of iterations needed to complete the training must be an integer multiple of $k$. In order to compare the resulting performances, we chose the persistences as a range of powers of 2. The total number of iterations $J$ is selected empirically so that the estimated Q-function has reached convergence for all tested persistences.

\textbf{Time Discretization}~~Every environment has its own way to deal with time discretization. In some cases, in order to make the benefits of persistence evident, we needed to reduce the base control timestep of the environment \wrt to the original implementation. We report in the column \quotes{Original timestep} ($\Delta t_{\text{original}}$) the control timestep in the original implementation of the environment, while the base time step ($\Delta t_0$) is obtained as a fraction of $\Delta t_{\text{original}}$. The reduction of the timestep by a factor $m = \Delta t_{\text{original}} / \Delta t_0 $ results in an extension of the horizon of the same factor, hence there is a greater number of rewards to sum, with the consequent need of a larger discount factor to maintain the same \quotes{effective horizon}. Thus, the new horizon $H$ (resp. discount factor $\gamma$) can be determined starting from the original horizon $H_{\text{original}}$ (resp. original discount factor $\gamma_{\text{original}}$) as:
\begin{equation*}
	H =  m H_{\text{original}}, \qquad \gamma = \left(\gamma_{\text{original}}\right)^\frac{1}{m} , \qquad \text{where} \quad m = \frac{\Delta t_{\text{original}}}{\Delta t_0}.
\end{equation*} 

\textbf{Regressor Hyperparameters}~~We used the class \emph{ExtraTreesRegressor} in the \textit{scikit-learn} library~\citep{scikit-learn} with the following parameters: n\_estimators = 100,	min\_samples\_split = 5, and min\_samples\_leaf = 2.

\setlength{\tabcolsep}{1.8pt}
\renewcommand{\arraystretch}{1.2}
\begin{table*}[h!]
\caption{Parameters of the experimental setting, used for the \algname experiments.}
\label{tab:PFQI_experiments}
\footnotesize
\vskip 0.15in
\begin{center}
\begin{tabular}{L{2cm} C{1.8cm} C{1.5cm} C{2cm} C{1.5cm} C{1.5cm} C{2cm} C{1.5cm} C{1.5cm}}
\toprule
Environment &
Action space $\As$ & 
Sampling Persistence $k_{\text{sampling}}$ & 
Original timestep $\Delta t_{\text{original}}$ (sec)& 
Factor $m = \Delta t_{\text{original}} /\Delta t_{0}$ &
 Original Horizon $H_{\text{original}}$ & 
Original Discount factor $\gamma_{\text{original}}$ & 
 Batch size $n$ & 
Iterations $J$\\
\midrule
\footnotesize Cartpole & $\{-1,1\}$ & 1 & 0.02 & 4 & 128 & 0.99 & 400 & 512 \\
\footnotesize Mountain Car & $\{-1,0, 1\}$ & 8 & 1 & 2 & 128 & 0.99 & 20 & 256 \\
\footnotesize Lunar Lander & \{Nop, left, main, right\} & 1 & $0.02$  & 1 &  256 & 0.99 & 100 & 256 \\
\footnotesize Pendulum & $\{-2,0,2\}$ & 1 & 0.05 & 1 & 256 & 0.99 & 100 & 64 \\
\footnotesize Acrobot & $\{-1,0,1\}$ & 4 & 0.2 & 4 & 128 & 0.99 & 200 & 512 \\
\footnotesize Swimmer & $\{-1,0,1\}^2$ & 1 & 2 (frame-skip) & 2 & 128 & 0.99 & 100 & 128 \\
\footnotesize Hopper & $\{-1,0,1\}^3$ & 1 & 1 (frame-skip) & 2 & 128 & 0.99 & 100 & 128 \\
\footnotesize Walker 2D & $\{-1,0,1\}^9$ & 1 & 1 (frame-skip) & 2 & 128 & 0.99 & 100 & 128 \\
%
\bottomrule
\end{tabular}
\end{center}
\vskip -0.1in
\end{table*}

\clearpage

\subsection{Additional Plots}\label{apx:additionalPlots}
\begin{figure}[h!]
\raggedleft
\centerline{Cartpole}\vspace{.2cm}
\subcaptionbox*{}{\includegraphics[scale=1.18]{./plot_cartpole}}
\centerline{Mountain Car}\vspace{.2cm}
\subcaptionbox*{}{\includegraphics[scale=1.18]{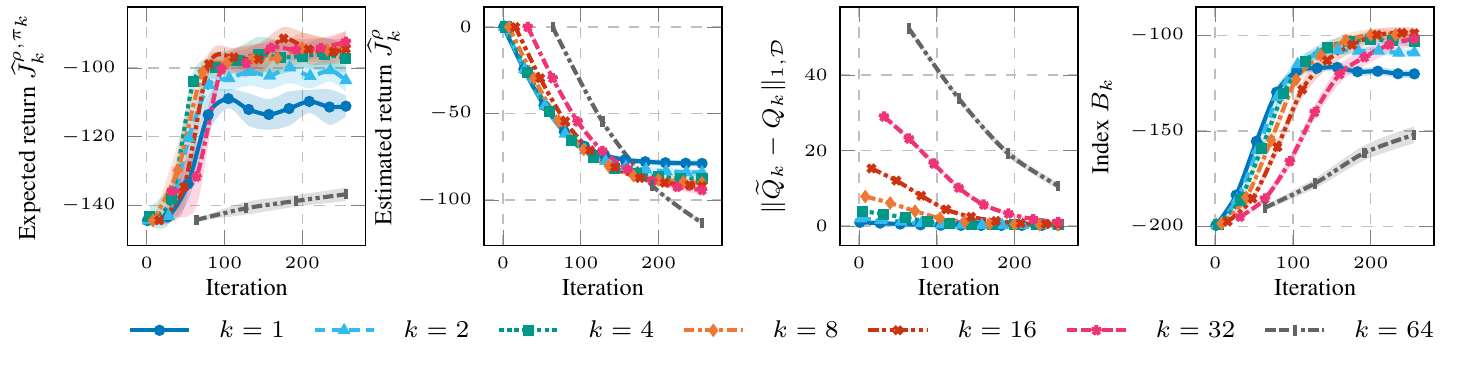}}
\centerline{Lunar Lander}\vspace{.2cm}
\subcaptionbox*{}{\includegraphics[scale=1.18]{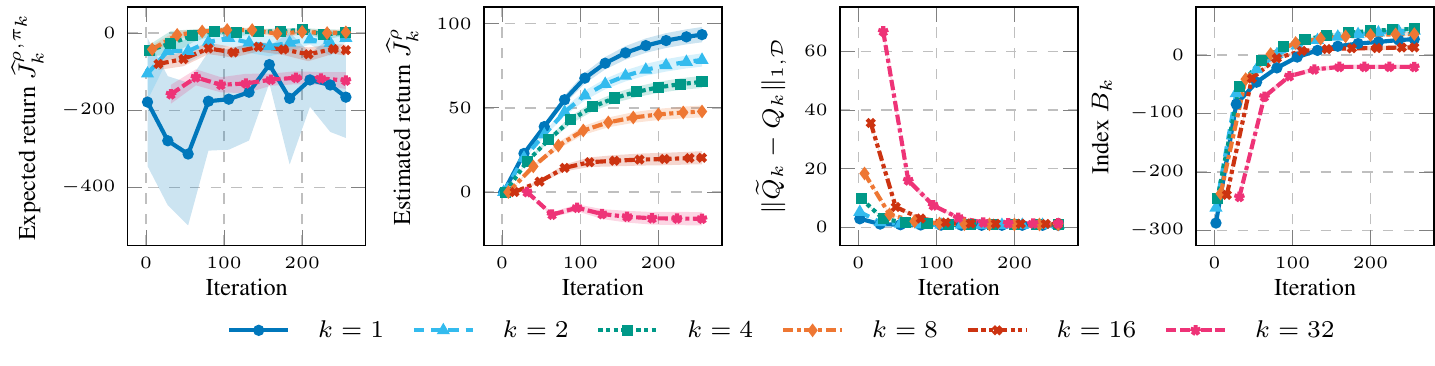}}
\centerline{Pendulum}\vspace{.2cm}
\subcaptionbox*{}{\includegraphics[scale=1.18]{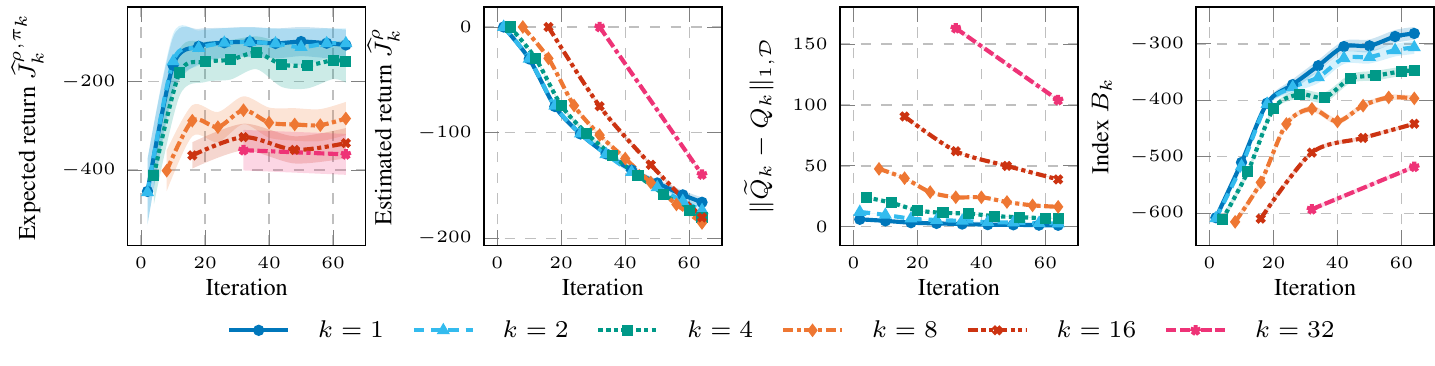}}
\end{figure}

\addtocounter{figure}{-1}
\begin{figure}[h!]
\raggedleft
\centerline{Acrobot}\vspace{.2cm}
\subcaptionbox*{}{\includegraphics[scale=1.18]{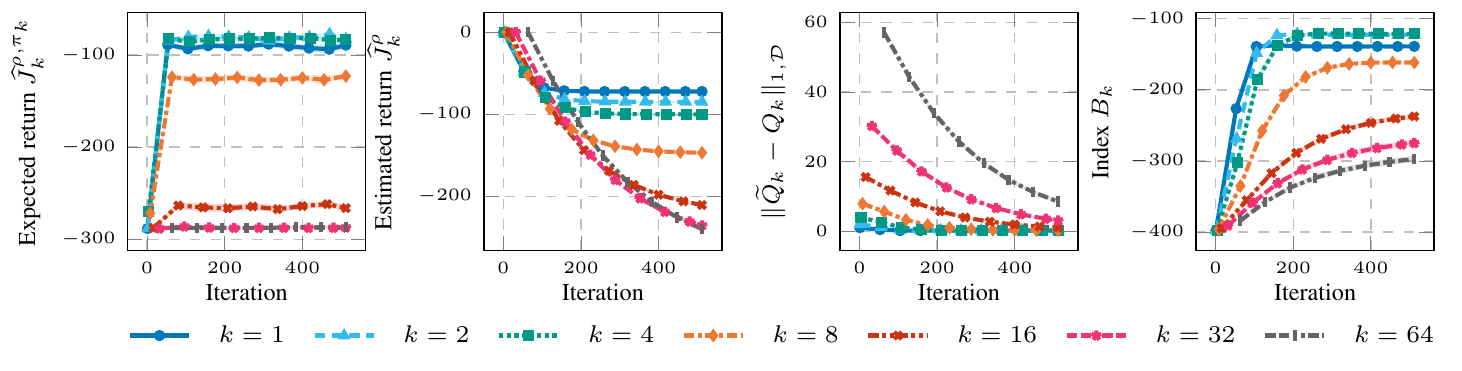}}
\centerline{Swimmer}\vspace{.2cm}
\subcaptionbox*{}{\hspace{.3cm}\includegraphics[scale=1.18]{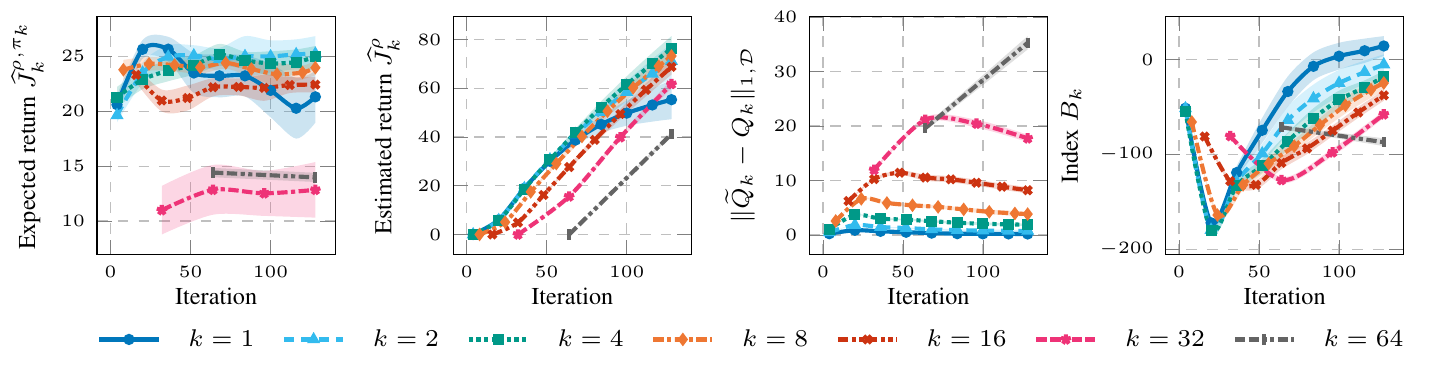}}
\centerline{Hopper}\vspace{.2cm}
\subcaptionbox*{}{\hspace{.3cm}\includegraphics[scale=1.18]{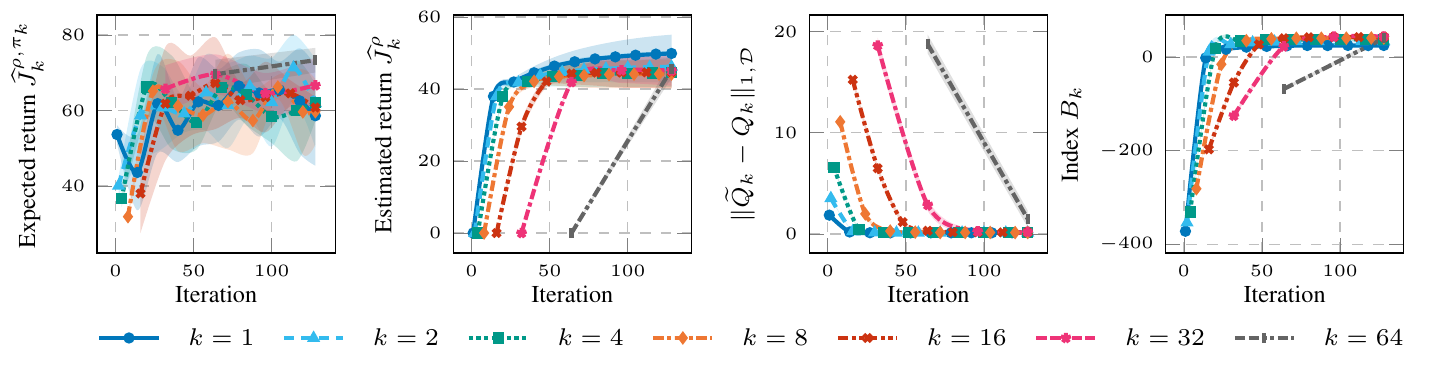}}
\centerline{Walker 2D}\vspace{.2cm}
\subcaptionbox*{}{\hspace{.3cm}\includegraphics[scale=1.18]{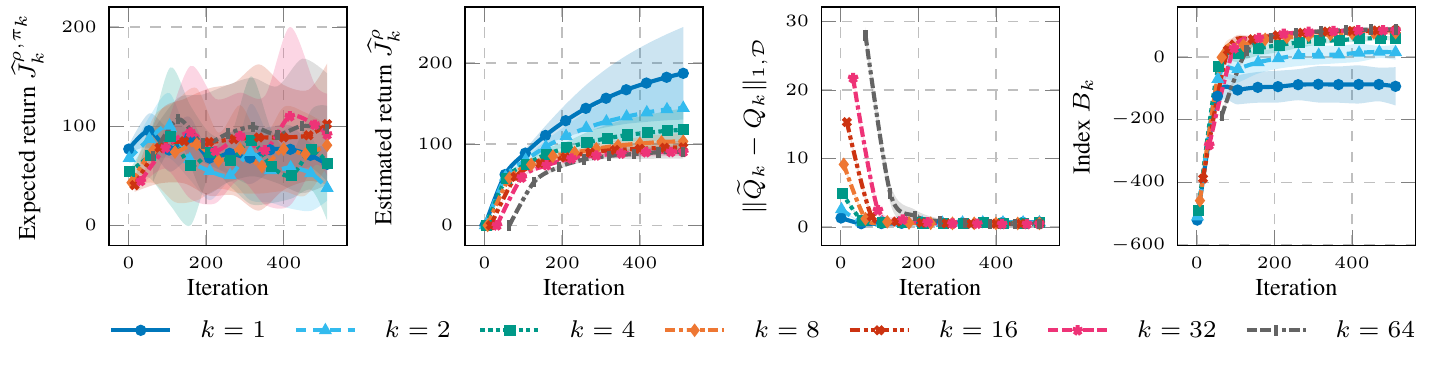}}\vspace{-.5cm}
\caption{Expected return $\widehat{J}_k^{\rho,\pi_k}$, estimated return $\widehat{J}_k^{\rho}$, estimated expected Bellman residual $\| \widetilde{Q}_k - Q_k \|_{1,\mathcal{D}}$, and persistence selection index $B_k$ for the different experiments as a function of the number of iterations for different persistences. 20 runs, 95 \% c.i.}
\end{figure}

\clearpage
\subsection{PFQI with Neural Network as regressor}
In the previous experiments we employed extra-trees as regressor to run PFQI. In this appendix, we investigate the effect of employing a neural network as regressor. More specifically, we consider a two-layer network with 64 neurons each and ReLU activation. Figure~\ref{fig:resFigNN} and Table~\ref{tab:resultsNN} show the results. The experimental setting is identical to that presented in Appendix~\ref{apx:experimentaldetails}. Although the performances are overall lower compared to the case of extremely randomized trees, we notice the same trade-off in the choice of the persistence.

\begin{figure}[h!]
\raggedleft
\centerline{Cartpole}\vspace{.2cm}
\subcaptionbox*{}{\includegraphics[scale=1.18]{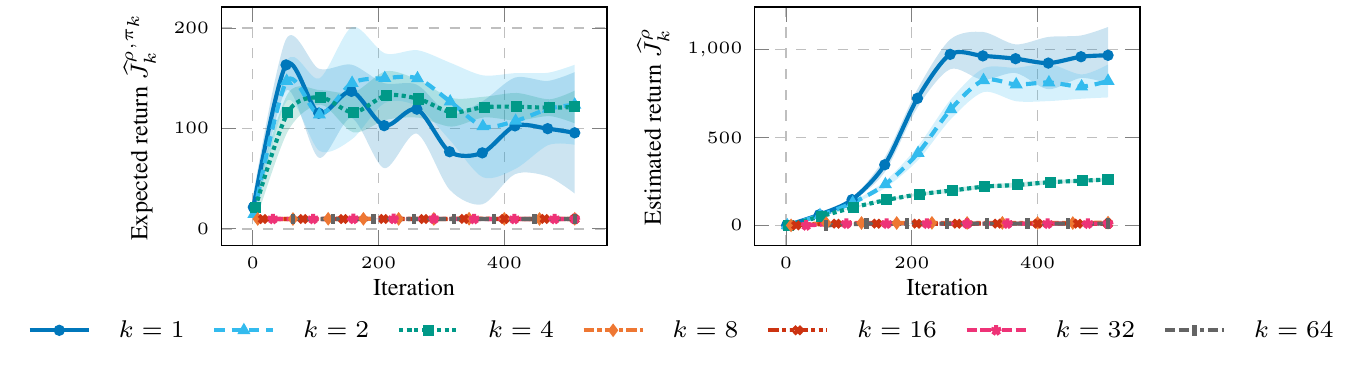}}
\centerline{Lunar Lander}\vspace{.2cm}
\subcaptionbox*{}{\includegraphics[scale=1.18]{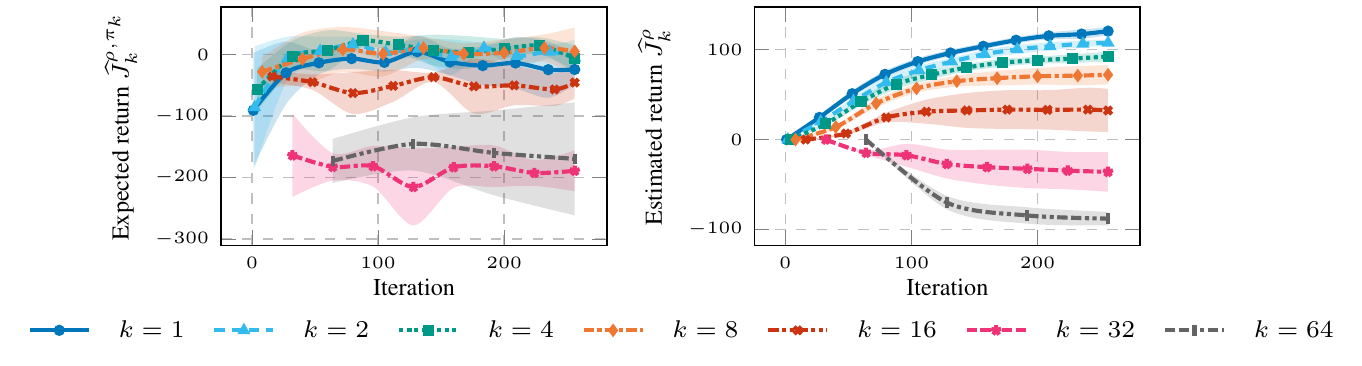}}
\centerline{Acrobot}\vspace{.2cm}
\subcaptionbox*{}{\includegraphics[scale=1.18]{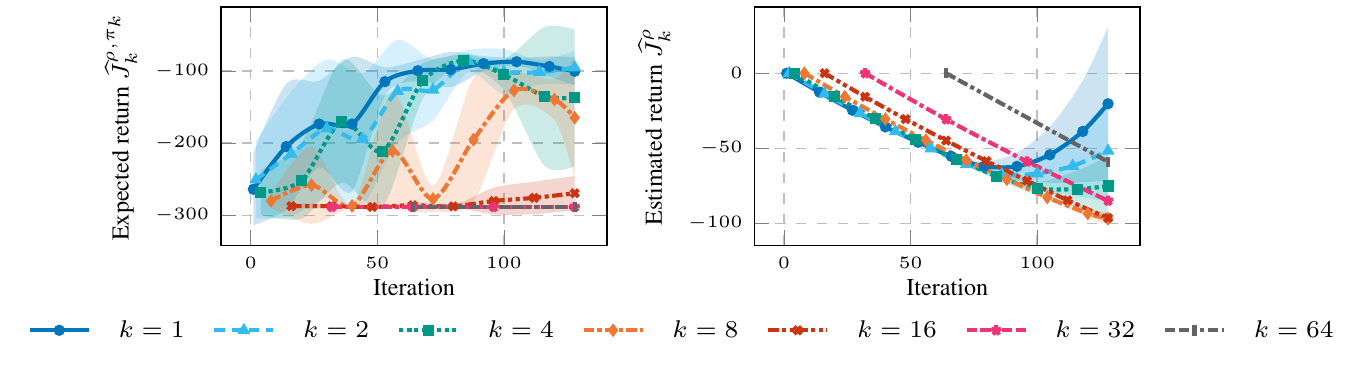}}\vspace{-.5cm}
\caption{Expected return $\widehat{J}_k^{\rho,\pi_k}$ and estimated return $\widehat{J}_k^{\rho}$ for the different experiments with neural network, as a function of the number of iterations for different persistences. 20 runs, 95 \% c.i.}\label{fig:resFigNN}
\end{figure} 

\setlength{\tabcolsep}{2pt}
\begin{table*}[!h]
\caption{Results of PFQI execution in different environments and persistences with neural network as regressor. For each persistence $k$, we report the sample mean and the standard deviation of the estimated return of the last policy $\widehat{J}_k^{\rho, \pi_k}$. For each environment, the persistence with the highest average performance and the ones that are not statistically significantly different from that one (Welch's t-test with $p < 0.05$) are in bold.}
\label{tab:resultsNN}
\begin{center}
\scriptsize
\begin{tabular}{lccccccc}
\toprule
\multirow{2}{*}{\footnotesize Environment} & \multicolumn{7}{c}{\footnotesize Expected return at persistence $k$ \scriptsize ($\widehat{J}_k^{\rho, \pi_k}$, mean $\pm$ std)} \\
\cline{2-8}
 &  $k=1$ &  $k=2$ &  $k=4$ &  $k=8$ &  $k=16$ &  $k=32$ &  $k=64$ \\
\midrule
\footnotesize Cartpole & $ \mathbf{ 95.6 \pm 21.8 } $ & $ \mathbf{ 123.6 \pm 14.4 } $ & $ \mathbf{ 121.4 \pm 5.9 } $ & $ 10.0 \pm 0.1 $ & $ 9.7 \pm 0.0 $ & $ 9.8 \pm 0.1 $ & $ 9.8 \pm 0.0 $  \\
\footnotesize LunarLander & $ -24.3 \pm 8.8 $ & $ \mathbf{ -5.3 \pm 10.4 } $ & $ \mathbf{ -5.9 \pm 7.4 } $ & $ \mathbf{ 5.0 \pm 14.0 } $ & $ -45.7 \pm 9.2 $ & $ -189.0 \pm 12.0 $ & $ -169.7 \pm 33.1 $ \\
\footnotesize Acrobot & $ \mathbf{ -100.6 \pm 7.1 } $ & $ \mathbf{ -95.4 \pm 8.5 } $ & $ \mathbf{ -137.2 \pm 34.0 } $ & $ -164.9 \pm 30.9 $ & $ -269.4 \pm 8.6 $ & $ -288.4 \pm 0.0 $ & $ -288.4 \pm 0.0 $ \\
\bottomrule
\end{tabular}
\end{center}
\end{table*}

\subsection{Performance Dependence on Batch Size}\label{apx:batch_dep}
In previous experiments we assumed we could choose the batch size, however, in real contexts this is not always allowed. In PFQI, lower batch sizes increase the estimation error, but the effect can change according to the used persistence.
We wanted to investigate how the batch size influences the performance of PFQI policies for different persistences. Therefore, we run PFQI on the Trading environment (described below) changing the number of sampled trajectories. As it can be noticed in Figure \ref{fig:trading_fig}, if the batch size is small ($10,50, 100$), higher persistences ($2,4,8$) results in better performances, while, with persistence $1$, performance decreases with the iterations. In particular, with $50$ trajectories, we can notice how all persistences except from $1$ obtain a positive gain. 

\textbf{FX Trading Environment Description}~~This environment simulates trading on a foreign exchange market. Trader's own currency is $USD$ and it can be traded with $EUR$. The trader can be in three different position w.r.t. the foreign currency: long, short or flat, indicated, respectively, with $1, -1, 0$. Short selling is possible, i.e., the agent can sell a stock it does not own. At each timestep the agent can choose its next position with its action $a_t$. The exchange rate at time $t$ is $p_t$, and the reward is equal to $R_t=a_t(p_t-p_{t-1})-f|a_t-a_{t-1}|$, where the first term is the profit or loss given by the action $a_t$, and the second term represents the transaction costs, where $f$ is a proportionality constant set to $4 \cdot 10^{-5}$. A timestep corresponds to 1 minute, an episode corresponds to a work day and it is composed by 1170 steps. It is assumed that at each time-step the trader goes long or short of the same unitary amount, thus the profits are not re-invested (and similarly for the losses), which means that the return is the sum of all the daily rewards (with a discount factor equal to $0.9999$).
The state consists of the last 60 minutes of price differences with the first price of the day ($p_t - p_0$), with the addition of the previous portfolio position as well as the fraction of time remaining until the end of the episode. For our experiments we sampled randomly daily episodes from a window of 64 work days of 2017, evaluating the performances on the last 20 days of the window. 

\textbf{Regressor Hyperparameters}~~We used the class \emph{ExtraTreesRegressor} in the \textit{scikit-learn} library~\citep{scikit-learn} with the following parameters: n\_estimators = 10, min\_samples\_split = 2, and min\_samples\_leaf = 2.

\begin{figure}[h!]
\includegraphics[scale=1.18]{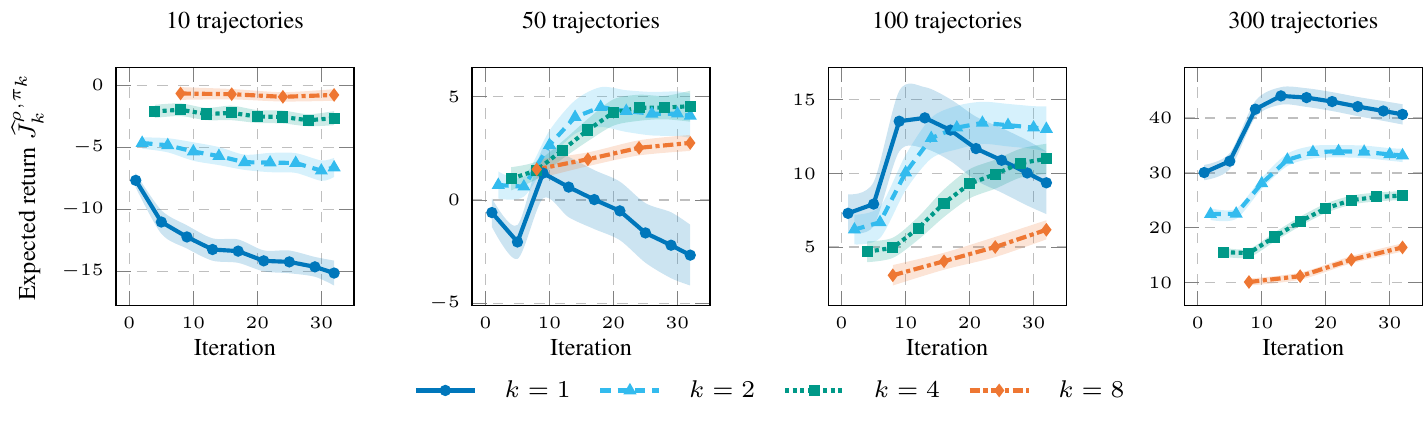}
\caption{Performances for each persistence along the iterations, with different numbers of trajectories. 10 runs, 95\% c.i.}\label{fig:trading_fig}
\end{figure}


\clearpage
\section{Preliminary Results on Open Questions (Section~\ref{sec:discussion})}\label{apx:openQuestions}
In this appendix, we report some preliminary results related to the first two open questions about action persistence we presented in Section~\ref{sec:discussion}.

\subsection{Improving Exploration with Persistence}\label{apx:persistenceExplor}
As we already mentioned, action persistence might have an effect on the exploration properties of distribution $\nu$ used to collect samples. To avoid this phenomenon, in this work, we assumed to feed \algname with the same dataset collected in the base MDP $\mathcal{M}$, independently on which target persistence $k$ we are interested in. In this appendix, we want to briefly analyze what happens when we feed standard FQI with a dataset collected by executing the same policy (\eg the uniform policy over $\As$) in the $k$--persistent MDP $\mathcal{M}_k$,\footnote{This procedure generates a different dataset compared to the case in which we use a \quotes{sampling persistence} $k_{\text{sampling}} > 1$, as illustrated in Appendix~\ref{apx:experimentaldetails}. Indeed, in this case we do not record in the dataset the intermediate repeated actions, since we want a dataset of transition of the $k$--persistent MDP $\mathcal{M}_k$.} in order to estimate the corresponding $k$--persistence action-value function $Q^*_k$. In this way, for each persistence $k$ we have a different sampling distribution $\nu_k$, but, being the dataset $\mathcal{D}_k \sim \nu_k$ collected in $\mathcal{M}_k$, we can apply standard FQI to estimate $Q^*_k$. Refer to Figure~\ref{fig:fqiComparison} for a graphical comparison between \algname executed in the base MDP and FQI executed in the $k$--persistent MDP.

\begin{figure*}[h!]
\centering
\subcaptionbox{\algname on $\mathcal{M}$}{\includegraphics[width=.4\textwidth]{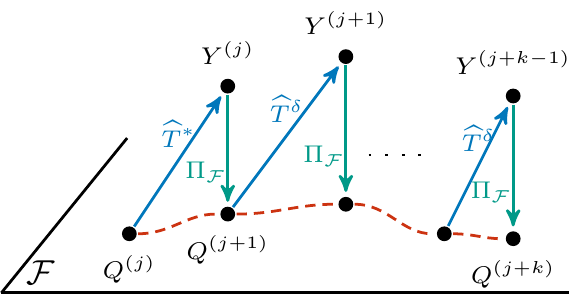}}\hspace{1cm}
\subcaptionbox{FQI on $\mathcal{M}_k$}{\includegraphics[width=.4\textwidth]{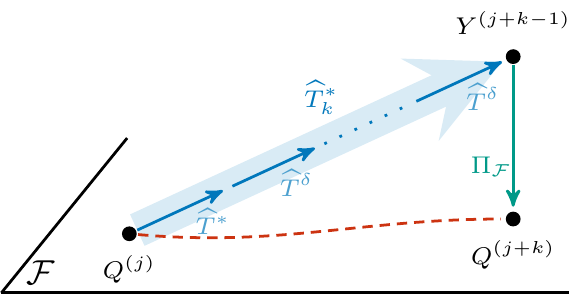}}
\caption{Illustration of (a) \algname executed in the base MDP $\mathcal{M}$ and (b) the standard FQI executed in the $k$-persistent MDP $\mathcal{M}_k$.}\label{fig:fqiComparison}
\end{figure*}

When we compare the performances of the policies obtained with different persistence levels learned starting with a dataset $\mathcal{D}_k \sim \nu_k$, we should consider two different effects: i) how training samples are generated (\ie the sampling distribution $\nu_k$, which changes for every persistence $k$); ii) how they affect the learning process in FQI. Unfortunately, in this setting we are not able to separate the two effects.

Our goal, in this appendix, is to compare for different values of $k \in \mathcal{K} = \{1,2,\dots 64\}$ the performance of \algname and the performance of FQI run on the $k$--persistent MDP $\mathcal{M}_k$. The experimental setting is the same as in Appendix~\ref{apx:Experiments}, apart from the \quotes{sampling persistence} which is set to 1 also for the Mountain Car environment. In Figure~\ref{fig:PFQI_FQI_results}, we show the performance at the end of training of the policies obtained with \algname, the one derived with FQI on $\mathcal{M}_k$, and the uniform policy over the action space. First of all, we observe that when $k=1$, executing FQI on $\mathcal{M}_1$ is in all regards equivalent to executing PFQI($1$) on $\mathcal{M}$, since PFQI($1$) is FQI and $\mathcal{M}_1$ is $\mathcal{M}$.
We can see that in the Cartpole environment, fixing a value of $k\in \mathcal{K}$, there is no significant difference in the performances obtained with \algname and FQI on $\mathcal{M}_k$. The behavior is significantly different when considering Mountain Car. Indeed, we notice that only FQI on $\mathcal{M}_k$ is able to learn a policy that reaches the goal for some specific values of $k \in \mathcal{K}$. We can justify this behavior with the fact that by collecting samples at a persistence $k$, like in FQI on $\mathcal{M}_k$, the exploration properties of the sampling distribution change, as we can see from the line \quotes{Uniform policy}. If the input dataset contains no trajectory reaching the goal, our algorithms cannot solve the task. This is why \algname, that uses persistence 1 to collect the samples, is unable to learn at all.\footnote{Recall that in our main experiments (Appendix~\ref{apx:experimentaldetails}), we had to employ for the Mountain Car a \quotes{sampling persistence} $k_{\text{sampling}} = 8$. Indeed, for $k_{\text{sampling}} \in \{1,2,4\}$ the uniform policy is unable to reach the goal, while for $k_{\text{sampling}} = 8$ it allows reaching the goal in the 6\% of the times on average.} 

This experiment gives a preliminary hint on how action persistence can affect exploration. More in general, we wonder which are the necessary characteristics of the environment such that the same sampling policy (\eg the uniform policy over $\As$) allows to perform a better exploration. More formally, we ask ourselves how the persistence affects the entropy of the stationary distribution induced by the sampling policy.

\begin{figure}[h!]
\centering
\includegraphics[scale=1.2]{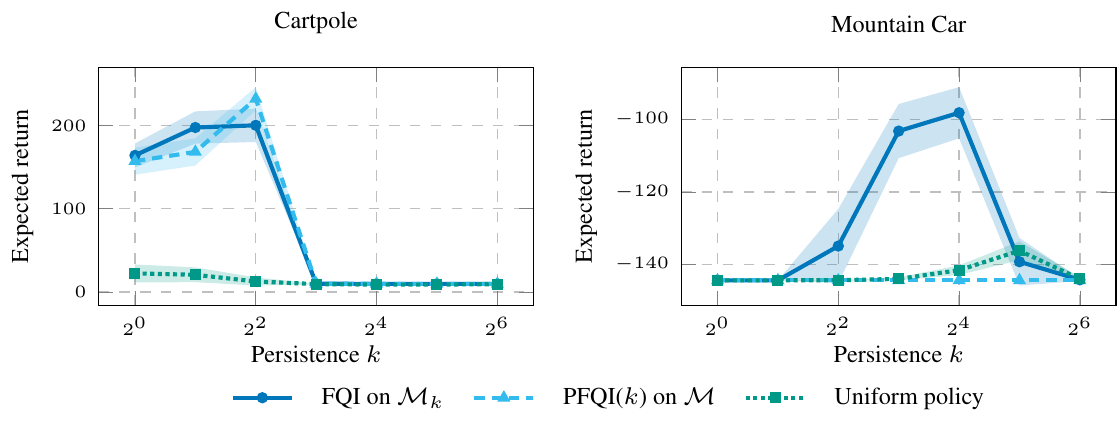}
\caption{Performance of the policies learned with FQI on $\mathcal{M}_k$, \algname on $\mathcal{M}$ and the one of the uniform policies for different values of the persistence $k \in \mathcal{K}$. 10 runs. 95\% c.i.}\label{fig:PFQI_FQI_results}
\end{figure}

\subsection{Learn in $\mathcal{M}_k$ and execute in $\mathcal{M}_{k'}$}\label{apx:khpersistence}
In this appendix, we empirically analyze what happens when a policy is learned by PFQI with a certain persistence level $k$ and executed later on with a different persistence level $k' \neq k$. We consider an experiment on the Cartpole environment, in the same setting as Appendix~\ref{apx:Experiments}. We run \algname for $k \in \mathcal{K} = \{1,2,\dots,256\}$ and then for each $k$ we execute policy $\pi_k$ (\ie the policy learned by applying the $k$--persistent operator) in the $k'$--persistent MDP $\mathcal{M}_{k'}$ for $k' \in \mathcal{K} $. The results are shown in Table~\ref{tab:reskh}. Thus, for each pair $(k,k')$, Table~\ref{tab:reskh} shows the sample mean and the sample standard deviation over 20 runs of the expected return of policy $\pi_k$ in MDP $\mathcal{M}_k$, \ie $J^{\rho,\pi_k}_{k'}$. First of all, let us observe that the diagonal of Table~\ref{tab:reskh} corresponds to the first row of Table~\ref{tab:results} (apart from the randomness due to the evaluation). If we take a row $k$, \ie we fix the persistence of the operator, we notice that, in the majority of the cases, the persistence $k'$ of the MDP yielding the best performance is smaller than $k$. Moreover, even if we learn a policy with the operator at a given persistence $k$ and we see that such a policy displays a poor performance in the $k$--persistent MDP (\eg for $k \ge 8$), when we reduce the persistence, the performance of that policy seems to improve. 

\setlength{\tabcolsep}{4pt}
\begin{table*}[t]
\caption{Results of PFQI execution of the policy $\pi_k$ learned with the $k$--persistent operator in the $k'$--persistent MDP $\mathcal{M}_{k'}$ in the Cartpole experiment. For each $k$, we report the sample mean and the standard deviation of the estimated return of the last policy $\widehat{J}_{k'}^{\rho, \pi_k}$. For each $k$, the persistence $k'$ with the highest average performance and the ones $k'$ that are not statistically significantly different from that one (Welch's t-test with $p < 0.05$) are in bold.}
\label{tab:reskh}
\begin{center}
\scriptsize
\begin{tabular}{lccccccccc}
\toprule
 &  $k'=1$ &  $k'=2$ &  $k'=4$ &  $k'=8$ &  $k'=16$ &  $k'=32$ &  $k'=64$ & $k'=128$ & $k'=256$ \\
\midrule
$ k=1 $ & $ \mathbf{ 172.0 \pm 6.8 } $ & $ \mathbf{ 174.1 \pm 6.5 } $ & $ 113.0 \pm 5.3 $ & $ 9.8 \pm 0.0 $ & $ 9.7 \pm 0.0 $ & $ 9.7 \pm 0.1 $ & $ 9.8 \pm 0.0 $ & $ 9.7 \pm 0.0 $ & $ 9.7 \pm 0.0 $\\
$ k=2 $ & $ \mathbf{ 178.4 \pm 6.7 } $ & $ \mathbf{ 182.2 \pm 7.2 } $ & $ 151.6 \pm 5.1 $ & $ 9.9 \pm 0.0 $ & $ 9.8 \pm 0.0 $ & $ 9.8 \pm 0.0 $ & $ 9.8 \pm 0.0 $ & $ 9.8 \pm 0.0 $ & $ 9.8 \pm 0.0 $\\
$ k=4 $ & $ 276.2 \pm 3.8 $ & $ \mathbf{ 287.3 \pm 1.1 } $ & $ 237.0 \pm 5.4 $ & $ 10.0 \pm 0.0 $ & $ 9.8 \pm 0.0 $ & $ 9.8 \pm 0.0 $ & $ 9.9 \pm 0.0 $ & $ 9.8 \pm 0.0 $ & $ 9.9 \pm 0.0 $\\
$ k=8 $ & $ \mathbf{ 284.3 \pm 1.6 } $ & $ \mathbf{ 281.4 \pm 3.0 } $ & $ 211.5 \pm 4.0 $ & $ 10.0 \pm 0.0 $ & $ 9.8 \pm 0.0 $ & $ 9.8 \pm 0.0 $ & $ 9.8 \pm 0.0 $ & $ 9.8 \pm 0.0 $ & $ 9.9 \pm 0.0 $\\
$ k=16 $ & $ \mathbf{ 285.9 \pm 1.1 } $ & $ \mathbf{ 282.9 \pm 2.6 } $ & $ 223.5 \pm 3.2 $ & $ 10.0 \pm 0.0 $ & $ 9.9 \pm 0.0 $ & $ 9.8 \pm 0.0 $ & $ 9.9 \pm 0.0 $ & $ 9.9 \pm 0.0 $ & $ 9.8 \pm 0.0 $\\
$ k=32 $ & $ \mathbf{ 285.7 \pm 1.3 } $ & $ \mathbf{ 283.6 \pm 2.7 } $ & $ 222.2 \pm 3.6 $ & $ 10.0 \pm 0.0 $ & $ 9.9 \pm 0.0 $ & $ 9.9 \pm 0.0 $ & $ 9.8 \pm 0.0 $ & $ 9.9 \pm 0.0 $ & $ 9.9 \pm 0.0 $\\
$ k=64 $ & $ \mathbf{ 283.6 \pm 2.3 } $ & $ \mathbf{ 284.1 \pm 2.0 } $ & $ 225.5 \pm 4.4 $ & $ 10.0 \pm 0.0 $ & $ 9.9 \pm 0.0 $ & $ 9.8 \pm 0.0 $ & $ 9.9 \pm 0.0 $ & $ 9.8 \pm 0.0 $ & $ 9.9 \pm 0.0 $\\
$ k=128 $ & $ \mathbf{ 282.9 \pm 2.2 } $ & $ \mathbf{ 282.5 \pm 3.1 } $ & $ 221.9 \pm 4.7 $ & $ 10.0 \pm 0.0 $ & $ 9.8 \pm 0.0 $ & $ 9.9 \pm 0.0 $ & $ 9.9 \pm 0.0 $ & $ 9.9 \pm 0.0 $ & $ 9.9 \pm 0.0 $\\
$ k=256 $ & $ \mathbf{ 282.5 \pm 2.3 } $ & $ \mathbf{ 283.4 \pm 2.4 } $ & $ 224.3 \pm 3.9 $ & $ 10.0 \pm 0.0 $ & $ 9.9 \pm 0.0 $ & $ 9.9 \pm 0.0 $ & $ 9.9 \pm 0.0 $ & $ 9.9 \pm 0.0 $ & $ 9.9 \pm 0.0 $\\
\bottomrule
\end{tabular}
\end{center}
\end{table*}

Figure~\ref{fig:kh} compares for different values of $k$, determining the persistence of the operator, the performance of the policy $\pi_k$ when we execute it in $\mathcal{M}_k$ and the performance of $\pi_k$ in the MDP $\mathcal{M}_{(k')^*}$, where $(k')^* \in \argmax_{k' \in 
\mathcal{K}} \widehat{J}^{\rho,\pi_k}_{k'}$. We clearly see that suitably selecting the persistence $k'$ of the MDP in which we will deploy the policy, allows reaching higher performances.

\begin{figure}[h!]
\centering
\includegraphics[scale=1.2]{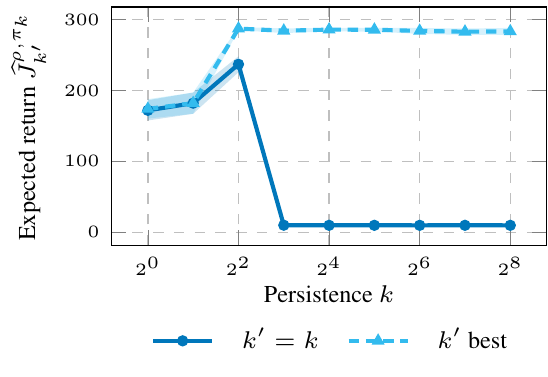}
\caption{Performance of the policies $\pi_k$ for $k\in\mathcal{K}$ comparing when they are executed in $\mathcal{M}_k$ and when they are executed in $\mathcal{M}_{(k')^*}$. 20 runs, 95\% c.i.}\label{fig:kh}
\end{figure}

The question we wonder is whether this behavior is a property of the Cartpole environment or is a general phenomenon that we expect to occur in environments with certain characteristics. If so, which are those characteristics? Furthermore, when we allow executing $\pi_k$ in $\mathcal{M}_{k'}$ we should rephrase the persistence selection problem (Equation~\eqref{eq:persistenceSelectionProblem}) as follows:
\begin{equation}
	k^*, (k')^* \in \argmax_{k,k' \in \mathcal{K}} J^{\rho,\pi_k}_{k'}, \quad \rho \in \mathscr{P}(\Ss).
\end{equation} 
Similarly to the case of Equation~\eqref{eq:persistenceSelectionProblem}, we cannot directly solve the problem if we are not allowed to interact with the environment. Is it possible to extend Lemma~\ref{thr:lowerBoundQ} and the subsequent heuristic simplifications to get a usable index $B_{k,k'}$ similar to Equation~\eqref{eq:index}?

\end{document}